\documentclass[preprint,3p,times]{elsarticle}
\RequirePackage{multirow,booktabs,subfigure,color,array,hhline,makecell}
\usepackage{amssymb}
\usepackage{amsmath}
\usepackage{graphicx}
\usepackage{amsthm}
\usepackage{mathrsfs}
\usepackage{indentfirst}
\usepackage[colorlinks,citecolor=blue,urlcolor=blue]{hyperref}
\allowdisplaybreaks
\usepackage[table]{xcolor}
\usepackage{tikz-network}
\usepackage{pgf}
\usepackage{tikz}
\usepackage{subfigure}
\usetikzlibrary{arrows, decorations.pathmorphing, backgrounds, positioning, fit, petri, automata}
\usetikzlibrary{shadows,arrows,positioning}
\RequirePackage{algorithm,algpseudocode}
\allowdisplaybreaks
\usepackage{changes}

\newtheorem{thm}{Theorem}
\newtheorem{defin}{Definition}
\newtheorem{lem}{Lemma}
\newtheorem{assum}{Assumption}
\newtheorem{rem}{Remark}

\allowdisplaybreaks[4]

\AtBeginDocument{%
	\providecommand\BibTeX{{%
			\normalfont B\kern-0.5em{\scshape i\kern-0.25em b}\kern-0.8em\TeX}}}
\journal{~}
\begin{document}
\begin{frontmatter}
\title{Goodness-of-Fit Tests for Latent Class Models with Ordinal Categorical Data}

\author[label1]{Huan Qing\corref{cor1}}
\ead{qinghuan@u.nus.edu\&qinghuan@cqut.edu.cn}
\cortext[cor1]{Corresponding author.}
\address[label1]{School of Economics and Finance, Chongqing University of Technology, Chongqing, 400054, China}
\begin{abstract}
Ordinal categorical data are widely collected in psychology, education, and other social sciences, appearing commonly in questionnaires, assessments, and surveys. Latent class models provide a flexible framework for uncovering unobserved heterogeneity by grouping individuals into homogeneous classes based on their response patterns. A fundamental challenge in applying these models is determining the number of latent classes, which is unknown and must be inferred from data. In this paper, we propose one test statistic for this problem. The test statistic centers the largest singular value of a normalized residual matrix by a simple sample-size adjustment. Under the null hypothesis that the candidate number of latent classes is correct, its upper bound converges to zero in probability. Under an under‑fitted alternative, the statistic itself exceeds a fixed positive constant with probability approaching one. This sharp dichotomous behavior of the test statistic yields two sequential testing algorithms that consistently estimate the true number of latent classes. Extensive experimental studies confirm the theoretical findings and demonstrate their accuracy and reliability in determining the number of latent classes.
\end{abstract}
\begin{keyword}
Ordinal categorical data\sep latent class model \sep goodness-of-fit \sep estimation of number of latent classes
\end{keyword}
\end{frontmatter}
\section{Introduction}\label{sec:intro}
Ordinal categorical data are commonly encountered in psychology, education, political science, and many other fields. In psychological surveys, respondents rate their agreement on a Likert scale, with options coded as \(0,1,2,3,4\) representing “strongly disagree,” “disagree,” “neutral,” “agree,” and “strongly agree.” In educational assessments, student performance is often classified into ordered proficiency levels such as \(0\) (below basic), \(1\) (basic), \(2\) (proficient), and \(3\) (advanced). In political polls, individuals indicate their level of support for a policy using a similar ordered scale. Such data can be organized into an \(N\times J\) response matrix \(R\), where \(N\) is the number of subjects (individuals), \(J\) is the number of items, and each entry \(R(i,j)\) records subject \(i\)'s response to item \(j\). Responses take values in \(\{0,1,\ldots,M\}\), with \(0\) denoting the lowest intensity and \(M\) the highest. When \(M=1\), the data are dichotomous (also known as binary); when \(M\geq 2\), they are polytomous. A key feature of ordinal categorical data is that while the categories follow a natural order, the distances between them are not necessarily equal or interpretable. Any valid statistical analysis must respect this ordinal nature without imposing unwarranted assumptions about equal spacing \citep{agresti2012categorical}.

The latent class model (LCM) provides a flexible and interpretable framework for uncovering latent population structure from such data \citep{goodman1974exploratory}. This model is widely used in psychological, behavioral, and social sciences \citep{hagenaars2002applied,wang2011latent,collins2013latent,gu2020partial}. The LCM posits that the population consists of \(K\) distinct latent classes (also known as groups), and that conditional on class membership, an individual's responses to different items are independent. This assumption accounts for the associations observed among items, as any dependence between responses arises solely from their shared latent class membership. For ordinal responses, a natural specification models each item response as a Binomial random variable with \(M\) trials and a class‑specific success probability. When \(M=1\), this reduces to the Bernoulli distribution commonly used for binary data \citep{chen2017regularized,zeng2023tensor,lyu2025spectral}. For \(M\geq 2\), the Binomial formulation directly captures the polytomous nature by allowing the expected response to increase with the success probability, thereby shifting the distribution toward higher categories. As discussed extensively in the literature on categorical data analysis \citep{agresti2012categorical}, the Binomial distribution is a fundamental tool for modeling discrete responses, and its mean is directly linked to the underlying probability of success. Historically, parameter estimation has relied on the Expectation‑Maximization (EM) algorithm, which treats class memberships as missing data \citep{dempster1977maximum}. However, the EM algorithm can be computationally demanding for large datasets and is sensitive to initialization, often requiring multiple random starts to avoid local optima \citep{biernacki2003choosing}. To address these limitations, alternative methods have been developed in recent years, including spectral clustering that exploits the low‑rank structure of the data matrix \citep{qing2024finding,qing2024grade,qing2025mixed,lyu2025degree,lyu2025spectral}, tensor decomposition methods that operate on low‑order moments \citep{zeng2023tensor}, and regularized estimation techniques that perform simultaneous parameter estimation and model selection \citep{chen2017regularized}. Other frameworks that can model ordinal categorical data with polytomous responses exist, such as the general diagnostic model proposed by \citep{von2008general}.

A fundamental and unresolved challenge in applying LCMs is determining the number of latent classes \(K\). In real‑world applications, \(K\) is rarely known a priori and must be inferred from the observed response matrix \(R\). Selecting too few classes can obscure meaningful heterogeneity, while selecting too many can lead to overfitting and scientifically spurious conclusions. Traditional approaches for selecting \(K\) include information criteria such as the Akaike Information Criterion (AIC) and Bayesian Information Criterion (BIC) \citep{akaike1998information, schwarz1978estimating}, and likelihood‑ratio tests, typically implemented with the Expectation‑Maximization (EM) algorithm. These methods inherit the computational burden of the EM algorithm, making them expensive when many candidate \(K\) values must be evaluated. Moreover, their theoretical properties in high‑dimensional settings where the number of items \(J\) grows with the sample size \(N\) remain largely unexplored. The consistency of these criteria often relies on regularity conditions that may not hold in complex latent variable models—a point examined by \citet{keribin2000consistent} in the context of mixture models, which include LCMs as a special case. More recent approaches address these limitations. Regularized latent class analysis, proposed by \citet{chen2017regularized}, learns the latent structure by shrinking small parameter differences toward zero, thereby revealing the underlying number of latent classes. Their framework uses a generalized information criterion (GIC) to select both the regularization parameter and the number of classes. Spectral methods offer another direction. By exploiting the low‑rank structure of the expected response matrix, the number of latent classes can be estimated by thresholding the singular values of the data matrix. For binary responses, \citet{lyu2025spectral} developed such a method, which is computationally efficient and theoretically grounded, but yields only a point estimate of \(K\) without a formal goodness-of-fit assessment. While promising, these existing methods lack rigorous goodness-of-fit guarantees, especially for ordinal categorical data. This gap motivates the present work.

In this paper, we introduce novel goodness‑of‑fit testing procedures that directly addresses the problem of estimating \(K\) in LCMs for ordinal categorical data. Our approaches transform a challenging model selection problem into a sequence of simple spectral checks, offering both computational efficiency and statistical rigor. The main contributions of this work are as follows:
\begin{itemize}
    \item We develop a goodness-of-fit framework for determining the number of latent classes in ordinal categorical data under the latent class models. The framework introduces a test statistic constructed from a normalized residual matrix by a simple sample-size adjustment. We prove that its upper bound converges to zero in probability under the null hypothesis that the candidate number of latent classes \(K_0\) equals the true \(K\), while the statistic itself exceeds a fixed positive constant with probability approaching one under an under-fitted alternative. These results are obtained using matrix concentration inequalities and a perturbation analysis that controls the error from estimated parameters.
    \item Based on the dichotomous behavior of the test statistic, we develop two sequential testing procedures that consistently estimates the true number of latent classes under the latent class models. The consistency theorems specify the required growth rates of the thresholds relative to the sample dimensions.
    \item We evaluate both procedures through extensive simulations. The results show that our methods significantly outperform existing approach in accuracy. A real-data application further demonstrates their practical value.
\end{itemize}

The remainder of this paper is organized as follows. Section~\ref{sec:model} defines the latent class model for ordinal categorical data and formally states the problem of estimating the number of latent classes. Section~\ref{sec:statistic} introduces the test statistic and establishes its asymptotic properties under both null and under‑fitted alternatives. Section~\ref{sec:algorithms} presents two sequential testing procedures based on these statistics and proves their consistency. Section~\ref{sec:numerical} reports simulation studies and a real data application. Section~\ref{sec:conclusion} concludes this paper and discusses future research directions. Technical proofs are given in the Appendix.
\section{Model and problem}\label{sec:model}
This section introduces the latent class model for ordinal categorical data and then states the problem of estimating the number of latent classes.\subsection{Latent class model}\label{subsec:model}
We now give a formal definition of the latent class model. Recall that the observed data form an \(N\times J\) response matrix \(R\), where \(N\) is the number of subjects and \(J\) is the number of items. Each entry \(R(i,j)\) records the response of subject \(i\) to item \(j\), taking an ordinal value in the set \(\{0,1,\dots,M\}\), with \(M\ge 1\) fixed and representing the highest response category.

Suppose all $N$ subjects are divided into \(K\) distinct latent classes, which we denote by \(\mathcal{C}_1, \mathcal{C}_2,\ldots,\mathcal{C}_K\). Let \(\ell(i)\in[K]\) be the class membership of subject \(i\); that is, \(\ell(i)=k\) precisely when \(i\in\mathcal{C}_k\). The classification matrix \(Z\in\{0,1\}^{N\times K}\) is defined by \(Z(i,k)=\mathbf{1}_{\{\ell(i)=k\}}\). Hence each row of \(Z\) contains exactly one \(1\), and \(Z^\top Z = \operatorname{diag}(N_1,\ldots,N_K)\), where \(N_k = |\mathcal{C}_k|\) denotes the number of subjects in class \(k\) for $k\in[K]$. We also set \(N_{\min}=\min_{k\in[K]} N_k\) and \(N_{\max}=\max_{k\in[K]} N_k\).

For each item \(j\) and each latent class \(k\), let \(\Theta(j,k)\in[0,M]\) be the expected response of a subject in class \(k\) to item \(j\). These values form the item parameter matrix \(\Theta\in[0,M]^{J\times K}\). Define the \(N\times J\) expected response matrix \(\mathcal{R}:=\mathbb{E}[R]\) as  
\[
\mathcal{R}=Z\Theta^\top.
\]  

By definition, \(\mathcal{R}(i,j)=\Theta(j,\ell(i))\). We now specify the data-generating mechanism of the observed response matrix \(R\). Following the framework of generalized linear models for categorical data \citep{agresti2012categorical}, conditional on the latent structure (i.e., on \(Z\)), we assume that the entries of \(R\) are independent and each follows a binomial distribution. The distribution depends on \(Z\) and \(\Theta\) only through the expected response matrix \(\mathcal{R}\). Explicitly,  
\[
R(i,j)\;\sim\; \mathrm{Binomial}\!\left(M,\;\frac{\mathcal{R}(i,j)}{M}\right),\qquad 
i\in[N],\;j\in[J],
\]  
independently across all pairs \((i,j)\). Equivalently, for each \(m\in\{0,1,\ldots,M\}\), we have
\[
\mathbb{P}\bigl(R(i,j)=m\bigr)=\binom{M}{m}\Bigl(\frac{\mathcal{R}(i,j)}{M}\Bigr)^{\!m}\Bigl(1-\frac{\mathcal{R}(i,j)}{M}\Bigr)^{\!M-m} ,\qquad 
i\in[N],\;j\in[J].
\]  

It follows immediately that  
\[
\mathbb{E}[R(i,j)]=\mathcal{R}(i,j),\qquad 
\operatorname{Var}[R(i,j)]=\mathcal{R}(i,j)\Bigl(1-\frac{\mathcal{R}(i,j)}{M}\Bigr)\qquad 
i\in[N],\;j\in[J].
\]  

We now formally define the latent class model via the following definition.
\begin{defin}[Latent class model for ordinal categorical data]\label{def:LCM_ordinal}
For fixed positive integers \(N,J,M,K\in\mathbb{N}\) with \(M\ge1\) and \(K\ge1\), the latent class model with \(K\) latent classes, denoted \(\mathrm{LCM}(K)\), is a statistical model for the observed response matrix \(R\in\{0,1,\ldots,M\}^{N\times J}\) with parameters
\[
(Z,\Theta)\in\{0,1\}^{N\times K}\times[0,M]^{J\times K},
\]
where \(Z\) is the classification matrix induced by a partition \(\{\mathcal{C}_1,\ldots,\mathcal{C}_K\}\) of \([N]\) and \(\Theta\) is the item parameter matrix. The expected response matrix is \(\mathcal{R}=Z\Theta^\top\). Conditional on \(Z\) and \(\Theta\), the entries of \(R\) are independent and satisfy
\[
R(i,j)\sim\mathrm{Binomial}\!\left(M,\;\frac{\mathcal{R}(i,j)}{M}\right),\qquad i\in[N],\;j\in[J].
\]
\end{defin}
The formulation of the latent class model given in Definition \ref{def:LCM_ordinal} follows the seminal framework introduced by \citep{goodman1974exploratory}, where the population is assumed to consist of \(K\) latent classes and, conditional on class membership, the responses to different items are independent. The LCM model considered in this paper extends that classical model to accommodate ordinal categorical responses through a binomial specification. When \(M=1\), the binomial distribution reduces to Bernoulli, and \(\mathrm{LCM}(K)\) becomes the classical latent class model for binary responses, which has been extensively studied in the literature \citep{chen2017regularized,xu2018identifying,gu2020partial,gu2023bayesian,zeng2023tensor,lyu2025spectral}.
\subsection{Problem statement}\label{subsec:problem}
Throughout this paper, we use \(K\) to denote the true number of latent classes in the latent class model defined in Definition \ref{def:LCM_ordinal}, and use \(K_0\) to denote a hypothetical number of latent classes. The goal is to estimate \(K\) from the observed response matrix \(R\). We adopt a sequential goodness‑of‑fit testing framework. This approach was first introduced for stochastic block models by \citet{bickel2016hypothesis}, who developed a recursive bipartitioning algorithm based on the limiting Tracy‑Widom distribution of the largest eigenvalue of the centered adjacency matrix. Subsequent work extended this idea to test \(H_0:K=K_0\) directly using various test statistics with known asymptotic null distributions \citep{lei2016goodness, hu2021using,jin2023optimal,wu2024spectral}. For a candidate value \(K_0\), we test
\[
H_0: K = K_0 \quad \text{versus} \quad H_1: K > K_0.
\]

If \(H_0\) is not rejected, we set \(\hat K = K_0\); otherwise we increase \(K_0\) and repeat. The key is to construct a test statistic \(T_{K_0}\) whose behavior under the null and alternative hypotheses is sharply different. The following sections develop such a statistic and establish its asymptotic properties.
\section{Test statistic}\label{sec:statistic}
The core of our method is a test statistic whose behavior is fundamentally different when the model is correctly specified compared to when it is under-fitted. We construct this statistic in two stages. First, we consider an idealized version that uses the true, unknown parameters. This ideal version provides the theoretical foundation. We then replace the unknowns with estimates to obtain a practical, data-driven statistic. Finally, we analyze the asymptotic properties of this practical statistic under both the null and alternative hypotheses.
\subsection{Ideal normalized residual matrix}\label{subsec:ideal}
To build intuition, we first construct the ideal normalized residual matrix using true parameters, which serves as the theoretical foundation for our test statistic. Define the ideal normalized residual matrix $R^* \in \mathbb{R}^{N\times J}$ as:
\begin{equation}\label{eq:Rstar}
R^*(i,j) = \frac{R(i,j) - \mathcal{R}(i,j)}{\sqrt{N\mathcal{V}(i,j)}}, \qquad
\mathcal{V}(i,j)=\mathcal{R}(i,j)\left(1-\frac{\mathcal{R}(i,j)}{M}\right).
\end{equation}

For this ideal normalized residual matrix $R^*$ to be well-defined and for its entries to have stable variance, the denominator must be bounded away from zero. This leads to our first assumption.

\begin{assum}[Parameter boundedness]\label{ass:A1}
There exists a constant \(\delta \in (0,1/2]\) such that for all \(j\in[J], k\in[K]\), \(\delta \le \frac{\Theta(j,k)}{M} \le 1-\delta\).
\end{assum}
Assumption \ref{ass:A1} guarantees that \(\mathcal{V}(i,j)\) is bounded below by a positive constant, making the normalization in Equation \eqref{eq:Rstar} valid. Under this assumption, it is straightforward to verify that \(\mathbb{E}[R^*(i,j)]=0\) and \(\operatorname{Var}(R^*(i,j)) = 1/N\) for all \(i,j\). Meanwhile, we call $\delta$ signal strength parameter in this paper for the following reason: $\delta$ controls the range of the expected item responses $\Theta(j,k)$ through the constraint $\delta \le \Theta(j,k)/M \le 1-\delta$. When $\delta$ is small, the admissible interval $[\delta M, (1-\delta)M]$ is wide, allowing the class-specific parameters to take values near the extremes $0$ or $M$. This creates substantial differences between classes, resulting in a strong signal that facilitates accurate estimation of $K$. As $\delta$ increases toward $0.5$, the interval shrinks and becomes symmetric around $M/2$. The class-specific parameters are then confined to a narrow central region, reducing the between-class differences and weakening the signal. Consequently, distinguishing between latent classes becomes more difficult, and estimating the true number of latent classes $K$ becomes increasingly challenging. Thus, we see that smaller values of $\delta$ correspond to stronger signals and easier estimation.

We are interested in the spectral norm of this ideal residual matrix. The following lemma characterizes the asymptotic behavior of $\sigma_1(R^*)$, the largest singular value of $R^*$ (i.e., spectral norm).
\begin{lem}[Spectral norm of ideal residual matrix]\label{lem:ideal_spectral}
When Assumption \ref{ass:A1} holds, for any $\epsilon > 0$, we have
\begin{equation*}
\lim_{N\to\infty} \mathbb{P}\bigl( \|R^*\| \le 1 + \sqrt{\frac{J}{N}} + \epsilon \bigr) = 1.
\end{equation*}
\end{lem}

Lemma~\ref{lem:ideal_spectral} implies that $\sigma_1(R^*)$ is asymptotically no larger than $1+\sqrt{J/N}$. This motivates the definition of an ideal test statistic,
\[
T_{\mathrm{ideal}} := \sigma_1(R^*) - \Bigl(1+\sqrt{\frac{J}{N}}\Bigr),
\]
which by Lemma \ref{lem:ideal_spectral}, satisfies \(\mathbb{P}(T_{\mathrm{ideal}} > \epsilon) \to 0\) for any \(\epsilon>0\).
\subsection{Practical test statistic}\label{subsec:practical}
In practice, the true parameters $Z$ and $\Theta$ in the ideal test statistic $T_{\mathrm{ideal}}$ are unknown. We now construct a practical test statistic using estimated parameters and establish its asymptotic properties under both null and alternative hypotheses.

Given a candidate number of latent classes $K_0$, apply a consistent classification estimator $\mathcal{M}$ to obtain estimates estimated classification matrix $\hat{Z}$, estimated item parameter matrix $\hat{\Theta}$, and the fitted expectation response matrix $\hat{\mathcal{R}} = \hat{Z}\hat{\Theta}^\top$. 

With these estimates, we define our practical normalized residual matrix \(\tilde{R} \in \mathbb{R}^{N\times J}\) as
\begin{equation}\label{eq:norm_res}
\tilde{R}(i,j) =
\begin{cases}
\dfrac{R(i,j) - \hat{\mathcal{R}}(i,j)}{\sqrt{N\hat{\mathcal{V}}(i,j)}}, & \text{if } \hat{\mathcal{V}}(i,j) > 0,\\[10pt]
0, & \text{otherwise},
\end{cases}
\end{equation}
where $\hat{\mathcal{V}}(i,j) = \hat{\mathcal{R}}(i,j)\left(1 - \hat{\mathcal{R}}(i,j)/M\right)$.

To ensure that \(\tilde{R}\) is a good proxy for the ideal  normalized residual matrix \(R^*\) under the null hypothesis, we need to control the error introduced by parameter estimation. This requires several regularity conditions on the latent class structure and the estimator.
\begin{assum}[Class balance]\label{ass:A2}
There exists a constant \(c_0 > 0\) such that for all \(k\in[K]\),
\(c_0 \frac{N}{K} \le N_k \le \frac{1}{c_0}\frac{N}{K}\).
\end{assum}
Assumption~\ref{ass:A2} ensures that no latent class is too small, which is necessary for obtaining uniform concentration bounds across classes—a key ingredient in bounding the estimation error of \(\hat{\mathcal{R}}\).

\begin{assum}[Separation condition]\label{ass:A3}
There exist two constants \(\zeta > 0\) and \(c_1>0\) such that for any distinct true classes \(k\) and \(l\), the set
\(\mathcal{T}_{kl} = \{ j\in[J] : |\Theta(j,k)-\Theta(j,l)| \ge \zeta /2 \}\)
satisfies \(|\mathcal{T}_{kl}| \ge c_1 J\).
\end{assum}
Assumption~\ref{ass:A3} guarantees that any two distinct latent classes are distinguished by at least a constant fraction of the items, with a minimum difference in expectation. This separation will play a key role both in establishing consistency of the estimator (under the null) and in creating a detectable signal under under‑fitted alternatives.

\begin{assum}[Growth conditions]\label{ass:A4}
The number of items \(J\) and the number of latent classes \(K\) satisfy the following growth rates as $N\to\infty$:
\[
\text{(i)}\ \frac{K^2 \log(N+J)}{N} \longrightarrow 0; \qquad
\text{(ii)}\ \frac{JK\log(JK)}{N} \longrightarrow 0.
\]
\end{assum}
Assumption~\ref{ass:A4} provides precise growth rates for the dimensions. They are needed to ensure that accumulated estimation errors vanish asymptotically. 
\begin{rem}
Assumptions \ref{ass:A1}--\ref{ass:A3} are mild and widely adopted in the analysis of latent class models with high-dimensional data. Specifically, the boundedness condition in Assumption \ref{ass:A1} and the class separation condition in Assumption \ref{ass:A3} are essential for establishing the consistency of estimators in binary LCMs when both the number of subjects and the number of items increase \citep{zeng2023tensor}. Assumption \ref{ass:A2}, which ensures a minimum class size, is a typical requirement for obtaining uniform concentration bounds across all latent classes. In contrast, Assumption \ref{ass:A4} provides precise growth rates on the dimensions \(J\) and \(K\) relative to the sample size \(N\) that are necessary for the technical proofs of our main theorems, particularly to control the accumulated estimation errors.
\end{rem}

Finally, the asymptotic analysis under the null hypothesis requires a condition on the classifier itself.
\begin{assum}[Consistency of classification estimator]\label{ass:A5}
When \(K_0 = K\), the classification estimation method \(\mathcal{M}\) satisfies
\[
\mathbb{P}\bigl( \hat{Z} = Z\Pi \bigr) \to 1,
\]
where \(\Pi \in \{0,1\}^{K\times K}\) is a permutation matrix.
\end{assum}

Assumption~\ref{ass:A5} requires a classifier that achieves exact recovery: \(\hat{Z}=Z\Pi\) with probability tending to one. The spectral clustering with likelihood refinement (SOLA) proposed by \citet{lyu2025spectral} provides such a guarantee, but it is designed for binary responses and does not directly handle the polytomous ordinal responses in our model. To obtain a practical estimator for our setting, we introduce a spectral clustering algorithm called SC-LCM in \ref{app:sc_lcm}. Theorem~\ref{thm:sc_lcm_consistency} shows that SC-LCM is consistent for estimating the class membership matrix \(Z\) in the sense that the clustering error \(\operatorname{err}(\hat{Z},Z)\) defined in \ref{app:sc_lcm} converges to zero in probability, but it does not meet the exact recovery requirement of Assumption~\ref{ass:A5}. This gap between a theoretical condition and the actual performance of an estimator is not uncommon. A similar situation occurs in the stochastic block model literature for network analysis. In developing a goodness-of-fit test for stochastic block models, \citet{lei2016goodness} assumed the existence of a community estimator that exactly recovers the true partition. In their numerical work, however, they used the spectral clustering algorithm studied by \citet{lei2015consistency}, which is only known to be consistent (i.e., its misclassification proportion tends to zero). Despite this discrepancy, their testing procedure performed very well in simulations. Recently, important theoretical advances suggest that exact recovery might also be achievable for SC-LCM under stronger conditions. The leave-one-out singular subspace perturbation analysis developed by \citet{zhang2024leave} provides a powerful tool for obtaining entrywise bounds on the difference between empirical and population singular vectors. Building on this, \citet{lyu2025spectral} proved that their SOLA algorithm achieves exact recovery for binary latent class models when the signal is sufficiently strong. This indicates that SC-LCM  could potentially be shown to achieve exact recovery under analogous strengthened conditions. A rigorous proof of this, however, would require a detailed entrywise analysis that is beyond the scope of the present paper, whose main focus is to develop a goodness-of-fit test for the number of latent classes. We leave this meaningful theoretical question for our future work. Inspired by the precedent in network analysis and supported by recent theoretical insights, we adopt SC-LCM as the practical implementation of \(\mathcal{M}\) in our sequential testing algorithm. As the numerical studies in Section~\ref{sec:numerical} demonstrate, the resulting sequential testing procedure estimates the true number of latent classes with high accuracy across a wide range of settings, providing empirical support for using SC-LCM as a reliable plug-in estimator even though its exact recovery property has not been formally established here.

With these assumptions in place, we can quantify the difference between the practical and ideal  normalized residual matrix matrices under the null hypothesis.
\begin{lem}[Perturbation control of normalized residual matrix]\label{lem:perturbation}
When $K=K_0$ and Assumptions \ref{ass:A1}--\ref{ass:A5} hold, we have
\begin{equation}\label{eq:L3}
\| \tilde{R} - R^* \| = o_P(1).
\end{equation}
\end{lem}

Lemma~\ref{lem:perturbation} shows that under the null hypothesis, the empirical normalized residual matrix \(\tilde{R}\) is asymptotically equivalent to its oracle counterpart \(R^{*}\) in spectral norm. Hence, \(\tilde{R}\) asymptotically inherits the behavior of \(R^{*}\) shown in Lemma~\ref{lem:ideal_spectral}. This means that after correctly fitting a \(K_0\)-class model, the empirical residual matrix is asymptotically indistinguishable from a pure noise matrix. Based on this lemma, our test statistic for a candidate \(K_0\) is then defined as
\begin{equation}\label{eq:test_stat}
T_{K_0} = \sigma_1(\tilde{R}) - (1 + \sqrt{\frac{J}{N}}).
\end{equation}

The following theorem establishes $T_{K_0}$'s behavior when the model is correctly specified.
\begin{thm}[Null behavior of test statistic]\label{thm:null}
When $K = K_0$ and Assumptions \ref{ass:A1}--\ref{ass:A5} hold, for any $\epsilon > 0$, we have
\begin{equation*}
\lim_{N\to\infty} \mathbb{P}\bigl( T_{K_0} < \epsilon \bigr) = 1.
\end{equation*}
\end{thm}

Theorem~\ref{thm:null} establishes that under the null hypothesis, \(T_{K_0}\) is bounded above by any positive constant with probability tending to one. Hence, the test statistic is asymptotically negligible, providing no evidence against the null. This result (i.e., Theorem~\ref{thm:null}) alone, however, does not suffice for model selection. Therefore, we must also understand the behavior of \(T_{K_0}\) when the candidate number of latent classes is too small. When \(K_0 < K\), the estimated expectation response matrix \(\hat{\mathcal{R}}\) cannot capture the full structure of the data, leaving a deterministic (i.e., no-random) signal in the residual matrix that we can detect. To quantify this signal, we introduce the final condition.
\begin{assum}[Signal‑to‑noise ratio]\label{ass:A6}
Define the constants
\(c \;:=\; \frac{\sqrt{c_0^3c_1}}{2\sqrt2}, C_{\mathrm{signal}} \;:=\; \frac{2c\zeta}{\sqrt{M}}, C_{\mathrm{noise}} \;:=\; \frac{5}{\delta}\). There exists a constant \(\eta_0>0\) such that for all sufficiently large \(N\),
\begin{equation*}
C_{\mathrm{signal}}\,\frac{\sqrt{J}}{\sqrt{K}\,K_0}\;\ge\;C_{\mathrm{noise}}+1+3\eta_0,
\end{equation*}
\end{assum}
Assumption~\ref{ass:A6} quantifies the strength of the signal: the left‑hand side is the minimal deterministic signal (after scaling in \(\tilde{R}\)), while the right‑hand side aggregates the maximal possible noise and the centering term, plus a buffer. When this holds, the signal dominates the noise, guaranteeing that the test will reject \(H_0\) with probability tending to one.

We are now ready to state the power guarantee.
\begin{thm}[Alternative behavior of the test statistic]\label{thm:alt}
Under \(K_0 < K\) and Assumptions~\ref{ass:A1}--\ref{ass:A4} and \ref{ass:A6}, we have
\[
\lim_{N\to\infty}\mathbb{P}\bigl(T_{K_0}>2\eta_0\bigr)=1.
\]
\end{thm}
Theorem~\ref{thm:alt} shows that under the alternative hypothesis \(K_0 < K\), the test statistic \(T_{K_0}\) exceeds a fixed positive constant \(2\eta_0\) with probability tending to one. Hence, the test has power against any under‑fitted model. Together with Theorem~\ref{thm:null}, we have established the fundamental dichotomy: $T_{K_0}$'s upper bound is near zero under correct specification but itself is large under under-fitting. This sharp dichotomous behavior of \(T_{K_0}\) under correct and under‑fitted specifications provides the foundation for the consistent sequential estimation of the true number of latent classes \(K\).
\section{Algorithms}\label{sec:algorithms}
The sharp dichotomy in the behavior of \(T_{K_0}\) under the null and alternative hypotheses motivates two natural sequential testing procedures for estimating \(K\): the first stops at the smallest \(K_0\) for which \(T_{K_0}\) falls below a threshold; the second stops when the ratio of successive test statistics exceeds a diverging threshold. This section presents two such sequential algorithms and proves their estimation consistency.
\subsection{GoF-LCM algorithm}\label{subsec:gof}
Our first algorithm directly implements the dichotomous behavior of $T_{K_0}$ for estimating $K$. Algorithm~\ref{alg:gof_lcm} sequentially tests candidates \(K_0 = 1, 2, \dots, K_{\max}\).  
It accepts the first \(K_0\) for which \(T_{K_0}\) falls below a threshold \(\tau_N\) that decays to zero.  
The maximum candidate \(K_{\max}\) is chosen so that any \(K_0 \le K_{\max}\) respects the growth condition in Assumption~\ref{ass:A4}(i). A convenient default is \(K_{\max} = \bigl\lfloor \sqrt{N/\log(N+J)} \bigr\rfloor\). This choice ensures that for all candidates considered, \(K_0^2 \log(N+J)/N \to 0\), which is required for the technical arguments in the consistency proof.

\begin{algorithm}[ht]
\caption{GoF-LCM: Goodness-of-Fit Testing for Latent Class Models}\label{alg:gof_lcm}
\begin{algorithmic}[1]
\Require{Observed response matrix $R\in\{0,1,\dots,M\}^{N\times J}$, maximum candidate number $K_{\max}$ (default $\lfloor \sqrt{N/\log(N+J)} \rfloor$), threshold sequence $\tau_N$ (default $\tau_N = N^{-1/5}$), and a classification estimator $\mathcal{M}$.}
\Ensure{Estimated number of latent classes $\hat{K}$}
\State Initialize $\hat{K} \gets K_{\max}$
\For{$K_0 = 1, 2, \dots, K_{\max}$}
    \State Apply $\mathcal{M}$ to $R$ with candidate $K_0$ to obtain $\hat{Z}$, $\hat{\Theta}$, $\hat{\mathcal{R}} = \hat{Z}\hat{\Theta}^\top$
    \State Compute $T_{K_0}$ via Equation (\ref{eq:test_stat})
    \If{$T_{K_0} < \tau_N$}
        \State $\hat{K} \gets K_0$
        \State \textbf{break} 
    \EndIf
\EndFor
\State \Return $\hat{K}$ 
\end{algorithmic}
\end{algorithm}

The following theorem establishes the consistency of Algorithm \ref{alg:gof_lcm} in estimating the true number of latent classes $K$ under the latent class models.

\begin{thm}[Consistency of GoF–LCM]\label{thm:consistency}
Let the true number of latent classes be \(K\) (which may grow with \(N\) subject to Assumption~\ref{ass:A4}). Suppose Assumptions~\ref{ass:A1}–\ref{ass:A6} hold and the threshold sequence \(\{\tau_N\}_{N\ge1}\) used in Algorithm~\ref{alg:gof_lcm} satisfy
\begin{enumerate}[(Con1)]
\item \(\tau_N\xrightarrow[N\to\infty]{}0\) and \(\displaystyle \frac{N\tau_N^2}{\mathrm{max}(JK\log (JK),\log(N+J))}\xrightarrow[N\to\infty]{}\infty\).
\end{enumerate}

Then the estimator \(\hat K\) produced by Algorithm~\ref{alg:gof_lcm} satisfies
\[
\lim_{N\to\infty}\mathbb{P}\bigl(\hat K = K\bigr)=1.
\]
\end{thm}

Theorem \ref{thm:consistency} ensures that GoF-LCM correctly identifies the true number of latent classes with probability approaching 1 as sample sizes grow. Condition (Con1) quantifies how slowly \(\tau_N\) must decay. The first requirement \(\tau_N \to 0\) ensures that under the null \(K_0 = K\) the test statistic \(T_K\) eventually falls below the threshold with high probability (Theorem~\ref{thm:null}). The second requirement guarantees that for every under‑fitted \(K_0 < K\) the statistic \(T_{K_0}\) exceeds \(\tau_N\) with probability tending to one (Theorem~\ref{thm:alt}). The sequential nature of GoF–LCM makes it computationally efficient, typically requiring only a few iterations before stopping at the true \(K\).

\begin{rem}[Choice of \(\tau_N\)]\label{rem:tau}
A simple and theoretically valid default is \(\tau_N = N^{-1/5}\).  
Under Assumption~\ref{ass:A4}, we have \(JK\log(JK) = o(N)\) and \(\log(N+J) = o(N)\). Hence \(N\tau_N^2 = N^{3/5}\) grows faster than both terms in the denominator, satisfying (Con1) provided that \(JK\log(JK)\) and \(\log(N+J)\) do not approach the order of \(N\) too rapidly.  
In practice, the algorithm is robust to moderate variations of \(\tau_N\) as long as it decays slowly. Other choices such as \(\tau_N = (\log N)^{-1}\) are also admissible under (Con1) when \(JK\log(JK)\) grows sufficiently slowly.
\end{rem}
\subsection{RGoF-LCM algorithm}\label{subsec:rgof}
In this section, we develop a ratio‑based goodness‑of‑fit test for the latent class model. This method complements the GoF‑LCM algorithm by using the ratio of successive test statistics, which often exhibits more robust finite‑sample behaviour. Recall that for a candidate number of latent classes \(K_0\), the test statistic \(T_{K_0}\) is defined in Equation \eqref{eq:test_stat}. For \(K_0\ge 2\), we introduce the ratio
\begin{equation}\label{eq:ratio_def}
r_{K_0} \;:=\; \left|\frac{T_{K_0-1}}{T_{K_0}}\right|,
\end{equation}
with the convention that \(r_1\) is not defined. The absolute value handles the possibility that \(T_{K}\) may be negative under the true model, though Theorem~\ref{thm:null} guarantees that \(T_K\) converges to zero from above. The following theorem characterizes the asymptotic behaviour of the ratio statistic $r_{K_0}$.

\begin{thm}[Asymptotic behaviour of the ratio statistic]\label{thm:ratio_behavior}
Assume that Assumptions~\ref{ass:A1}--\ref{ass:A6} hold, and $K^3=o(N)$. We have:
\begin{enumerate}
\item (Divergence at the true model.) For the true candidate \(K_0=K\),
\[
r_K \xrightarrow{\mathbb{P}} \infty \quad\text{as }N\to\infty.
\]
\item (Upper bound under under‑fitting.) For every \(K_0\) with \(2\le K_0 < K\),
\[
\lim_{N\to\infty} \mathbb{P}\left( r_{K_0} > \frac{\sqrt{M}}{c_{\mathrm{low}}}\,\sqrt{K}\,(K-1) \right) = 0,
\]
where \(c_{\mathrm{low}}\) is the constant from Lemma~\ref{lem:ratio_lower} (depending only on the model parameters \(\delta,c_0,c_1,M,\zeta\)), and \(\sqrt{K}(K-1)\) reflects the dependence on the true number of latent classes.
\end{enumerate}
\end{thm}

In contrast to the original test statistic \(T_{K_0}\), whose upper bound converges to zero under the null (Theorem~\ref{thm:null}) while the statistic itself exceeds a fixed positive constant under under‑fitting (Theorem~\ref{thm:alt}), the ratio \(r_{K_0}\) amplifies this difference. It diverges to infinity at the true model \(K_0 = K\), but for every under‑fitted \(K_0 < K\) it remains bounded above by a quantity depending on \(K\). Thus \(r_{K_0}\) exhibits a sharp peak at the true number of latent classes, suggesting a natural sequential stopping rule: compute \(r_{K_0}\) for increasing \(K_0\) and stop at the first value for which the ratio exceeds a diverging threshold \(\gamma_N\).  

We now propose a sequential algorithm based on this idea. The algorithm first checks the candidate \(K_0 = 1\) using the original test statistic with a decaying threshold \(\tau_N\); if accepted, it returns \(\hat K = 1\). Otherwise, it proceeds to compute ratios for \(K_0=2,3,\dots,K_{\max}\) and stops at the first \(K_0\) for which \(r_{K_0} > \gamma_N\). The choice of thresholds must satisfy the conditions in Theorem~\ref{thm:ratio_consistency} provided later to ensure consistent estimation of \(K\).

\begin{algorithm}[H]
\caption{RGoF-LCM: Ratio-based Goodness-of-Fit for Latent Class Models}\label{alg:rgof_lcm}
\begin{algorithmic}[1]
\Require{Observed response matrix $R \in \{0,1,\dots,M\}^{N\times J}$, maximum candidate number $K_{\max}$ (default: $\lfloor \sqrt{N/\log(N+J)} \rfloor$), threshold sequences $\tau_N$ (default $\tau_N = N^{-1/5}$) and $\gamma_N$ (default $\gamma_N = \log N$), and a classification estimator $\mathcal{M}$}
\Ensure{Estimated number of latent classes $\hat{K}$}
\State Compute $T_1$ with $K_0=1$
\If{$T_1 < \tau_N$}
    \State \Return $\hat{K} = 1$
\EndIf
\For{$K_0 = 2, 3, \dots, K_{\max}$}
    \State Compute $r_{K_0}$ via Equation (\ref{eq:ratio_def})
    \If{$r_{K_0} > \gamma_N$}
        \State \Return $\hat{K} = K_0$
        \State \textbf{break} 
    \EndIf
\EndFor
\State \Return $\hat{K}$
\end{algorithmic}
\end{algorithm}

We now prove that under appropriate conditions on the threshold sequences, the RGoF-LCM algorithm consistently estimates the true number of latent classes \(K\). For this theorem, we assume that \(K\) is fixed (does not grow with \(N\)) to simplify our analysis and the choice of the threshold $\gamma_N$.

\begin{thm}[Consistency of RGoF-LCM]\label{thm:ratio_consistency}
Let the true number of latent classes be \(K\) (fixed, i.e., not growing with \(N\)).  
Assume that Assumptions~\ref{ass:A1}--\ref{ass:A6} hold.  
Let the threshold sequences \(\gamma_N\) satisfy:
\begin{enumerate}[(Con2)]
\item \(\gamma_N \xrightarrow[N\to\infty]{}\infty\) and \(\displaystyle \frac{\gamma_N}{\sqrt{N/\log J}}\xrightarrow[N\to\infty]{}0\).
\end{enumerate}

Then the estimator \(\hat K\) produced by Algorithm~\ref{alg:rgof_lcm} satisfies
\[
\lim_{N\to\infty} \mathbb{P}(\hat K = K) = 1.\]
\end{thm}

Theorem~\ref{thm:ratio_consistency} guarantees that RGoF-LCM consistently recovers the true number of latent classes when \(K\) is fixed. Condition (Con2) ensures that the diverging threshold \(\gamma_N\) grows fast enough to be exceeded by \(r_K\) at the true model (Theorem~\ref{thm:ratio_behavior} part 1), yet slowly enough so that under‑fitted ratios \(r_{K_0}\) (\(K_0<K\)) stay below \(\gamma_N\) with high probability (Theorem~\ref{thm:ratio_behavior} part 2). The two requirements together prevent both under‑estimation and over‑estimation asymptotically.

\begin{rem}[Why \(K\) is assumed fixed]
Theorem~\ref{thm:ratio_consistency} assumes that the true number of latent classes \(K\) does not grow with the sample size. This assumption is made for theoretical convenience: the upper bound for the under‑fitted ratios \(r_{K_0}\) in Theorem~\ref{thm:ratio_behavior} contains a factor \(\sqrt{K}(K-1)\) that depends on \(K\). When \(K\) is fixed, this factor is a constant, and the condition on \(\gamma_N\) can be expressed in a simple form independent of \(K\). If \(K\) were allowed to increase with \(N\), the bound would grow with \(K\), and the choice of a diverging threshold \(\gamma_N\) would have to depend on the growth rate of \(K\) as well. Such an extension is possible in principle but would make the analysis considerably more complex. We therefore restrict ourselves to the fixed‑\(K\) setting. In practice, as long as \(K\) is small relative to the sample size, the fixed‑\(K\) analysis provides reliable guidance for selecting \(\gamma_N\).
\end{rem}
\begin{rem}[Choice of \(\gamma_N\)]\label{rem:thresholds}
A simple and theoretically valid default is \(\gamma_N = \log N\).  
Indeed \(\log N \to \infty\), and under Assumption~\ref{ass:A4} we have \(J = o(N)\), hence \(\log J \le \log N\) for all sufficiently large \(N\). Consequently,
\[
\frac{\log N}{\sqrt{N/\log J}} \le \frac{\log N}{\sqrt{N/\log N}} = \frac{(\log N)^{3/2}}{\sqrt{N}} \xrightarrow[N\to\infty]{} 0,
\]
so \(\gamma_N = \log N\) satisfies Condition (Con2). In practice, other slowly diverging sequences such as \(\log\log N\) are also admissible as long as they meet the growth restrictions.
\end{rem}
\section{Numerical Studies}\label{sec:numerical}

In this section, we conduct comprehensive experimental studies to evaluate the performance of the proposed goodness-of-fit test and the two sequential estimation algorithms GoF-LCM and RGoF-LCM. Our numerical experiments are designed to empirically validate the theoretical properties established in Theorems~\ref{thm:null} to~\ref{thm:ratio_consistency}. Specifically, we investigate:
\begin{enumerate}
    \item The behavior of the test statistic $T_{K_0}$ and the ratio statistic $r_{K_0}$ under both the null hypothesis ($H_0: K = K_0$) and the alternative hypothesis ($H_1: K > K_0$).
    \item The acceptance rates of GoF-LCM and RGoF-LCM under the true model and their rejection rates under underfitted models.
    \item The accuracy of GoF-LCM and RGoF-LCM in estimating the true number of latent classes $K$ under various combinations of sample size $N$, number of items $J$, signal strength parameter $\delta$, and true number of latent classes $K$.
    \item The computational efficiency of the two proposed methods when the sample size $N$ increases.
    \item The sensitivity of the algorithms to their thresholds: $\tau_N$ for GoF-LCM and $\gamma_N$ for RGoF-LCM.
    \item The robustness of GoF-LCM and RGoF-LCM to the number of items $J$ when $J$ grows much faster than $N$, thereby violating the growth condition $J=o(N)$ required by Assumption~\ref{ass:A4}.
\end{enumerate}

\subsection{General simulation setup}

Data are generated from the latent class models for ordinal categorical data as defined in Definition~\ref{def:LCM_ordinal}. The generation process follows the steps outlined below, with all parameters chosen to satisfy Assumptions~\ref{ass:A1}--\ref{ass:A3} unless otherwise noted.

\paragraph{Class membership matrix $Z$}
For a given true number of latent classes $K$, we assign each of the $N$ subjects to one of the $K$ classes independently with equal probability $1/K$. This yields a membership vector $\ell \in [K]^N$, and the classification matrix $Z \in \{0,1\}^{N \times K}$ is defined by $Z(i,k) = \mathbf{1}_{\{\ell(i)=k\}}$. This random assignment ensures that Assumption~\ref{ass:A2} (class balance) holds with high probability for sufficiently large $N$.

\paragraph{Item parameter matrix $\Theta$}
We generate $\Theta \in [0,M]^{J \times K}$ with controlled signal strength and class separation as follows.  Let $M=5$ be fixed throughout all experiments.  For each $j \in [J]$ and $k \in [K]$, independently draw
\[
\theta_{jk} \sim \mathrm{Uniform}[\delta M,\; (1-\delta)M],
\]
where $\delta \in (0,0.5]$ is the signal strength parameter.  This construction guarantees $\delta \le \theta_{jk}/M \le 1-\delta$ by definition, thereby satisfying Assumption~\ref{ass:A1}.

Assumption~\ref{ass:A3} requires that for any two distinct classes \(k_1 \neq k_2\), at least a fraction \(c_1\) of the items exhibit a difference in expectation at least \(\zeta/2\).  We set \(c_1 = 0.3\) throughout and define the separation threshold as \(\zeta = (1-2\delta)M/2\), which is half the length of the uniform interval.  In our experiments, the number of items is at least \(J = 60\) and the maximum number of classes is \(K = 8\). For a fixed pair of distinct classes \(k_1 \neq k_2\), consider the indicator
\[
X_j = \mathbf{1}\bigl\{|\theta_{j,k_1} - \theta_{j,k_2}| \ge \zeta/2\bigr\},\quad j = 1,\dots,J.
\]

Because \(\theta_{j,k_1}\) and \(\theta_{j,k_2}\) are independent and uniformly distributed, a simple geometric argument yields
\[
\mathbb{P}(X_j = 1) = \left(1 - \frac{\zeta/2}{(1-2\delta)M}\right)^2 = \left(\frac{3}{4}\right)^2 = \frac{9}{16},
\]
which is independent of \(\delta\).  Hence \(\mathbb{E}\bigl[\sum_{j=1}^J X_j\bigr] = \frac{9J}{16}\), which for \(J = 60\) equals \(33.75\), well above the required \(c_1 J = 18\).  By Hoeffding's inequality, the probability that a single pair fails to meet the condition is bounded by
\[
\mathbb{P}\Bigl(\sum_{j=1}^J X_j < 0.3J\Bigr) \le \exp\!\left(-2J\Bigl(\frac{9}{16} - 0.3\Bigr)^2\right) = \exp(-0.1378125\,J).
\]

For \(J = 60\) this bound is \(\exp(-8.26875) \approx 2.57\times 10^{-4}\).  A union bound over all \(\binom{K}{2}\) class pairs (with \(K = 8\), \(\binom{8}{2}=28\)) gives an overall failure probability at most
\[
\binom{8}{2} \exp(-0.1378125\times 60)\approx 0.0072.
\]

Thus, with probability exceeding \(0.9928\), a randomly generated \(\Theta\) matrix automatically satisfies Assumption~\ref{ass:A3}.  For larger \(J\) this probability rapidly approaches \(1\) (e.g., for \(J=70\) it exceeds \(0.998\)).  Given this strong theoretical guarantee, no rejection sampling or verification step is needed; all matrices are used directly in the simulations.  The signal strength parameter \(\delta\) directly controls the width of the admissible interval \([\delta M,(1-\delta)M]\): larger \(\delta\) reduces \(\zeta\), weakening the signal, exactly as intended.

\paragraph{Response matrix $R$}
Given $Z$ and $\Theta$, the expected response matrix is $\mathcal{R} = Z\Theta^\top$. For each $i \in [N]$ and $j \in [J]$, the response $R(i,j)$ is independently drawn from a Binomial distribution with $M$ trials and success probability $\mathcal{R}(i,j)/M$.

\paragraph{Estimation and evaluation}
For each simulated dataset, we apply the SC-LCM algorithm (Algorithm~\ref{alg:sc_lcm} in \ref{app:sc_lcm}) as the classification estimator $\mathcal{M}$ to obtain $\hat{Z}$ and $\hat{\Theta}$. The test statistic $T_{K_0}$ is then computed using Equation~\eqref{eq:test_stat}, and the ratio $r_{K_0}$ is computed via Equation~\eqref{eq:ratio_def}. For the sequential algorithms, we use the default thresholds: $\tau_N = N^{-1/5}$ for GoF-LCM and $\gamma_N = \log N$ for RGoF-LCM, unless otherwise stated. The maximum candidate number is set to $K_{\max} = \lfloor \sqrt{N / \log(N+J)} \rfloor$. All results are based on 200 independent Monte Carlo replications.

\subsection{Experiment 1: behavior of $T_{K_0}$ and $r_{K_0}$ under $H_0$ and $H_1$}
This experiment empirically verifies the sharp dichotomous behavior of $T_{K_0}$ (Theorems~\ref{thm:null} and~\ref{thm:alt}) and $r_{K_0}$ (Theorem~\ref{thm:ratio_behavior}). We fix the true number of latent classes $K=4$, the number of items $J=60$, and the signal strength parameter $\delta=0.2$. We then compute $T_{K_0}$ for $K_0 = 1,2,3,4$ and the ratio statistics $r_2 = |T_1/T_2|$, $r_3 = |T_2/T_3|$, and $r_4 = |T_3/T_4|$. We vary the sample size $N \in \{200, 400, 600, 800, 1000\}$. For each $N$ and each candidate $K_0$, we compute the mean and standard deviation of $T_{K_0}$ over 200 replications.

Table~\ref{tab:stat_behavior_K4} reports the results. Under the correctly specified model ($K_0 = K = 4$), the mean of $T_4$ is close to zero and its absolute value decreases as $N$ grows, with the standard deviation also shrinking. This confirms Theorem~\ref{thm:null}. For the under‑fitted models ($K_0 = 1,2,3$), $T_{K_0}$ takes large positive values that increase with $N$, validating the divergence predicted by Theorem~\ref{thm:alt}. The ratios $r_2 = |T_1/T_2|$ and $r_3 = |T_2/T_3|$ remain bounded between $1.16$ and $1.21$ as $N$ increases, consistent with part (b) of Theorem~\ref{thm:ratio_behavior} for $K_0 < K$. In contrast, $r_4 = |T_3/T_4|$ diverges because $T_3$ grows while $|T_4|$ tends to zero, confirming part (a) of Theorem~\ref{thm:ratio_behavior} for the true candidate $K_0 = K$.
\begin{table}[htbp!]
\caption{Behavior of $T_{K_0}$ and ratios $r_2, r_3, r_4$ under $H_0$ ($K_0=4$) and under‑fitted models ($K_0=1,2,3$) for true $K=4$. Values are mean (standard deviation) over 200 replications.}
\label{tab:stat_behavior_K4}
\centering
\footnotesize
\renewcommand{\arraystretch}{1.2}
\begin{tabular}{c|cccc|ccc}
\hline
\multirow{2}{*}{$N$} & \multicolumn{4}{c|}{$T_{K_0}$} & \multicolumn{3}{c}{Ratios} \\
\cline{2-8}
 & $K_0=1$ ($H_1$) & $K_0=2$ ($H_1$) & $K_0=3$ ($H_1$) & $K_0=4$ ($H_0$) & $r_2 = |T_1/T_2|$ & $r_3 = |T_2/T_3|$ & $r_4 = |T_3/T_4|$ \\
\hline
200  & 2.020 (0.171) & 1.706 (0.147) & 1.418 (0.149) & -0.037 (0.026) & 1.19 (0.12) & 1.21 (0.13) & 255.42 (2952.98) \\
400  & 2.135 (0.167) & 1.835 (0.147) & 1.551 (0.153) & -0.024 (0.018) & 1.17 (0.10) & 1.19 (0.12) & 157.51 (486.75) \\
600  & 2.214 (0.174) & 1.891 (0.147) & 1.620 (0.157) & -0.018 (0.014) & 1.18 (0.10) & 1.18 (0.12) & 229.62 (676.48) \\
800  & 2.257 (0.167) & 1.933 (0.152) & 1.662 (0.153) & -0.017 (0.011) & 1.17 (0.11) & 1.17 (0.12) & 388.22 (2547.19) \\
1000  & 2.278 (0.169) & 1.956 (0.150) & 1.690 (0.161) & -0.014 (0.010) & 1.17 (0.10) & 1.16 (0.12) & 450.91 (2527.84) \\
\hline
\end{tabular}
\end{table}

\subsection{Experiment 2: acceptance and rejection rates}

This experiment evaluates the ability of both sequential algorithms, GoF-LCM and RGoF-LCM, to correctly accept the true model and reject underfitted models. We fix $N=1000$, $J=60$, $\delta=0.2$, and consider true numbers of latent classes $K \in \{2,3,4,5,6\}$. For each true $K$, we simulate 200 independent datasets. On each dataset, we apply both algorithms:
\begin{itemize}
    \item For GoF-LCM, we record the stopping $K_0$ as the first candidate for which $T_{K_0} < \tau_N$ (with $\tau_N = N^{-1/5}$).
    \item For RGoF-LCM, we record the stopping $K_0$ as the first candidate for which $r_{K_0} > \gamma_N$ (with $\gamma_N = \log N$). If $T_1 < \tau_N$, the algorithm stops at $K_0=1$; otherwise, it proceeds to compute $r_{K_0}$ for $K_0 \ge 2$.
\end{itemize}

We then compute the proportion of times each algorithm stops at each candidate $K_0$ over the 200 replications. 

Table~\ref{tab:accept_reject_both_ext} presents the results. For each true $K$, the table shows the stopping proportions for GoF-LCM and RGoF-LCM side by side, with columns covering $K_0 = 1$ to $10$ to fully capture overfitting behavior. The main findings are:
\begin{itemize}
    \item Both algorithms exhibit high acceptance rates at the true model. For $K=2,3,4,5$, both methods stop at the true $K$ in all $200$ replications. For $K=6$, GoF‑LCM stops at $K_0=6$ in all replications, while RGoF‑LCM does so in $99.5\%$ of the cases, with a small overfitting probability of $0.5\%$ stopping at $K_0=7$.
    \item The probability of stopping at an underfitted model ($K_0 < K$) is effectively zero for both algorithms in every case.
    \item Overfitting ($K_0 > K$) occurs with very low probability, only observed in the $K=6$ case for RGoF‑LCM.
\end{itemize}

These results empirically confirm the consistency of both sequential algorithms.
\begin{table}[htbp!]
\caption{Proportion of times GoF-LCM and RGoF-LCM stop at each candidate $K_0$ for true $K=2,3,4,5,6$. Columns cover $K_0=1$ to $10$.}
\label{tab:accept_reject_both_ext}
\centering
\small
\renewcommand{\arraystretch}{1.2}
\begin{tabular}{c|c|cccccccccc}
\hline
True $K$ & Algorithm & \multicolumn{10}{c}{Stopping $K_0$} \\
\cline{3-12}
 &  & 1 & 2 & 3 & 4 & 5 & 6 & 7 & 8 & 9 & 10 \\
\hline
\multirow{2}{*}{2} & GoF-LCM   & 0.000 & \textbf{1.000} & 0.000 & 0.000 & 0.000 & 0.000 & 0.000 & 0.000 & 0.000 & 0.000 \\
 & RGoF-LCM  & 0.000 & \textbf{1.000} & 0.000 & 0.000 & 0.000 & 0.000 & 0.000 & 0.000 & 0.000 & 0.000 \\
\hline
\multirow{2}{*}{3} & GoF-LCM   & 0.000 & 0.000 &\textbf{ 1.000} & 0.000 & 0.000 & 0.000 & 0.000 & 0.000 & 0.000 & 0.000 \\
 & RGoF-LCM  & 0.000 & 0.000 & \textbf{1.000} & 0.000 & 0.000 & 0.000 & 0.000 & 0.000 & 0.000 & 0.000 \\
\hline
\multirow{2}{*}{4} & GoF-LCM   & 0.000 & 0.000 & 0.000 & \textbf{1.000} & 0.000 & 0.000 & 0.000 & 0.000 & 0.000 & 0.000 \\
 & RGoF-LCM  & 0.000 & 0.000 & 0.000 & \textbf{1.000} & 0.000 & 0.000 & 0.000 & 0.000 & 0.000 & 0.000 \\
\hline
\multirow{2}{*}{5} & GoF-LCM   & 0.000 & 0.000 & 0.000 & 0.000 & \textbf{1.000} & 0.000 & 0.000 & 0.000 & 0.000 & 0.000 \\
 & RGoF-LCM  & 0.000 & 0.000 & 0.000 & 0.000 & \textbf{1.000} & 0.000 & 0.000 & 0.000 & 0.000 & 0.000 \\
\hline
\multirow{2}{*}{6} & GoF-LCM   & 0.000 & 0.000 & 0.000 & 0.000 & 0.000 & \textbf{1.000} & 0.000 & 0.000 & 0.000 & 0.000 \\
 & RGoF-LCM  & 0.000 & 0.000 & 0.000 & 0.000 & 0.000 & \textbf{0.995} & 0.005 & 0.000 & 0.000 & 0.000 \\
\hline
\end{tabular}
\end{table}
\subsection{Experiment 3: accuracy in estimating $K$}

We now assess the estimation accuracy of the proposed methods GoF-LCM and RGoF-LCM, and compare them with the spectral thresholding method (denoted Spec) proposed in Equation (17) of  \citet{lyu2025spectral}, which estimates $K$ by counting singular values exceeding $2.01(\sqrt{J}+\sqrt{N})$. The parameters are: $N \in \{200, 600\}$, $J \in \{60, 100\}$, $\delta \in \{0.1, 0.2, 0.3\}$, and $K \in \{1,2,3,4\}$. For each combination, we generate 200 independent datasets and apply all three methods. Accuracy is defined as the proportion of replications where the estimated $\hat{K}$ equals the true $K$. Standard errors are computed as $\sqrt{\text{accuracy}\times(1-\text{accuracy})/200}$ and are shown in parentheses.

Table~\ref{tab:accuracy_full} reports the accuracy for all $48$ parameter combinations. The main findings are:
\begin{itemize}
   \item The case $K=1$ is trivial: all methods achieve perfect accuracy across all settings.
    \item GoF-LCM and RGoF-LCM achieve near-perfect accuracy across almost all settings. For all combinations except one, both methods attain an accuracy of $1.000$. The only exception is $K=4,N=200,J=60,\delta=0.3$, where GoF-LCM yields $0.985$ and RGoF-LCM yields $0.990$. This demonstrates the strong consistency and reliability of the proposed sequential testing procedures, which are able to recover the true number of latent classes even under relatively weak signal conditions.
    \item The spectral thresholding method (Spec) performs unreliably, with accuracy ranging from perfect to zero depending on the signal strength. Its accuracy depends critically on the strength of the signal, which is controlled by $\delta$, $K$, $N$, and $J$. Recall that $\delta$ is a signal strength parameter: smaller $\delta$ allows the class-specific parameters $\Theta(j,k)$ to take values near the extremes $0$ or $M$, creating substantial differences between classes and a strong signal. Larger $\delta$ confines $\Theta(j,k)$ to a narrow interval around $M/2$, reducing class separation and weakening the signal. The data confirm this interpretation:
            \begin{itemize}
        \item At $\delta = 0.1$ (strong signal), Spec performs excellently, with accuracy $1.000$ in all but one case ($K=4,N=200,J=60,\delta=0.1$ yields $0.995$).
        \item At $\delta = 0.2$ (moderate signal), Spec remains good overall, but some degradation appears: $K=3,N=200,J=60,\delta=0.2$ gives $0.995$, and $K=4,N=200,J=60,\delta=0.2$ drops to $0.520$.
        \item At $\delta = 0.3$ (weak signal), Spec frequently fails. For $K=3$, accuracy ranges from $0.000$ to $0.995$; for $K=4$, it is $0.000$ in three of four combinations, reaching only $0.565$ at the largest sample size ($N=600,J=100$). 
    \end{itemize}
    \item While Spec is simple and computationally cheap, its performance is highly sensitive to the signal strength. In weak-signal regimes ($\delta=0.3$), it can completely fail. In contrast, GoF-LCM and RGoF-LCM adaptively test the fit of candidate models and maintain near-perfect accuracy across all scenarios, including the weakest signal considered. The slight advantage of RGoF-LCM over GoF-LCM in the hardest case ($K=4,N=200,J=60,\delta=0.3$) confirms its marginally improved robustness
\end{itemize}

These results highlight the practical advantage of the proposed sequential testing framework: it delivers highly accurate and robust estimates of the number of latent classes without requiring manual tuning, whereas the simple spectral thresholding method can fail catastrophically when the signal is weak.
\begin{table}[htbp!]
\caption{Accuracy (with standard error in parentheses) of GoF-LCM, RGoF-LCM, and the spectral thresholding method (Spec) for all parameter combinations.}
\label{tab:accuracy_full}
\centering
\footnotesize
\renewcommand{\arraystretch}{1.2}
\begin{tabular}{cccc|c@{\hspace{0.5cm}}c@{\hspace{0.5cm}}c}
\hline
\multicolumn{4}{c|}{Parameters} & \multicolumn{1}{c}{GoF-LCM} & \multicolumn{1}{c}{RGoF-LCM} & \multicolumn{1}{c}{Spec} \\
$K$ & $N$ & $J$ & $\delta$ & \multicolumn{1}{c}{Accuracy (S.E.)} & \multicolumn{1}{c}{Accuracy (S.E.)} & \multicolumn{1}{c}{Accuracy (S.E.)} \\
\hline
1 & 200 & 60 & 0.1 & 1.000 (0.000) & 1.000 (0.000) & 1.000 (0.000) \\
1 & 200 & 60 & 0.2 & 1.000 (0.000) & 1.000 (0.000) & 1.000 (0.000) \\
1 & 200 & 60 & 0.3 & 1.000 (0.000) & 1.000 (0.000) & 1.000 (0.000) \\
1 & 200 & 100 & 0.1 & 1.000 (0.000) & 1.000 (0.000) & 1.000 (0.000) \\
1 & 200 & 100 & 0.2 & 1.000 (0.000) & 1.000 (0.000) & 1.000 (0.000) \\
1 & 200 & 100 & 0.3 & 1.000 (0.000) & 1.000 (0.000) & 1.000 (0.000) \\
1 & 600 & 60 & 0.1 & 1.000 (0.000) & 1.000 (0.000) & 1.000 (0.000) \\
1 & 600 & 60 & 0.2 & 1.000 (0.000) & 1.000 (0.000) & 1.000 (0.000) \\
1 & 600 & 60 & 0.3 & 1.000 (0.000) & 1.000 (0.000) & 1.000 (0.000) \\
1 & 600 & 100 & 0.1 & 1.000 (0.000) & 1.000 (0.000) & 1.000 (0.000) \\
1 & 600 & 100 & 0.2 & 1.000 (0.000) & 1.000 (0.000) & 1.000 (0.000) \\
1 & 600 & 100 & 0.3 & 1.000 (0.000) & 1.000 (0.000) & 1.000 (0.000) \\
2 & 200 & 60 & 0.1 & 1.000 (0.000) & 1.000 (0.000) & 1.000 (0.000) \\
2 & 200 & 60 & 0.2 & 1.000 (0.000) & 1.000 (0.000) & 1.000 (0.000) \\
2 & 200 & 60 & 0.3 & 1.000 (0.000) & 1.000 (0.000) & 0.880 (0.023) \\
2 & 200 & 100 & 0.1 & 1.000 (0.000) & 1.000 (0.000) & 1.000 (0.000) \\
2 & 200 & 100 & 0.2 & 1.000 (0.000) & 1.000 (0.000) & 1.000 (0.000) \\
2 & 200 & 100 & 0.3 & 1.000 (0.000) & 1.000 (0.000) & 1.000 (0.000) \\
2 & 600 & 60 & 0.1 & 1.000 (0.000) & 1.000 (0.000) & 1.000 (0.000) \\
2 & 600 & 60 & 0.2 & 1.000 (0.000) & 1.000 (0.000) & 1.000 (0.000) \\
2 & 600 & 60 & 0.3 & 1.000 (0.000) & 1.000 (0.000) & 1.000 (0.000) \\
2 & 600 & 100 & 0.1 & 1.000 (0.000) & 1.000 (0.000) & 1.000 (0.000) \\
2 & 600 & 100 & 0.2 & 1.000 (0.000) & 1.000 (0.000) & 1.000 (0.000) \\
2 & 600 & 100 & 0.3 & 1.000 (0.000) & 1.000 (0.000) & 1.000 (0.000) \\
3 & 200 & 60 & 0.1 & 1.000 (0.000) & 1.000 (0.000) & 1.000 (0.000) \\
3 & 200 & 60 & 0.2 & 1.000 (0.000) & 1.000 (0.000) & 0.995 (0.005) \\
3 & 200 & 60 & 0.3 & 1.000 (0.000) & 1.000 (0.000) & 0.000 (0.000) \\
3 & 200 & 100 & 0.1 & 1.000 (0.000) & 1.000 (0.000) & 1.000 (0.000) \\
3 & 200 & 100 & 0.2 & 1.000 (0.000) & 1.000 (0.000) & 1.000 (0.000) \\
3 & 200 & 100 & 0.3 & 1.000 (0.000) & 1.000 (0.000) & 0.440 (0.035) \\
3 & 600 & 60 & 0.1 & 1.000 (0.000) & 1.000 (0.000) & 1.000 (0.000) \\
3 & 600 & 60 & 0.2 & 1.000 (0.000) & 1.000 (0.000) & 1.000 (0.000) \\
3 & 600 & 60 & 0.3 & 1.000 (0.000) & 1.000 (0.000) & 0.450 (0.035) \\
3 & 600 & 100 & 0.1 & 1.000 (0.000) & 1.000 (0.000) & 1.000 (0.000) \\
3 & 600 & 100 & 0.2 & 1.000 (0.000) & 1.000 (0.000) & 1.000 (0.000) \\
3 & 600 & 100 & 0.3 & 1.000 (0.000) & 1.000 (0.000) & 0.995 (0.005) \\
4 & 200 & 60 & 0.1 & 1.000 (0.000) & 1.000 (0.000) & 0.995 (0.005) \\
4 & 200 & 60 & 0.2 & 1.000 (0.000) & 1.000 (0.000) & 0.520 (0.035) \\
4 & 200 & 60 & 0.3 & 0.985 (0.009) & 0.990 (0.007) & 0.000 (0.000) \\
4 & 200 & 100 & 0.1 & 1.000 (0.000) & 1.000 (0.000) & 1.000 (0.000) \\
4 & 200 & 100 & 0.2 & 1.000 (0.000) & 1.000 (0.000) & 1.000 (0.000) \\
4 & 200 & 100 & 0.3 & 1.000 (0.000) & 1.000 (0.000) & 0.000 (0.000) \\
4 & 600 & 60 & 0.1 & 1.000 (0.000) & 1.000 (0.000) & 1.000 (0.000) \\
4 & 600 & 60 & 0.2 & 1.000 (0.000) & 1.000 (0.000) & 0.980 (0.010) \\
4 & 600 & 60 & 0.3 & 1.000 (0.000) & 1.000 (0.000) & 0.000 (0.000) \\
4 & 600 & 100 & 0.1 & 1.000 (0.000) & 1.000 (0.000) & 1.000 (0.000) \\
4 & 600 & 100 & 0.2 & 1.000 (0.000) & 1.000 (0.000) & 1.000 (0.000) \\
4 & 600 & 100 & 0.3 & 1.000 (0.000) & 1.000 (0.000) & 0.565 (0.035) \\
\hline
\end{tabular}
\end{table}
\subsection{Experiment 4: computational efficiency under weak signal}
We investigate the performance of GoF-LCM, RGoF-LCM, and Spec in a challenging weak-signal scenario with signal strength parameter $\delta=0.3$, number of items $J=60$, true number of latent classes $K=8$, and sample size $N$ ranging from $400$ to $4000$ in increments of $400$. For each configuration, we generate $200$ independent datasets and apply both algorithms. The accuracy (proportion of correct $K$ estimates) and the average running time (in seconds) are recorded. Figure~\ref{fig:exp4} displays the results. Under this weak signal scenario, RGoF-LCM demonstrates strong robustness, achieving high accuracy even at the smallest sample size and reaching perfect accuracy for almost all samples. GoF-LCM, on the other hand, has low accuracy at $N=400$ and $N=800$, but its performance improves rapidly as $N$ increases, reaching perfect accuracy for $N \ge 1600$. In contrast, Spec fails completely in this weak-signal regime: its accuracy is almost always zero. All three methods are computationally efficient. The running times of GoF-LCM and RGoF-LCM scale approximately linearly with $N$, and even at $N = 4000$, the average time remains below $0.6$ seconds. Spec is extremely fast, with negligible running time across all sample sizes, due to its simple singular value thresholding procedure. These results confirm that RGoF-LCM is the method of choice in weak-signal settings, offering both high accuracy and fast computation across the entire range of sample sizes. GoF-LCM also achieves excellent accuracy once the sample size is sufficiently large, while Spec, despite its speed, is unreliable when the signal is weak.

\begin{figure}
\centering
\resizebox{\columnwidth}{!}{
{\includegraphics[width=3\textwidth]{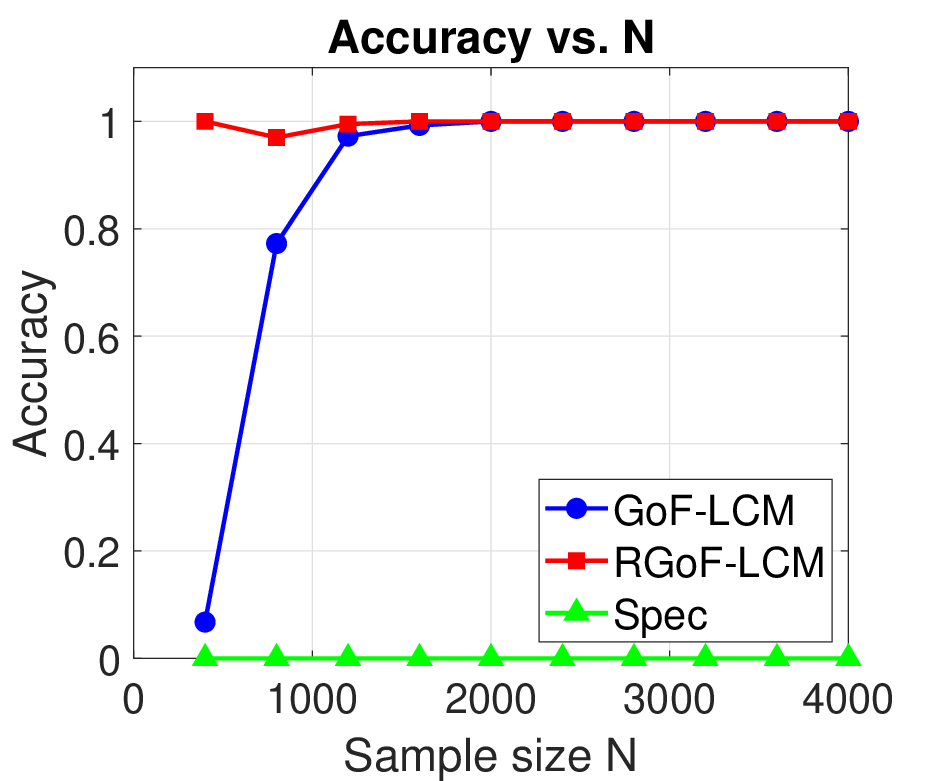}}
{\includegraphics[width=3\textwidth]{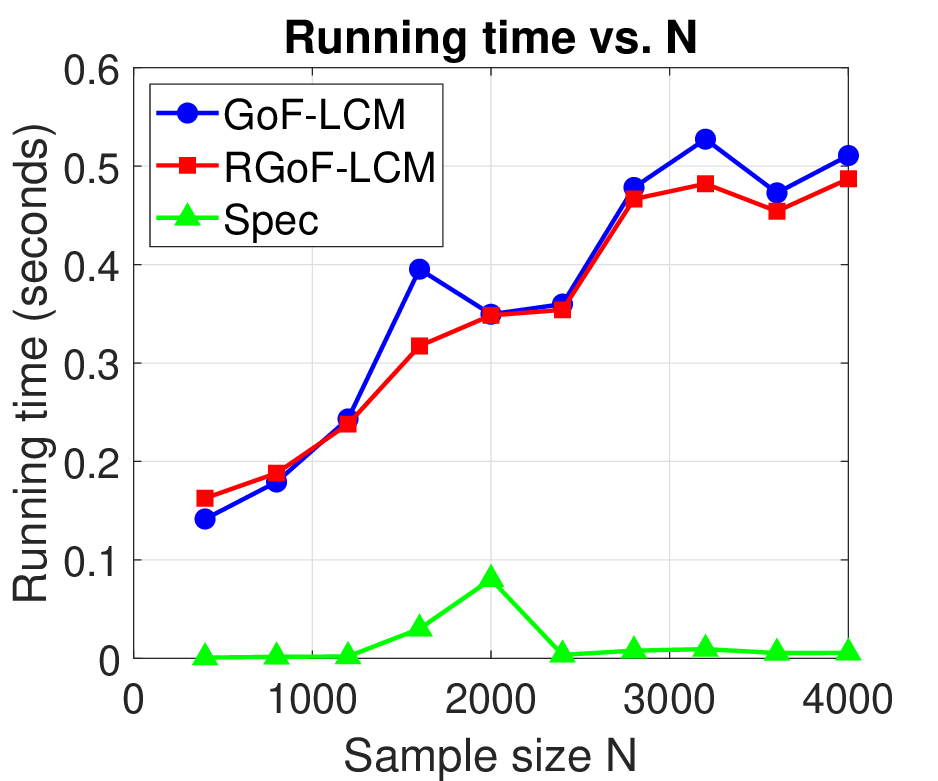}}
}
\caption{Accuracy (left) and running time (right) of GoF-LCM and RGoF-LCM for $K=8$, $J=60$, $\delta=0.3$, with varying $N$.}
\label{fig:exp4} 
\end{figure}
\subsection{Experiment 5: sensitivity to thresholds}

This final experiment investigates the sensitivity of the two algorithms to their primary tuning parameters. We fix $K=5$, $N=1000$, $J=60$, and $\delta=0.2$.

\subsubsection{Sensitivity of GoF-LCM to $\tau_N = N^{-\epsilon}$}
We vary the decay parameter $\epsilon$ in the threshold $\tau_N = N^{-\epsilon}$ over the set $\{0.1, 0.2, 0.3, 0.4, 0.5, 0.6, 0.7, 0.8, 0.9,1\}$. For each $\epsilon$, we run GoF-LCM on 200 simulated datasets and record the accuracy. The left part of Table~\ref{tab:sensitivity_gof} shows the results. Accuracy remains perfect ($1.000$) for $\epsilon \le 0.4$, then gradually declines: $0.995$ at $\epsilon=0.5$, $0.990$ at $\epsilon=0.6$, $0.970$ at $\epsilon=0.7$, $0.960$ at $\epsilon=0.8$, and finally $0.935$ and $0.930$ at $\epsilon=0.9$ and $1.0$, respectively. Overall, GoF-LCM is robust to the choice of $\epsilon$ as long as it does not exceed $0.5$; beyond this point, performance gradually degrades. Finally, the default $\epsilon = 0.2$ lies well within the stable region and yields perfect accuracy. 

\subsubsection{Sensitivity of RGoF-LCM to $\gamma_N = a \log N$}
We vary the multiplier $a$ in the threshold $\gamma_N = a \log N$ over the set $\{0.5, 1.0, 1.5, 2.0, 2.5, 3.0,3.5,4.0,4.5,5.0\}$. For each $a$, we run RGoF-LCM on 200 simulated datasets and record the accuracy. The right part of Table~\ref{tab:sensitivity_gof} shows the results. Accuracy is perfect ($1.000$) for all $a \le 4.5$, and only slightly lower ($0.995$) at $a = 5.0$. This demonstrates that RGoF-LCM is highly robust to the choice of the multiplier $a$ over a wide range. Finally, the default $a = 1$ ($\gamma_N = \log N$) lies well within the stable region.

\begin{table}[htbp!]
\caption{Sensitivity of GoF-LCM to $\epsilon$ in $\tau_N = N^{-\epsilon}$ and of RGoF-LCM to multiplier $a$ in $\gamma_N = a \log N$.}
\label{tab:sensitivity_gof}
\centering
\begin{tabular}{c|c}
\hline
\multicolumn{2}{c}{GoF-LCM: $\tau_N = N^{-\epsilon}$} \\
\hline
$\epsilon$ & Accuracy \\
\hline
0.1 & 1.000 \\
0.2 & 1.000 \\
0.3 & 1.000 \\
0.4 & 1.000 \\
0.5 & 0.995 \\
0.6 & 0.990 \\
0.7 & 0.970 \\
0.8 & 0.960 \\
0.9 & 0.935 \\
1.0 & 0.930 \\
\hline
\end{tabular}
\qquad
\begin{tabular}{c|c}
\hline
\multicolumn{2}{c}{RGoF-LCM: $\gamma_N = a \log N$} \\
\hline
$a$ & Accuracy \\
\hline
0.5 & 1.000 \\
1.0 & 1.000 \\
1.5 & 1.000 \\
2.0 & 1.000 \\
2.5 & 1.000 \\
3.0 & 1.000 \\
3.5 & 1.000 \\
4.0 & 1.000 \\
4.5 & 1.000 \\
5.0 & 0.995 \\
\hline
\end{tabular}
\end{table}

\subsection{Experiment 6: performance under large \(J\) (violating \(J=o(N)\))}\label{sec:exp6}

The theoretical analysis in Sections~\ref{sec:statistic} and~\ref{sec:algorithms} relies on Assumption~\ref{ass:A4}, which requires \(\frac{JK\log(JK)}{N}\to0\). In particular, this assumption forces the number of items \(J\) to grow slower than the sample size \(N\). This condition is used in the proofs to control accumulated estimation errors. Experiment 6 investigates the performance of GoF-LCM and RGoF-LCM when \(J\) is allowed to be substantially larger than \(N\), i.e. when \(J\) grows faster than \(N\) and the term \(\frac{JK\log(JK)}{N}\) does not tend to zero.

We fix the sample size \(N=600\), the true number of latent classes \(K=8\), and the signal strength parameter \(\delta=0.3\) (weak signal). The number of items \(J\) is varied from \(200\) to \(2000\) in steps of \(200\). For each value of \(J\), we generate \(200\) independent datasets. Figure~\ref{fig:exp6} reports the accuracy of both algorithms for all considered values of \(J\). Remarkably, for every value of \(J\) (including \(J=2000\), which is more than three times of \(N\)), both GoF-LCM and RGoF-LCM achieve almost perfect accuracy of \(1.000\).

\begin{figure}[htbp!]
\centering
\resizebox{0.5\columnwidth}{!}{\includegraphics{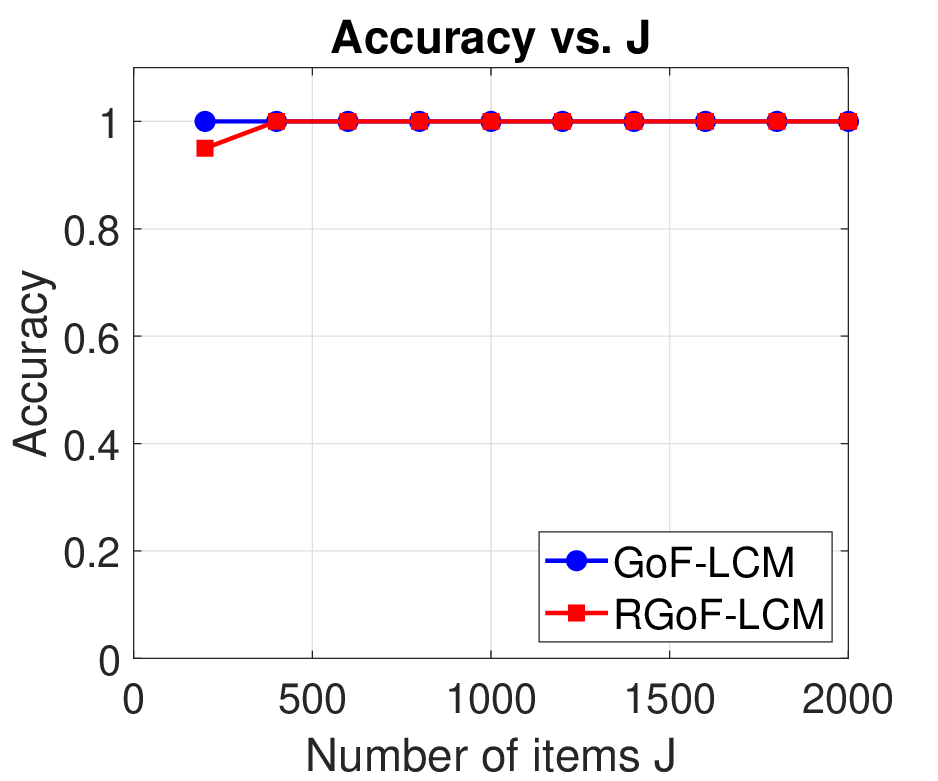}}
\caption{Accuracy of GoF‑LCM and RGoF‑LCM $K=8$, $N=600$, $\delta=0.3$, with varying \(J\).}
\label{fig:exp6}
\end{figure}

These results show that the proposed sequential testing procedures are robust to the dimension of the item space. Although Assumption~\ref{ass:A4} requires \(J=o(N)\) for the theoretical derivations, the actual performance does not deteriorate even when \(J\) significantly exceeds \(N\). For example, at \(J=2000\) we have \(\frac{JK\log(JK)}{N} \approx 258\), which is far from zero, yet the algorithms still recover the true number of classes with probability one. This suggests that the growth condition in Assumption~\ref{ass:A4} is a sufficient condition imposed by the proof technique and may not be necessary for the consistency of the algorithms. This empirical finding opens the possibility of relaxing the growth conditions in future theoretical work.
\subsection{Real data example}
We illustrate the proposed sequential testing procedures using a publicly available dataset from the Open-Source Psychometrics Project. The data file `randomnumber.zip` can be downloaded from \url{https://openpsychometrics.org/_rawdata/}. The survey consisted of a cognitive task (generating random numbers) followed by a standard Big Five Personality Test (BFPT).  The test comprises \(J = 50\) items, each rated on a six‑point scale with integer values from \(0\) to \(5\). For each item, \(0\) indicates the lowest level of agreement and \(5\) the highest. The dataset contains responses from \(N = 1369\) individuals, yielding a response matrix \(R \in \{0,1,\dots,5\}^{1369 \times 50}\).

\begin{figure}[htbp]
\centering
\resizebox{\columnwidth}{!}{
{\includegraphics[width=3\textwidth]{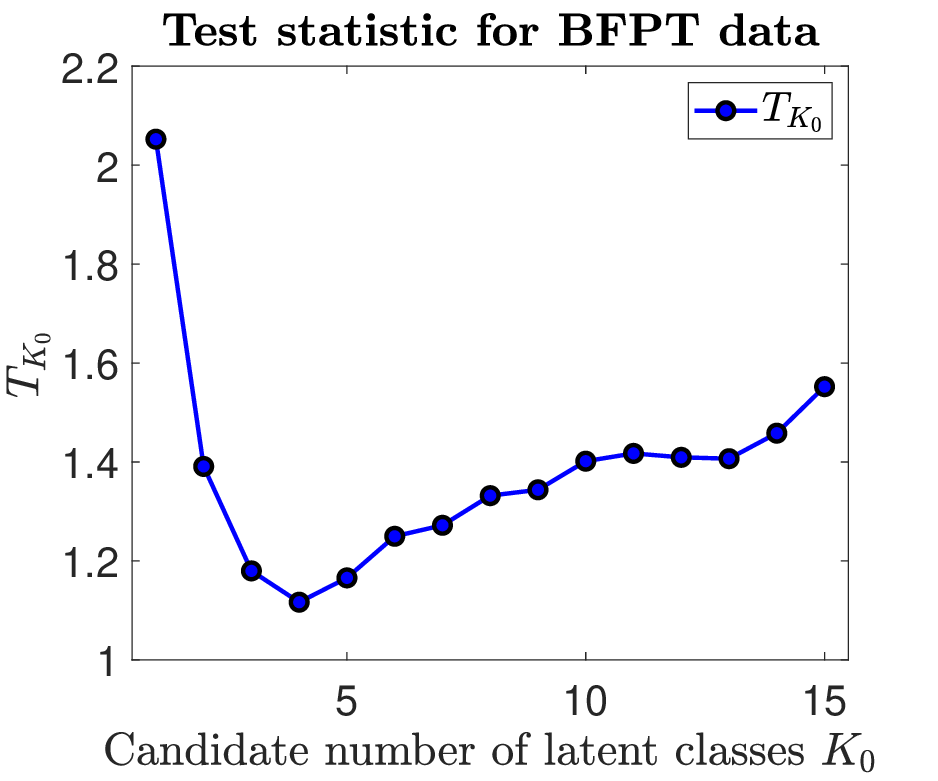}}
{\includegraphics[width=3\textwidth]{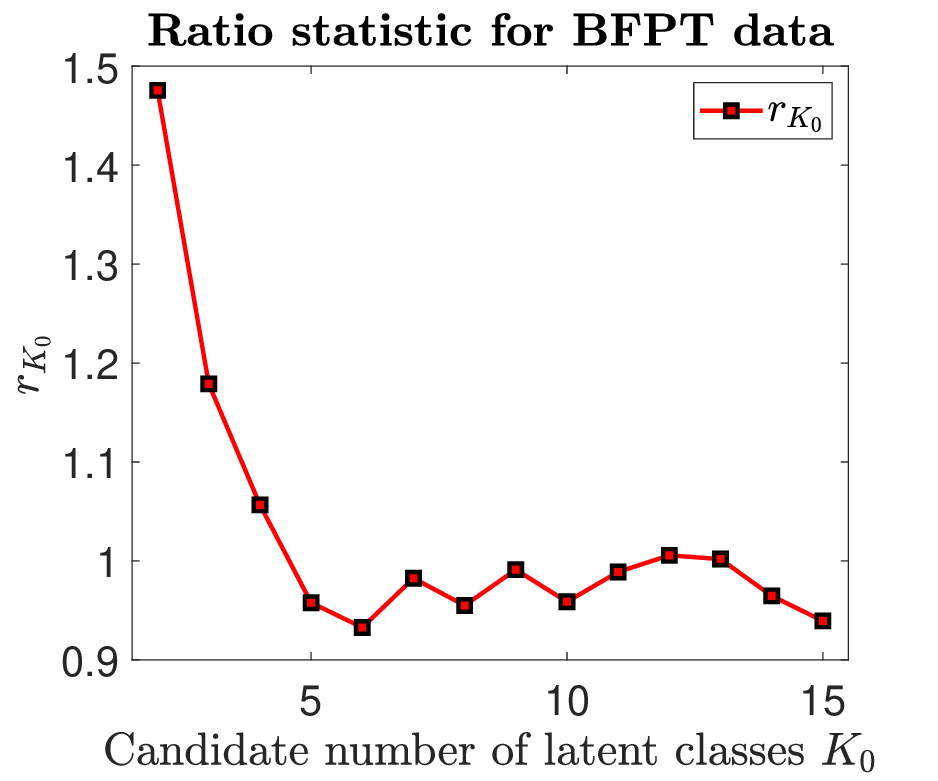}}
}
\caption{Test statistic $T_{K_0}$ (left) and ratio $r_{K_0}=|T(K_0-1)/T(K_0)|$ (right) versus candidate number of latent classes $K_0$ for the BFPT dataset.}
\label{fig:BFPT}
\end{figure}

The left panel of Figure~\ref{fig:BFPT} plots the test statistic $T_{K_0}$ for the BFPT data ($N=1369$, $J=50$). All values lie between $1$ and $2.2$, well above GoF-LCM's default threshold $\tau_N = N^{-1/5} \approx 0.24$. Consequently, GoF-LCM never encounters a $K_0$ satisfying $T_{K_0} < \tau_N$ and therefore runs to the maximum candidate $K_{\max} = \lfloor \sqrt{N/\log(N+J)} \rfloor \approx 13$, producing $\hat{K}=13$ by default. Thus, GoF-LCM yields an estimate that is likely too large and provides little insight into the underlying structure.

The right panel displays the ratio $r_{K_0}=|T_{K_0-1}/T_{K_0}|$. It attains a clear peak at $K_0=2$ ($r_2\approx1.47$), while for $K_0\ge3$ the ratios fluctuate between $0.9$ and $1.2$ without further structure. This pattern indicates that the relative change in $T_{K_0}$ is largest when moving from $K_0=1$ to $K_0=2$, suggesting that a two-class model captures the dominant heterogeneity in the data. Such a profile is exactly what RGoF-LCM is designed to exploit. 

Theorem~\ref{thm:ratio_consistency} requires the stopping threshold $\gamma_N$ to satisfy $\gamma_N\to\infty$ and $\gamma_N = o(\sqrt{N/\log J})$; any such sequence is admissible. The default $\gamma_N = \log N \approx 7.22$ exceeds $r_2$, so the formal sequential procedure does not stop at $K_0=2$. Nevertheless, the pronounced maximum provides valuable diagnostic information: it suggests that a two-class model offers the most proper description of the BFPT data, even though the default threshold is too conservative to trigger a stop. 

This observation highlights a key difference between the two approaches. GoF-LCM relies on an absolute threshold and fails when model assumptions are violated. RGoF-LCM, by focusing on relative improvements, can reveal dominant structural transitions through the shape of $r_{K_0}$—even when all $T_{K_0}$ values are inflated. In Figure~\ref{fig:BFPT}, this leads to the clear peak at $K_0=2$, a conclusion that would be missed by examining $T_{K_0}$ alone or by a mechanical application of the default threshold rule.

From a broader perspective, these results illustrate the value of diagnostic tools sensitive to relative rather than absolute fit. GoF-LCM enjoys clean asymptotic theory under correct specification, but RGoF-LCM demonstrates greater empirical robustness in practice. For the BFPT data, the evidence supports $\hat{K}=2$ as a reasonable description, a conclusion that emerges clearly from the ratio profile.
\section{Conclusion}\label{sec:conclusion}
We develop a goodness-of-fit framework for determining the number of latent classes in ordinal categorical data. The framework introduces a test statistic, exhibiting a sharp dichotomous behavior that leads to consistent sequential estimation. By transforming a challenging model selection problem into simple sequential tests, the proposed methods are both computationally efficient and theoretically principled. Simulation studies and real‑data applications demonstrate their accuracy and reliability.

Several directions for future research emerge from this work. Relaxing the current modeling assumptions is a natural extension. The binomial response assumption could be generalized to accommodate other exponential‑family distributions, such as Poisson for count data or multinomial for polytomous responses. The boundedness and separation conditions might be weakened through more refined concentration arguments. Extending the framework to richer latent structures presents another important direction. Mixed‑membership formulations, such as Grade‑of‑Membership (GoM) models \citet{woodbury1978mathematical}, allow individuals partial membership across multiple latent profiles. Degree heterogeneous latent class models \citet{lyu2025degree} (Dh-LCM) introduce additional individual‑specific parameters to capture varying levels of activity. For these more complex models, a fundamental first step is to estimate the number of latent classes \(K\), which remains largely unexplored and could build upon the sequential testing framework developed in this work, though adapting them would require fundamental reconsideration of the residual concept. Hierarchical latent structures, where classes themselves are organized into higher‑order groupings, pose a further challenge that may necessitate layered testing procedures. Finally, extending to dynamic settings, such as latent transition models for longitudinal data, would enable studying how class memberships evolve over time. Pursuing these directions will promisingly extend our framework to more complex data environments.
\section*{CRediT authorship contribution statement}
\textbf{Huan Qing} is the sole author of this article.

\section*{Declaration of competing interest}
The author declares no competing interests.

\section*{Data availability}
Data and code will be made available on request.


\appendix
\section{Technical proofs}\label{app:proofs}
\subsection{Proof of Lemma~\ref{lem:ideal_spectral}}
\begin{proof}
For any fixed $(i,j)$, we have
\[
\mathbb{E}[R^*(i,j)] = \frac{\mathbb{E}[R(i,j)] - \mathcal{R}(i,j)}{\sqrt{N\mathcal{V}(i,j)}} = 0.
\]

The variance is
\[
\operatorname{Var}(R^*(i,j)) = \frac{\operatorname{Var}(R(i,j))}{N\mathcal{V}(i,j)}.
\]

Since $\operatorname{Var}(R(i,j)) = \mathcal{V}(i,j)$, we obtain
\[
\operatorname{Var}(R^*(i,j)) = \frac{\mathcal{V}(i,j)}{N\mathcal{V}(i,j)} = \frac{1}{N}.
\]

By Assumption \ref{ass:A1}, we have $p_{ij} := \frac{\mathcal{R}(i,j)}{M} \in [\delta, 1-\delta]$. Then
\[
\mathcal{V}(i,j) = \mathcal{R}(i,j)\left(1 - \frac{\mathcal{R}(i,j)}{M}\right) = M\, p_{ij}(1-p_{ij}).
\]

The function $g(p) = p(1-p)$ is continuous and concave on $[\delta, 1-\delta]$. Since $g(p)$ is symmetric about $p=1/2$ and $\delta \le 1/2$, its minimum on $[\delta, 1-\delta]$ is attained at both endpoints, i.e., 
\[
g(p) \ge g(\delta) = \delta(1-\delta) \quad \text{for all } p \in [\delta, 1-\delta].
\]

Consequently, we have
\begin{equation*}
\mathcal{V}(i,j) \ge M\,\delta(1-\delta) \quad \text{for all } i\in[N], j\in[J].
\end{equation*}

Because $R(i,j)$ takes values in $\{0,1,\dots,M\}$, we have
\[
|R(i,j) - \mathcal{R}(i,j)| \le M.
\]

Hence, we get
\[
|R^*(i,j)| = \frac{|R(i,j) - \mathcal{R}(i,j)|}{\sqrt{N\mathcal{V}(i,j)}} \le \frac{M}{\sqrt{N \cdot M\delta(1-\delta)}} = \sqrt{\frac{M}{N\delta(1-\delta)}}.
\]

Hence, the maximum entrywise magnitude (denoted by $\sigma_*$ in matrix concentration inequalities) satisfies
\begin{equation*}
\sigma_* := \max_{i,j} \|R^*(i,j)\|_\infty \le \sqrt{\frac{M}{N\delta(1-\delta)}}.
\end{equation*}

For each row $i \in [N]$, we have
\[
\sum_{j=1}^{J} \mathbb{E}\bigl[(R^*(i,j))^2\bigr] = \sum_{j=1}^{J} \frac{1}{N} = \frac{J}{N},
\]
which gives
\[
\sigma_1 := \max_{i \in [N]} \sqrt{\sum_{j=1}^{J} \mathbb{E}\bigl[(R^*(i,j))^2\bigr]} = \sqrt{\frac{J}{N}}.
\]

For each column $j \in [J]$, we have
\[
\sum_{i=1}^{N} \mathbb{E}\bigl[(R^*(i,j))^2\bigr] = \sum_{i=1}^{N} \frac{1}{N} = 1,
\]
which gives
\[
\sigma_2 := \max_{j \in [J]} \sqrt{\sum_{i=1}^{N} \mathbb{E}\bigl[(R^*(i,j))^2\bigr]} = 1.
\]

The following lemma is obtained from the statements after Corollary 3.12 of \citep{Afonso2016}. 
\begin{lem}\label{BanderiaCor}
Let $X$ be any $n_1 \times n_2$ random matrix with independent entries satisfying $\mathbb{E}[X_{ij}]=0$. Define
\[
\sigma_1 = \max_{i} \sqrt{\sum_{j} \mathbb{E}[X_{ij}^2]}, \quad
\sigma_2 = \max_{j} \sqrt{\sum_{i} \mathbb{E}[X_{ij}^2]}, \quad
\sigma_* = \max_{i,j} \|X_{ij}\|_\infty.
\]

Then for any $0 < \eta \le \frac{1}{2}$, there exists a constant $C_\eta > 0$ such that for all $t \ge 0$,
\[
\mathbb{P}\bigl( \|X\| \ge (1+\eta)(\sigma_1 + \sigma_2) + t \bigr) \le (n_1+n_2) \exp\left( -\frac{t^2}{C_\eta \sigma_*^2} \right).
\]
\end{lem}

We apply Lemma \ref{BanderiaCor} with $X = R^*$, $n_1 = N$, $n_2 = J$, and $\sigma_1$, $\sigma_2$, $\sigma_*$ as computed before. Substituting these values, for any $t \ge 0$, we obtain 
\begin{equation*}
\mathbb{P}\left( \|R^*\| \ge (1+\eta)\left(\sqrt{\frac{J}{N}} + 1\right) + t \right) \le (N+J) \exp\left( -\frac{t^2 \delta(1-\delta) N}{C_\eta M} \right).
\end{equation*}

Let $\gamma > 2$ be a fixed constant (e.g., $\gamma = 3$). Choose
\(t = \sqrt{ \frac{\gamma C_\eta M \log(N+J)}{\delta(1-\delta) N} }.\)
Substituting this $t$ into the right-hand side of the above inequality gives
\begin{align*}
(N+J) \exp\left( -\frac{t^2 \delta(1-\delta) N}{C_\eta M} \right)
&= (N+J) \exp\left( -\frac{ \gamma C_\eta M \log(N+J) }{\delta(1-\delta) N} \cdot \frac{\delta(1-\delta) N}{C_\eta M} \right)= (N+J)^{1-\gamma}.
\end{align*}

Since $\gamma > 2$, $(N+J)^{1-\gamma} = o(1)$ as $N\to \infty$. Therefore, with probability at least $1 - o(1)$, we have
\[
\|R^*\| \le (1+\eta)\left(\sqrt{\frac{J}{N}} + 1\right) + \sqrt{ \frac{\gamma C_\eta M \log(N+J)}{\delta(1-\delta) N} }.
\]

Because the term $\sqrt{ \frac{\gamma C_\eta M \log(N+J)}{\delta(1-\delta) N} } = O\left( \sqrt{\frac{\log N}{N}} \right) = o(1)$ and the parameter $\eta$ can be chosen arbitrarily small, we obtain that with probability approaching 1,
\[
\|R^*\| \le 1 + \sqrt{\frac{J}{N}} + o(1).
\]

Thus, for any fixed $\epsilon > 0$, we have
\[
\lim_{N\to\infty} \mathbb{P}\left( \|R^*\| \le 1 + \sqrt{\frac{J}{N}} + \epsilon \right) = 1.
\]
\end{proof}
\subsection{Proof of Lemma \ref{lem:perturbation}}
\begin{proof}
We prove this lemma via five parts.

\vspace{2mm}
\noindent\textbf{Part 1:  A high‑probability event.}
By Lemma~\ref{lem:param_consistency} and the fact that a finite intersection of events with probability tending to 1 still has probability tending to 1, there exists an event $\mathcal{F}_N$ with $\mathbb{P}(\mathcal{F}_N)\to1$ such that on $\mathcal{F}_N$ the following hold simultaneously:
\[
\begin{aligned}
&\text{(i)}\quad \hat Z = Z\Pi;\\[2mm]
&\text{(ii)}\quad \sum_{k=1}^{K}\sum_{j=1}^{J}\bigl(\hat\Theta(j,\pi(k))-\Theta(j,k)\bigr)^2 \le C_1\,\frac{JK^2}{N};\\[2mm]
&\text{(iii)}\quad \max_{i\in[N],j\in[J]}\bigl|\hat{\mathcal R}(i,j)-\mathcal R(i,j)\bigr| \le C_2\sqrt{\frac{K\log(JK)}{N}};\\[2mm]
&\text{(iv)}\quad \hat{\mathcal V}(i,j)\ge v_{\min}\quad\forall i,j;\\[2mm]
&\text{(v)}\quad \mathcal V(i,j)\ge\delta(1-\delta)M\quad\forall i,j .
\end{aligned}
\]
(Here $\mathcal V(i,j)=\mathcal R(i,j)(1-\mathcal R(i,j)/M)$, $\hat{\mathcal V}(i,j)=\hat{\mathcal R}(i,j)(1-\hat{\mathcal R}(i,j)/M)$.)  
All subsequent estimates are performed pointwise on $\mathcal{F}_N$ and are therefore deterministic on this event.

\vspace{2mm}
\noindent\textbf{Part 2:  Decomposition of the difference matrix.}
Set $\Delta:=\tilde R-R^*\in\mathbb{R}^{N\times J}$.  
On $\mathcal{F}_N$, $\hat{\mathcal V}(i,j)>0$ by (iv), so the definition in Equation \eqref{eq:norm_res} reduces to the standard case with a positive denominator. For a fixed pair $(i,j)$, write
\[
A:=R(i,j),\qquad p:=\mathcal R(i,j),\qquad \hat p:=\hat{\mathcal R}(i,j),\qquad 
v:=\mathcal V(i,j),\qquad \hat v:=\hat{\mathcal V}(i,j).
\]

Then, we have
\[
\Delta(i,j)=\frac{A-\hat p}{\sqrt{N\hat v}}-\frac{A-p}{\sqrt{Nv}}
          =(A-p)\Bigl(\frac1{\sqrt{N\hat v}}-\frac1{\sqrt{Nv}}\Bigr)+\frac{p-\hat p}{\sqrt{N\hat v}}.
\]

Define two $N\times J$ matrices $E_1,E_2$ entrywise by
\[
E_1(i,j):=(A-p)\Bigl(\frac1{\sqrt{N\hat v}}-\frac1{\sqrt{Nv}}\Bigr),\qquad 
E_2(i,j):=\frac{p-\hat p}{\sqrt{N\hat v}}.
\]

Thus $\Delta=E_1+E_2$ on $\mathcal{F}_N$.

\vspace{2mm}
\noindent\textbf{Part 3:  Spectral norm of $E_2$.}
Because $\hat p=\hat{\mathcal R}(i,j)=\hat\Theta(j,\pi(\ell(i)))$ is constant on each true class, $E_2(i,j)$ depends only on the class of subject $i$ and the item $j$.  
Let $\mathcal C_k:=\{i:\ell(i)=k\}$ for $k\in[K]$.  On $\mathcal{F}_N$, for $i\in\mathcal C_k$, we have
\[
\hat{\mathcal V}(i,j)=\hat\Theta(j,\pi(k))\Bigl(1-\frac{\hat\Theta(j,\pi(k))}{M}\Bigr)=:\hat v_k(j),
\]
which is independent of $i\in\mathcal C_k$ and, by (iv), satisfies $\hat v_k(j)\ge v_{\min}$.  
Hence for $i\in\mathcal C_k$, we have
\[
E_2(i,j)=\frac{\Theta(j,k)-\hat\Theta(j,\pi(k))}{\sqrt{N\hat v_k(j)}}.
\]

Construct $V\in\mathbb{R}^{J\times K}$ by $V(j,k):=\dfrac{\Theta(j,k)-\hat\Theta(j,\pi(k))}{\sqrt{N\hat v_k(j)}}$.  
Since $Z(i,k)=\mathbf{1}_{\{i\in\mathcal C_k\}}$, we have $(ZV^\top)_{ij}=V(j,\ell(i))=E_2(i,j)$. Thus, we get $E_2=ZV^\top$ on $\mathcal{F}_N$.
 
$Z^\top Z=\operatorname{diag}(N_1,\dots,N_K)$ gives $\|Z\|=\sqrt{\max_k N_k}=:\sqrt{N_{\max}}$.  
Assumption~\ref{ass:A2} gives $c_0N/K\le N_k\le N/(c_0K)$, so $N_{\max}\le N/(c_0K)$ and
\[
\|Z\|\le\sqrt{\frac{N}{c_0K}}.
\]

Using $\hat v_k(j)\ge v_{\min}$ gets
\[
\|V\|_F^2=\sum_{k=1}^K\sum_{j=1}^J\frac{\bigl(\Theta(j,k)-\hat\Theta(j,\pi(k))\bigr)^2}{N\hat v_k(j)}
\le\frac1{Nv_{\min}}\sum_{k=1}^K\sum_{j=1}^J\bigl(\Theta(j,k)-\hat\Theta(j,\pi(k))\bigr)^2.
\]

On $\mathcal{F}_N$, (ii) bounds the double sum by $C_1JK^2/N$. Thus, we have
\[
\|V\|_F^2\le\frac1{Nv_{\min}}\cdot C_1\frac{JK^2}{N}=C_1v_{\min}^{-1}\frac{JK^2}{N^2},\qquad
\|V\|_F\le\sqrt{C_1v_{\min}^{-1}}\;\frac{K\sqrt J}{N}.
\]

Submultiplicativity of the spectral norm gives
\[
\|E_2\|\le\|Z\|\,\|V\|\le\|Z\|\,\|V\|_F
\le\sqrt{\frac{N}{c_0K}}\;\sqrt{C_1v_{\min}^{-1}}\;\frac{K\sqrt J}{N}
=\underbrace{\sqrt{\frac{C_1}{c_0v_{\min}}}}_{=:C_{E_2}}\;\sqrt{\frac{JK}{N}}.
\]

By Assumption \ref{ass:A4} , we have $\sqrt{JK/N}=o(1)$.  Consequently, for any $\varepsilon>0$, $\exists N_0$ s.t. $\forall N\ge N_0$, on $\mathcal{F}_N$ we have $\|E_2\|<\varepsilon$, which gives $\|E_2\|=o(1)$ on $\mathcal{F}_N$.

\vspace{2mm}
\noindent\textbf{Part 4:  Spectral norm of $E_1$.}
Set $f(x)=1/\sqrt{x}$ for $x>0$. We have $f'(x)=-\frac12x^{-3/2}$ and $|f'(x)|=\frac12x^{-3/2}$. Fix $(i,j)$.  On $\mathcal{F}_N$, $\hat v\ge v_{\min}$ by (iv).  From (v), $v\ge\delta(1-\delta)M$.  
Because $\delta\le\frac12$, one has $\delta(1-\delta)\ge\delta^2/4$, hence $v\ge v_{\min}$ as well.  
Thus $\hat v,v\in[v_{\min},\infty)$.

Apply the mean value theorem to $f$ on the interval between $\hat v$ and $v$: there exists $\xi$ strictly between $\hat v$ and $v$ such that
\[
\frac1{\sqrt{\hat v}}-\frac1{\sqrt v}=f'(\xi)(\hat v-v).
\]

Since $\xi\ge v_{\min}$, we have
\[
\Bigl|\frac1{\sqrt{\hat v}}-\frac1{\sqrt v}\Bigr|=|f'(\xi)|\,|\hat v-v|\le\frac12v_{\min}^{-3/2}\,|\hat v-v|.
\]

Define $h(x)=x(1-x/M)$ for $x\in[0,M]$. Then, we have $\hat v=h(\hat p)$, $v=h(p)$.  
$|h'(x)|=|1-2x/M|\le1$ for all $x\in[0,M]$, so the mean value theorem yields
\[
|\hat v-v|=|h(\hat p)-h(p)|\le\Bigl(\max_{z\in[0,M]}|h'(z)|\Bigr)\,|\hat p-p|\le|\hat p-p|.
\]

Thus, we have
\[
\Bigl|\frac1{\sqrt{\hat v}}-\frac1{\sqrt v}\Bigr|\le\frac12v_{\min}^{-3/2}\,|\hat p-p|.
\]

Because $R(i,j)\in\{0,1,\dots,M\}$ and $\mathcal R(i,j)\in[0,M]$, we have $|A-p|\le M$.  Therefore, we get
\[
|E_1(i,j)|\le|A-p|\;\frac1{\sqrt N}\;\Bigl|\frac1{\sqrt{\hat v}}-\frac1{\sqrt v}\Bigr|
\le\frac{M}{2\sqrt N\,v_{\min}^{3/2}}\;|\hat p-p|.
\]

On $\mathcal{F}_N$, (iii) gives $\max_{i,j}|\hat p-p|\le C_2\sqrt{\frac{K\log(JK)}{N}}$.  Hence, for every $(i,j)$, we have
\[
|E_1(i,j)|\le\frac{M C_2}{2v_{\min}^{3/2}}\;\sqrt{\frac{K\log(JK)}{N^2}}
=:\underbrace{\frac{M C_2}{2v_{\min}^{3/2}}}_{=:C_{E_1}}\;\sqrt{\frac{K\log(JK)}{N^2}}.
\]

For any matrix, $\|E_1\|\le\|E_1\|_F$ and
\(\|E_1\|_F=\sqrt{\sum_{i=1}^N\sum_{j=1}^JE_1(i,j)^2}\le\sqrt{NJ}\,\max_{i,j}|E_1(i,j)|.\)
Consequently, on $\mathcal{F}_N$, we have
\[
\|E_1\|\le\sqrt{NJ}\;C_{E_1}\sqrt{\frac{K\log(JK)}{N^2}}
=C_{E_1}\sqrt{\frac{JK\log(JK)}{N}}.
\]

Assumption~\ref{ass:A4} directly gives $JK\log(JK)/N\to0$, so $\sqrt{JK\log(JK)/N}=o(1)$.  
Thus for any $\varepsilon>0$, $\exists N_0$ s.t. $\forall N\ge N_0$, on $\mathcal{F}_N$ we have $\|E_1\|<\varepsilon$.

\vspace{2mm}
\noindent\textbf{Part 5:  Conversion to $o_P(1)$.}
On the event $\mathcal{F}_N$, the triangle inequality yields
\[
\|\Delta\|\le\|E_1\|+\|E_2\|
\le C_{E_1}\sqrt{\frac{JK\log(JK)}{N}}+C_{E_2}\sqrt{\frac{JK}{N}}.
\]

Both terms are $o(1)$. Therefore, for every $\varepsilon>0$ there exists $N_0$ (depending only on $\varepsilon$ and the model constants) such that for all $N\ge N_0$ and every realization belonging to $\mathcal{F}_N$, $\|\Delta\|<\varepsilon$. Recall $\mathbb{P}(\mathcal{F}_N)\to1$.  For an arbitrary $\varepsilon>0$, we have
\[
\mathbb{P}\bigl(\|\Delta\|>\varepsilon\bigr)
=\mathbb{P}\bigl(\{\|\Delta\|>\varepsilon\}\cap\mathcal{F}_N\bigr)+\mathbb{P}\bigl(\{\|\Delta\|>\varepsilon\}\cap\mathcal{F}_N^c\bigr)
\le\mathbb{P}\bigl(\{\|\Delta\|>\varepsilon\}\cap\mathcal{F}_N\bigr)+\mathbb{P}(\mathcal{F}_N^c).
\]

From the previous results, we know that for sufficiently large $N$ the set $\{\|\Delta\|>\varepsilon\}\cap\mathcal{F}_N$ is empty. Hence its probability is $0$. Thus, we have $\limsup_{N\to\infty}\mathbb{P}(\|\Delta\|>\varepsilon)\le\lim_{N\to\infty}\mathbb{P}(\mathcal{F}_N^c)=0$, and since $\varepsilon>0$ was arbitrary, $\|\Delta\|\xrightarrow{\mathbb{P}}0$, which is exactly Equation \eqref{eq:L3}.  
\end{proof}
\subsection{Proof of Theorem \ref{thm:null}}\label{proof:thm1}
\begin{proof}
For any two real matrices of identical dimensions, Weyl's inequality for singular values asserts
\[
|\sigma_1(A) - \sigma_1(B)| \le \|A-B\|.
\]

Applying it with \(A = \tilde R\) and \(B = R^{*}\) (both are \(N\times J\) matrices) yields the deterministic bound
\[
\sigma_1(\tilde R) \le \sigma_1(R^{*}) + \|\tilde R - R^{*}\|. 
\]

Fix an arbitrary \(\varepsilon > 0\), we have
\[
\bigl\{\sigma_1(\tilde R) > 1 + \sqrt{J/N} + \varepsilon\bigr\}
\subseteq
\bigl\{\sigma_1(R^{*}) > 1 + \sqrt{J/N} + \varepsilon/2\bigr\}
\;\cup\;
\bigl\{\|\tilde R - R^{*}\| > \varepsilon/2\bigr\},
\]
which yields
\[
\mathbb{P}\Bigl(\sigma_1(\tilde R) > 1 + \sqrt{J/N} + \varepsilon\Bigr)
\le
\mathbb{P}\Bigl(\sigma_1(R^{*}) > 1 + \sqrt{J/N} + \varepsilon/2\Bigr)
+ \mathbb{P}\Bigl(\|\tilde R - R^{*}\| > \varepsilon/2\Bigr). 
\]

By Lemmas \ref{lem:ideal_spectral} and \ref{lem:perturbation}, we have that both probabilities on the right‑hand side of the above equation converge to zero as \(N\to\infty\).  Consequently, we get
\[
\lim_{N\to\infty} \mathbb{P}\Bigl(\sigma_1(\tilde R) > 1 + \sqrt{J/N} + \varepsilon\Bigr) = 0.
\]

Recall the definition of \(T_{K_0} \) is 
\(T_{K_0} = \sigma_1(\tilde R) - \bigl(1 + \sqrt{J/N}\bigr),\)
so for any \(\varepsilon > 0\), we have
\[
\{T_{K_0} > \varepsilon\} = \{\sigma_1(\tilde R) > 1 + \sqrt{J/N} + \varepsilon\},
\]
which immediately implies
\[
\lim_{N\to\infty} \mathbb{P}\bigl(T_{K_0} > \varepsilon\bigr) = 0.
\]

Hence, we have 
\[
\mathbb{P}\bigl(T_{K_0} < \varepsilon\bigr) \longrightarrow 1,
\]
as \(N\to\infty\).  This completes the proof of Theorem~\ref{thm:null}.
\end{proof}
\subsection{Proof of Theorem \ref{thm:alt}}\label{proof:thm2}
\begin{proof}[Proof of Theorem~\ref{thm:alt}]
We work conditionally on the output of the classifier \(\mathcal{M}\). Lemma~\ref{lem:underfit_lower} holds deterministically for every possible \(\hat Z\) whenever \(K_0<K\).

\vspace{3mm}
\noindent\textbf{Step 1.  Deterministic objects from Lemma~\ref{lem:underfit_lower}.}
Because \(K_0<K\), Lemma~\ref{lem:underfit_lower} guarantees the existence of two distinct true classes \(k_1,k_2\in[K]\), subsets
\[
S_1\subseteq\mathcal{C}_{k_1},\; S_2\subseteq\mathcal{C}_{k_2},\qquad S:=S_1\cup S_2\subseteq[N],\qquad T\subseteq[J],
\]
and unit vectors \(u\in\mathbb{R}^{|S|}, v\in\mathbb{R}^{|T|}\) with the specific forms constructed in the lemma (see its proof) such that the following deterministic properties hold:
\[
|S_1|,|S_2|\ge\frac{c_0 N}{K K_0},\qquad |S|=|S_1|+|S_2|\le\frac{2N}{c_0 K},\qquad |T|\ge c_1 J,
\]
\[
u_i=
\begin{cases}
\alpha/\sqrt{|S_1|},& i\in S_1,\\[2mm]
-\beta/\sqrt{|S_2|},& i\in S_2,
\end{cases}\qquad 
v_j=\frac{s_j}{\sqrt{|T|}},\quad s_j=\operatorname{sign}\bigl(\Theta(j,k_1)-\Theta(j,k_2)\bigr),
\]
where \(\alpha=\sqrt{|S_2|}/\sqrt{|S_1|+|S_2|},\ \beta=\sqrt{|S_1|}/\sqrt{|S_1|+|S_2|}\).
Moreover, from the construction in Lemma~\ref{lem:underfit_lower}, we have the deterministic lower bound
\begin{equation}\label{eq:uvMv_bound}
u^\top (\mathcal{R}-\hat{\mathcal{R}})_{S,T}\,v \;=\; \frac{\gamma_0}{\sqrt{|T|}}\sum_{j\in T}|\Theta(j,k_1)-\Theta(j,k_2)|
\;\ge\; c\,\zeta\,\frac{\sqrt{NJ}}{\sqrt{K}K_0},
\end{equation}
where \(\gamma_0=\sqrt{|S_1||S_2|/(|S_1|+|S_2|)}\).
All subjects in \(S\) are assigned by \(\hat Z\) to the same estimated class, say \(\kappa\in[K_0]\);
consequently for every \(j\in T\) the fitted value \(\hat{\mathcal{R}}(i,j)=\hat\Theta(j,\kappa)\) is constant on \(S\times\{j\}\).

\vspace{3mm}
\noindent\textbf{Step 2.  Scaling factors.}
For \(j\in T\), define the following scaling factors:
\[
\gamma_j:=\frac{1}{\sqrt{N\hat{\mathcal{V}}(i,j)}},\qquad i\in S,
\]
where \(\hat{\mathcal{V}}(i,j)=\hat{\mathcal{R}}(i,j)\bigl(1-\hat{\mathcal{R}}(i,j)/M\bigr)\) and the value does not depend on the particular \(i\in S\).
Set \(D:=\operatorname{diag}(\gamma_j)_{j\in T}\in\mathbb{R}^{|T|\times|T|}\).
Then the submatrix of the practical normalized residual matrix satisfies
\[
\tilde{R}_{S,T}= (R-\hat{\mathcal{R}})_{S,T}\,D = \bigl(W+\mathscr{M}\bigr)D,
\]
where \(W:=(R-\mathcal{R})_{S,T}\) and \(\mathscr{M}:=(\mathcal{R}-\hat{\mathcal{R}})_{S,T}\).

\vspace{3mm}
\noindent\textbf{Step 3.  Three high‑probability events.}
We now construct three events whose probabilities tend to \(1\):
\begin{itemize}
\item\noindent\textit{Event \(\mathscr{E}_A\) – no zero denominator.}  
Define \(\mathscr{E}_A:=\{\gamma_j>0\ \text{for all }j\in T\}\).  
Fix a column \(j\in T\). Recall that for any \(i\in S\) we have \(\hat{\mathcal{R}}(i,j)=\hat\Theta(j,\kappa)\) (a constant over \(S\)) and \(\hat{\mathcal{V}}(i,j)=\hat{\mathcal{R}}(i,j)\bigl(1-\hat{\mathcal{R}}(i,j)/M\bigr)\).  
Since the function \(h(x)=x(1-x/M)\) satisfies \(h(x)=0\) iff \(x\in\{0,M\}\), we obtain  
\[
\gamma_j=\frac{1}{\sqrt{N\hat{\mathcal{V}}(i,j)}}>0 \quad\Longleftrightarrow\quad \hat\Theta(j,\kappa)\notin\{0,M\}.
\]

Because \(\hat\Theta(j,\kappa)=\frac{1}{|S|}\sum_{i\in S}R(i,j)\) (the sample mean of the \(|S|\) independent responses), the event \(\{\hat\Theta(j,\kappa)=0\}\) occurs only if every \(R(i,j)=0\); likewise \(\{\hat\Theta(j,\kappa)=M\}\) occurs only if every \(R(i,j)=M\). Hence \(\gamma_j=0\) (equivalently \(\hat\Theta(j,\kappa)\in\{0,M\}\)) is contained in the union of the two events “all responses in item \(j\) on the set \(S\) are \(0\)” and “all are \(M\)”.

Now we bound the probability of each of these two extreme events.  
Under Assumption~\ref{ass:A1}, for any true class \(k\in[K]\), we have \(\Theta(j,k)/M\in[\delta,1-\delta]\).  
Therefore, for any individual \(i\) (regardless of its true class), we have 
\[
\mathbb{P}\bigl(R(i,j)=0\bigr) = \Bigl(1-\frac{\Theta(j,\ell(i))}{M}\Bigr)^M \le (1-\delta)^M,
~~~\mathbb{P}\bigl(R(i,j)=M\bigr) = \Bigl(\frac{\Theta(j,\ell(i))}{M}\Bigr)^M \le (1-\delta)^M.
\]

The responses \(\{R(i,j):i\in S\}\) are mutually independent. Consequently, we have
\[
\mathbb{P}\bigl(\text{all }R(i,j)=0,\ i\in S\bigr) 
   = \prod_{i\in S}\mathbb{P}(R(i,j)=0) \le (1-\delta)^{M|S|},
\]
\[
\mathbb{P}\bigl(\text{all }R(i,j)=M,\ i\in S\bigr) 
   = \prod_{i\in S}\mathbb{P}(R(i,j)=M) \le (1-\delta)^{M|S|}.
\]

Applying the union bound gives, conditionally on the set \(S\),
\[
\mathbb{P}(\gamma_j=0\mid S) \le 2(1-\delta)^{M|S|}.
\]

From Lemma~\ref{lem:underfit_lower}, whenever \(K_0<K\), we have the deterministic lower bound \(|S|\ge \frac{2c_0N}{KK_0}\) (it holds for every realisation of \(\hat Z\)). Thus, we have \((1-\delta)^{M|S|}\le (1-\delta)^{2c_0M N/(KK_0)}\) almost surely.  
Taking expectations and using the fact that the bound is non‑random yields
   \[
   \mathbb{P}(\gamma_j = 0) \;=\; \mathbb{E}\bigl[\,\mathbb{P}(\gamma_j = 0 \mid S)\,\bigr]
   \;\le\; \mathbb{E}\left[2(1-\delta)^{\frac{2c_0 M N}{K K_0}}\right]
   \;=\; 2(1-\delta)^{\frac{2c_0 M N}{K K_0}},
   \]

Finally, applying the union bound over the columns in \(T\) (note that \(|T|\le J\)) yields
\begin{equation}\label{eq:EA_bound_union}
\mathbb{P}(\mathscr{E}_A^c) \le \sum_{j\in T}\mathbb{P}(\gamma_j=0) \le 2J\,(1-\delta)^{2c_0M N/(K K_0)}.
\end{equation}

We now prove that the right‑hand side of Equation \eqref{eq:EA_bound_union} tends to \(0\).  
Set \(\varrho := -\ln(1-\delta) > 0\) (since \(0<\delta\le 1/2\)).  Then \((1-\delta)^x = e^{-\varrho x}\) and Equation \eqref{eq:EA_bound_union} becomes
\begin{equation}\label{eq:EA_bound_exp}
\mathbb{P}(\mathscr{E}_A^c) \le 2J \exp\Bigl(-\varrho\cdot\frac{2c_0 M N}{K K_0}\Bigr).
\end{equation}

Because \(K_0 < K\), we have \(K K_0 \le K^2\) and consequently \(N/(K K_0) \ge N/K^2\).  Hence, we have
\begin{equation}\label{eq:EA_bound_final}
\mathbb{P}(\mathscr{E}_A^c) \le 2J \exp\Bigl(-\underbrace{2c_0 M \varrho}_{=:c'}\,\frac{N}{K^2}\Bigr).
\end{equation}

By Assumption~\ref{ass:A4}, for any positive constant \(\epsilon\) there exists an integer \(N_0(\epsilon)\) such that for all \(N \ge N_0(\epsilon)\),
\begin{equation}\label{eq:A5_M}
\frac{K^2 \log(N+J)}{N} \le \frac{1}{\epsilon},
\qquad\text{or equivalently}\qquad
\frac{N}{K^2} \ge \epsilon \log(N+J).
\end{equation}

We now fix a specific value of \(\epsilon\).  Choose \(\epsilon = \dfrac{2}{c'}\).  Then for all sufficiently large \(N\) (say \(N \ge N_0(2/c')\)), we have
\begin{equation}\label{eq:NK2_bound}
\frac{N}{K^2} \ge \frac{2}{c'} \log(N+J).
\end{equation}

Substituting Equation \eqref{eq:NK2_bound} into Equation \eqref{eq:EA_bound_final} gives
\[
\mathbb{P}(\mathscr{E}_A^c) \le 2J \exp\Bigl(-c'\cdot\frac{N}{K^2}\Bigr)
\le 2J \exp\Bigl(-c'\cdot\frac{2}{c'}\log(N+J)\Bigr)
= 2J (N+J)^{-2}\xrightarrow[N\to\infty]{} 0,
\]
which gives \(\mathbb{P}(\mathscr{E}_A)\to1\).

\item\noindent\textit{Event \(\mathscr{E}_B\) – estimated class mean stays away from boundaries.}  
Let \(\kappa\) be the estimated class containing \(S\).  
By Lemma~\ref{lem:underfit_lower}, \(n_\kappa:=|S|\ge 2c_0N/(KK_0)\) deterministically.  
For any \(j\in T\), \(\hat\Theta(j,\kappa)=\frac1{n_\kappa}\sum_{i\in S}R(i,j)\) is the average of \(n_\kappa\) independent \(\mathrm{Binomial}(M,\cdot)\) variables, each bounded in \([0,M]\).  By Assumption~\ref{ass:A1}, \(\mathbb{E}[R(i,j)]=\Theta(j,\ell(i))\in[\delta M,(1-\delta)M]\) for all \(i\in S\). Hence, we have \(\mathbb{E}[\hat\Theta(j,\kappa)]\in[\delta M,(1-\delta)M]\).  Hoeffding's inequality gives, for \(t=\delta M/2\),
\[
\mathbb{P}\Bigl(|\hat\Theta(j,\kappa)-\mathbb{E}[\hat\Theta(j,\kappa)]|\ge\frac{\delta M}{2}\Bigr)
\le 2\exp\Bigl(-\frac{2n_\kappa(\delta M/2)^2}{M^2}\Bigr)=2\exp\Bigl(-\frac{n_\kappa\delta^2}{2}\Bigr).
\]
Since \(n_\kappa\ge 2c_0N/(KK_0)\ge 2c_0N/K^2\) (because \(K_0<K\)) and the exponential function is decreasing, the right‑hand side is bounded by \(2\exp\bigl(-c_0\delta^2 N/K^2\bigr)\).  
If the deviation is less than \(\delta M/2\), then \(\hat\Theta(j,\kappa)\in[\delta M/2,(1-\delta/2)M]\) because \(\mathbb{E}[\hat\Theta(j,\kappa)]\in[\delta M,(1-\delta)M]\).  Define
\[
\mathscr{E}_B:=\bigl\{\hat\Theta(j,\kappa)\in[\delta M/2,(1-\delta/2)M]\ \text{for all }j\in T\bigr\}.
\]
The event \(\{\forall j:|\hat\Theta(j,\kappa)-\mathbb{E}[\hat\Theta(j,\kappa)]|<\delta M/2\}\) is contained in \(\mathscr{E}_B\). Therefore, we have
\[
\mathbb{P}(\mathscr{E}_B^c)\le \mathbb{P}\bigl(\exists j:|\hat\Theta(j,\kappa)-\mathbb{E}[\hat\Theta(j,\kappa)]|\ge\delta M/2\bigr)
\le |T|\cdot 2\exp\Bigl(-\frac{c_0\delta^2 N}{K^2}\Bigr)\le 2J\exp\Bigl(-\frac{c_0\delta^2 N}{K^2}\Bigr).
\]
Since \(K^2=o(N)\) and \(J=o(N)\) by Assumption~\ref{ass:A4}, the right‑hand side tends to \(0\); thus \(\mathbb{P}(\mathscr{E}_B)\to1\).

By Assumption~\ref{ass:A4}, we have \(\frac{N}{K^2}\gg\log N\) and \(J=o(N)\). Hence, for all sufficiently large \(N\), we have \(J\le N\). Therefore, we have
\[
2J\exp\Bigl(-\frac{c_0\delta^2 N}{K^2}\Bigr)\;\le\;2N\exp\Bigl(-\frac{c_0\delta^2 N}{K^2}\Bigr)\xrightarrow[N\to\infty]{} 0.
\]

Consequently, we obtain \(\mathbb{P}(\mathscr{E}_B^c)\to0\) and \(\mathbb{P}(\mathscr{E}_B)\to1\).

\item\noindent\textit{Event \(\mathscr{E}_\mathcal{E}\) – spectral norm bound for the noise matrix.}
Recall \(X = R - \mathcal{R}\in\mathbb{R}^{N\times J}\).
By Lemma~\ref{lem:random_control}, we have \(\|X\| = O_P(\sqrt{N}+\sqrt{J})\).
Assumption~\ref{ass:A4} implies \(J/N\to0\). Hence, we have \(\|X\| = O_P(\sqrt{N})\).
More concretely, from the proof of Lemma~\ref{lem:random_control}, we can take the constant
\(C_X := 2.5\sqrt{M}\) such that
\[
\mathbb{P}\bigl(\|X\| \ge C_X\sqrt{N}\bigr) \to 0,
\]
and consequently
\[
\mathbb{P}\bigl(\|X\| \le C_X\sqrt{N}\bigr) \to 1.
\]

Define
\[
\mathscr{E}_\mathcal{E} := \bigl\{\|X\| \le C_X\sqrt{N}\bigr\}.
\]
Then \(\mathbb{P}(\mathscr{E}_\mathcal{E}) \to 1\).
\end{itemize}
\vspace{3mm}
\noindent\textbf{Step 4.  Behaviour on the intersection \(\mathscr{E}_C:=\mathscr{E}_A\cap \mathscr{E}_B\cap\mathscr{E}_\mathcal{E}\).}
Because each of \(\mathscr{E}_A,\mathscr{E}_B,\mathscr{E}_\mathcal{E}\) has probability tending to \(1\), \(\mathbb{P}(\mathscr{E}_C)\to1\).
On the event \(\mathscr{E}_C\) we have the following deterministic bounds.

\begin{itemize}
\item[(i)]  \textbf{Lower bound for \(\gamma_j\).}
On the event \(\mathscr{E}_A\), \(\gamma_j>0\).  Since \(0\le\hat{\mathcal{R}}(i,j)\le M\) always implies \(\hat{\mathcal{V}}(i,j)\le M/4\), we have
\[
\gamma_j = \frac{1}{\sqrt{N\hat{\mathcal{V}}(i,j)}} \ge \frac{1}{\sqrt{N\cdot M/4}} = \frac{2}{\sqrt{NM}}\qquad(\forall j\in T).
\]

\item[(ii)] \textbf{Upper bound for \(\|D\|\).}
On the event \(\mathscr{E}_B\), \(\hat\Theta(j,\kappa)\in[\delta M/2,(1-\delta/2)M]\) for all \(j\in T\).
The function \(h(x)=x(1-x/M)\) satisfies \(h(x)\ge h(\delta M/2)=\delta(2-\delta)M/4\ge\delta^2M/4\) (since \(\delta\le1/2\)).
Hence, we have \(\hat{\mathcal{V}}(i,j)\ge\delta^2M/4\) and consequently
\[
\gamma_j \le \frac{1}{\sqrt{N\cdot\delta^2M/4}} = \frac{2}{\delta\sqrt{NM}},\qquad 
\|D\|=\max_{j\in T}\gamma_j\le\frac{2}{\delta\sqrt{NM}}.
\]

\item[(iii)] \textbf{Upper bound for \(\|WD\|\).}
Because \(W\) is a submatrix of \(X\), \(\|W\|\le\|X\|\).
On the event \(\mathscr{E}_\mathcal{E}\), \(\|X\|\le C_X\sqrt{N}=2.5\sqrt{MN}\).
Thus on \(\mathscr{E}_B\cap\mathscr{E}_\mathcal{E}\), we have
\[
\|WD\|\le\|W\|\|D\|\le 2.5\sqrt{MN}\cdot\frac{2}{\delta\sqrt{NM}}=\frac{5}{\delta}=C_{\mathrm{noise}}.
\]

\item[(iv)] \textbf{Lower bound for \(\|\mathscr{M}D\|\).}
As shown in the proof of Lemma~\ref{lem:underfit_lower}, for each column \(j\in T\) we have the deterministic identity
\[
u^\top\mathscr{M}_{:,j}v_j = \frac{\gamma_0}{\sqrt{|T|}}\,|\Theta(j,k_1)-\Theta(j,k_2)|\;\ge\;0,
\]
where \(\mathscr{M}_{:,j}\) denotes the \(j\)-th column of \(\mathscr{M}\).
Using (i) and Equation \eqref{eq:uvMv_bound} gives
\[
u^\top(\mathscr{M}D)v = \sum_{j\in T}\gamma_j\,(u^\top\mathscr{M}_{:,j})\,v_j
\ge \Bigl(\min_{j\in T}\gamma_j\Bigr)\,u^\top\mathscr{M}v
\ge \frac{2}{\sqrt{NM}}\cdot c\zeta\frac{\sqrt{NJ}}{\sqrt{K}K_0}
= C_{\mathrm{signal}}\,\frac{\sqrt{J}}{\sqrt{K}K_0}.
\]
Since \(\|u\|_2=\|v\|_2=1\), the variational characterization of the spectral norm yields
\[
\|\mathscr{M}D\|\ge |u^\top(\mathscr{M}D)v|\ge C_{\mathrm{signal}}\frac{\sqrt{J}}{\sqrt{K}K_0}.
\]
\end{itemize}

\vspace{3mm}
\noindent\textbf{Step 5.  Lower bound for \(\|\tilde{R}_{S,T}\|\) and for \(\sigma_1(\tilde{R})\).}
On the event \(\mathscr{E}_C\), the triangle inequality gives
\[
\|\tilde{R}_{S,T}\| = \|(\mathscr{M}+W)D\| \ge \|\mathscr{M}D\| - \|WD\|
\ge C_{\mathrm{signal}}\frac{\sqrt{J}}{\sqrt{K}K_0} - C_{\mathrm{noise}}.
\]

Because \(\tilde{R}\) is an \(N\times J\) matrix, \(\sigma_1(\tilde{R})=\|\tilde{R}\|\ge\|\tilde{R}_{S,T}\|\).
Hence on the event \(\mathscr{E}_C\), we have
\[
\sigma_1(\tilde{R}) \ge L_N,\qquad L_N:=C_{\mathrm{signal}}\frac{\sqrt{J}}{\sqrt{K}K_0}-C_{\mathrm{noise}}.
\]

\vspace{3mm}
\noindent\textbf{Step 6.  Positive constant lower bound for \(T_{K_0}\).}
Assumption \ref{ass:A6} asserts that for all sufficiently large \(N\),
\[
L_N \;\ge\; 1+3\eta_0.
\]

Assumption~\ref{ass:A4} guarantees \(\sqrt{J/N}\to0\). Therefore, there exists \(N_0\) such that for all \(N\ge N_0\),
\[
\sqrt{\frac{J}{N}} \;\le\; \eta_0.
\]

Now for any \(N\ge N_0\) and on the event \(\mathscr{E}_C\), we have
\[
T_{K_0}= \sigma_1(\tilde{R})-\Bigl(1+\sqrt{\frac{J}{N}}\Bigr)
\;\ge\; L_N - \Bigl(1+\sqrt{\frac{J}{N}}\Bigr)
\;\ge\; (1+3\eta_0) - (1+\eta_0) = 2\eta_0.
\]

Thus, we have
\[
\mathbb{P}\bigl(T_{K_0} > 2\eta_0\bigr) \;\ge\; \mathbb{P}(\mathscr{E}_C) \;\longrightarrow\; 1 \qquad \mathrm{as~}N\to\infty.
\]

\vspace{3mm}
\noindent\textbf{Step 7.  Implication for the testing procedure.}
If Algorithm~\ref{alg:gof_lcm} uses a threshold sequence \(\tau_N\) with \(\tau_N\to0\) (e.g.\ \(\tau_N=N^{-1/5}\)), then for all sufficiently large \(N\) we have \(\tau_N < 2\eta_0\).
Consequently,
\[
\mathbb{P}\bigl(T_{K_0} > \tau_N\bigr) \;\ge\; \mathbb{P}\bigl(T_{K_0} > 2\eta_0\bigr) \;\longrightarrow\; 1.
\]
Hence, under the alternative \(K_0<K\) and Assumption \ref{ass:A6}, the test statistic exceeds the threshold with probability tending to one, leading to a correct rejection of \(H_0\). This completes the proof of Theorem~\ref{thm:alt}.
\end{proof}
\subsection{Proof of Theorem \ref{thm:consistency}}\label{proof:thm3}
\begin{proof}
The proof is organized into four parts:  
(I) Decomposition of the error probability;  
(II) Control of \(\mathbb{P}(C)\) – the probability of not stopping at the true \(K\);  
(III) Control of \(\mathbb{P}(B)\) – the probability of stopping at an under‑specified \(K_0\);  
(IV) Asymptotic analysis under Assumption~\ref{ass:A4}.

\vspace{3mm}
\noindent\textbf{Part 1:  Events and error decomposition.}  
Define the events
\[
\mathcal{G}_A:=\{\hat K = K\},\qquad 
\mathcal{G}_B:=\bigl\{\exists\,K_0<K: T_{K_0}<\tau_N\bigr\},\qquad 
\mathcal{G}_C:=\{T_K\ge\tau_N\}.
\]

If neither \(\mathcal{G}_B\) nor \(\mathcal{G}_C\) occurs, then the algorithm does not stop at any \(K_0<K\) (because \(T_{K_0}\ge\tau_N\) for all such \(K_0\)) and it does stop at \(K_0=K\) (since \(T_K<\tau_N\)). Hence \(\mathcal{G}_A^c\subseteq \mathcal{G}_B\cup \mathcal{G}_C\) and consequently
\begin{equation}\label{eq:decomp_cons}
\mathbb{P}(\mathcal{G}_A^c)\le\mathbb{P}(\mathcal{G}_B)+\mathbb{P}(\mathcal{G}_C).
\end{equation}

Thus, it suffices to prove \(\mathbb{P}(\mathcal{G}_B)\to0\) and \(\mathbb{P}(\mathcal{G}_C)\to0\).

\vspace{3mm}
\noindent\textbf{Part 2:  Control of \(\mathbb{P}(\mathcal{G}_C)\) – correct specification.}  
When \(K_0=K\), \(\tilde R\) is the practical normalized residual matrix and \(R^*\) its ideal counterpart. By Weyl's inequality, we have
\[
\sigma_1(\tilde R)\le \sigma_1(R^*)+\|\tilde R-R^*\|,
\]
and therefore
\[
\{T_K\ge\tau_N\}=\{\sigma_1(\tilde R)\ge 1+\sqrt{J/N}+\tau_N\}
\subseteq\{\sigma_1(R^*)\ge 1+\sqrt{J/N}+\tau_N/2\}\cup\{\|\tilde R-R^*\|\ge\tau_N/2\}.
\]

\noindent\emph{Bound for \(\sigma_1(R^*)\).}  
From the proof of Lemma~\ref{lem:ideal_spectral}, we extract the following explicit rate.  Set
\[
t_N = \sqrt{\frac{\gamma C_{\eta} M\log(N+J)}{\delta(1-\delta)N}},
\]
where \(\gamma>2\) is fixed and \(C_{\eta}\) is the constant from Lemma~\ref{BanderiaCor}.  The proof of Lemma~\ref{lem:ideal_spectral} shows that
\[
\mathbb{P}\Bigl(\sigma_1(R^*) \ge (1+\eta)\Bigl(1+\sqrt{\frac{J}{N}}\Bigr)+t_N\Bigr)\le (N+J)^{1-\gamma}.
\]

Since \(\gamma>2\), the right‑hand side tends to \(0\).  Let $\eta$ be close to zero, we get
\[
\mathbb{P}\Bigl(\sigma_1(R^*) > 1+\sqrt{\frac{J}{N}}+C_R\sqrt{\frac{\log(N+J)}{N}}\Bigr)\longrightarrow0.
\]

Condition (Con1) implies \(\displaystyle \frac{N\tau_N^2}{\log(N+J)}\to\infty\). Hence, we have \(\frac{\tau_N}{2}\ge C_R\sqrt{\frac{\log(N+J)}{N}}\) holds for all sufficiently large \(N\).  For such \(N\), we have the inclusion
\[
\{\sigma_1(R^*)\ge 1+\sqrt{J/N}+\tau_N/2\}\subseteq
\{\sigma_1(R^*) > 1+\sqrt{J/N}+C_R\sqrt{\log(N+J)/N}\},
\]
because the left‑hand threshold is larger than the right‑hand threshold.  Consequently, we get
\[
\mathbb{P}\bigl(\sigma_1(R^*)\ge 1+\sqrt{J/N}+\tau_N/2\bigr)
\le \mathbb{P}\bigl(\sigma_1(R^*) > 1+\sqrt{J/N}+C_R\sqrt{\log(N+J)/N}\bigr)\xrightarrow[N\to\infty]{}0.
\]

\noindent\emph{Bound for \(\|\tilde R-R^*\|\).}  
Inspecting the proof of Lemma~\ref{lem:perturbation}, there exists an event \(\mathcal{F}_N\) with \(\mathbb{P}(\mathcal{F}_N^c)\to0\) and a constant \(C_{\text{pert}}>0\) (depending only on \(\delta,c_0,M\)) such that on \(\mathcal{F}_N\),
\[
\|\tilde R-R^*\|\le C_{\text{pert}}\sqrt{\frac{JK\log(JK)}{N}}.
\]

Condition (Con1) asserts that \(\displaystyle \frac{\tau_N}{\sqrt{JK\log(JK)/N}}\to\infty\).  Hence, there exists \(N_*\) such that for all \(N\ge N_*\),
\[
\frac{\tau_N}{2} > C_{\text{pert}}\sqrt{\frac{JK\log(JK)}{N}}.
\]

On the event \(\mathcal{F}_N\) and for \(N\ge N_*\), the above inequality together with the bound on \(\|\tilde R-R^*\|\) implies \(\|\tilde R-R^*\| < \tau_N/2\).  Therefore, we get
\[
\{\|\tilde R-R^*\|\ge\tau_N/2\}\subseteq\mathcal{F}_N^c,
\]
and consequently \(\mathbb{P}(\|\tilde R-R^*\|\ge\tau_N/2)\le\mathbb{P}(\mathcal{F}_N^c)\to0\).

Both probabilities in the decomposition tend to zero, hence we have \(\mathbb{P}(\mathcal{G}_C)\to0\).

\vspace{3mm}
\noindent\textbf{Part 3:  Control of \(\mathbb{P}(\mathcal{G}_B)\) – under‑specification.}  
If \(K=1\), the set \(\{K_0<K\}\) is empty and \(\mathbb{P}(\mathcal{G}_B)=0\) trivially.  Hence, we assume \(K\ge2\). Fix an arbitrary \(K_0\) with \(1\le K_0<K\).  We will derive an upper bound for \(\mathbb{P}(T_{K_0}<\tau_N)\) that does not depend on the particular \(K_0\).

\vspace{2mm}
\noindent\textbf{Step 1.  Deterministic objects from Lemma~\ref{lem:underfit_lower}.}  
Because \(K_0<K\), Lemma~\ref{lem:underfit_lower} guarantees the existence of two distinct true classes \(k_1,k_2\in[K]\), subsets \(S=S_1\cup S_2\subseteq[N]\) with \(S_1\subseteq\mathcal{C}_{k_1}, S_2\subseteq\mathcal{C}_{k_2}\), and a subset \(T\subseteq[J]\) satisfying the deterministic bounds
\[
|S_1|,\;|S_2|\ge\frac{c_0N}{KK_0},\qquad |S|\ge\frac{2c_0N}{KK_0},\qquad |T|\ge c_1J.
\]

All subjects in \(S\) are assigned by \(\hat Z\) to the same estimated class \(\kappa\in[K_0]\). Define the scaling factors \(\gamma_j:=1/\sqrt{N\hat{\mathcal{V}}(i,j)}\) for \(i\in S\) (constant on \(S\)), and set \(W:=(R-\mathcal{R})_{S,T}\), \(\mathscr{M}:=(\mathcal{R}-\hat{\mathcal{R}})_{S,T}\), \(D:=\operatorname{diag}(\gamma_j)_{j\in T}\). Then \(\tilde R_{S,T}= (W+\mathscr{M})D\).

\vspace{2mm}
\noindent\textbf{Step 2.  Events and probability bounds from Theorem~\ref{thm:alt}.}  
In the proof of Theorem~\ref{thm:alt}, the following events are introduced (with the same \(S,T,\kappa\)):
\[
\mathscr{E}_A:=\{\gamma_j>0\ \forall j\in T\},\quad
\mathscr{E}_B:=\{\hat\Theta(j,\kappa)\in[\delta M/2,(1-\delta/2)M]\ \forall j\in T\},\quad
\mathscr{E}_\mathcal{E}:=\{\|R-\mathcal{R}\|\le 2.5\sqrt{MN}\}.
\]

From that proof, we have the exponential bounds (valid for all sufficiently large \(N\)):
\begin{align}\label{boundE_ABC}
\mathbb{P}(\mathscr{E}_A^c)\le 2J e^{-c_A N/(KK_0)},\quad
\mathbb{P}(\mathscr{E}_B^c)\le 2J e^{-c_B N/(KK_0)},\quad
\mathbb{P}(\mathscr{E}_\mathcal{E}^c)\le (N+J)e^{-c_\mathcal{E}N},
\end{align}
where \(c_A:=2c_0M(-\ln(1-\delta))\), \(c_B:=c_0\delta^2\), and \(c_\mathcal{E}>0\) depends only on \(M\) and the constant from Lemma~\ref{BanderiaCor}. Moreover, on the event \(\mathscr{E}_A\cap\mathscr{E}_B\cap\mathscr{E}_\mathcal{E}\), the analysis in Theorem~\ref{thm:alt} yields the deterministic lower bound
\[
T_{K_0} \ge 2\eta_0,
\]
where \(\eta_0\) is the constant from Assumption~\ref{ass:A6}. 

\vspace{2mm}
\noindent\textbf{Step 3.  From the event bound to a bound on \(\mathbb{P}(T_{K_0}<\tau_N)\).}  
Condition (Con1) ensures \(\tau_N<2\eta_0\) for all large \(N\). Consequently, we have
\[
\mathscr{E}_A\cap\mathscr{E}_B\cap\mathscr{E}_\mathcal{E}\subseteq\{T_{K_0}\ge2\eta_0\}\subseteq\{T_{K_0}\ge\tau_N\}.
\]

Taking complements and applying the union bound, we obtain
\[
\mathbb{P}(T_{K_0}<\tau_N)\le \mathbb{P}(\mathscr{E}_A^c)+\mathbb{P}(\mathscr{E}_B^c)+\mathbb{P}(\mathscr{E}_\mathcal{E}^c).
\]

\vspace{2mm}
\noindent\textbf{Step 4.  Uniform bound independent of \(K_0\).}  
Since \(K_0\le K-1\) and \(K\ge2\), we have \(KK_0\le K(K-1)\le K^2\), which gives
\[
\frac{N}{KK_0}\ge\frac{N}{K(K-1)}\ge\frac{N}{K^2}.
\]

Therefore, the exponential bounds in Equation (\ref{boundE_ABC}) can be weakened to
\[
\mathbb{P}(\mathscr{E}_A^c)\le 2J\,e^{-c_A N/K^2},\qquad
\mathbb{P}(\mathscr{E}_B^c)\le 2J\,e^{-c_B N/K^2}.
\]

Note that the bound for \(\mathbb{P}(\mathscr{E}_\mathcal{E}^c)\) in Equation (\ref{boundE_ABC}) already does not involve \(K_0\). Consequently, for every \(K_0<K\) and all sufficiently large \(N\), we have
\[
\mathbb{P}(T_{K_0}<\tau_N)\le 2J\,e^{-c_A N/K^2}+2J\,e^{-c_B N/K^2}+(N+J)e^{-c_\mathcal{E}N}.
\]

\vspace{2mm}
\noindent\textbf{Step 5.  Union bound over all \(K_0<K\).}  
Because there are at most \(K-1\) candidates with \(K_0<K\), we have
\begin{align*}
\mathbb{P}(\mathcal{G}_B) &\le \sum_{K_0=1}^{K-1}\mathbb{P}(T_{K_0}<\tau_N)\le (K-1)\Bigl[2J\,e^{-c_A N/K^2}+2J\,e^{-c_B N/K^2}+(N+J)e^{-c_\mathcal{E}N}\Bigr].
\end{align*}

\vspace{3mm}
\noindent\textbf{Part 4:  Asymptotic analysis under Assumption~\ref{ass:A4}.}  
\begin{itemize}
  \item For the term \(2(K-1)J e^{-c N/K^2}\) (with \(c=c_A\) or \(c_B\)):  
Assumption~\ref{ass:A4}(i) states \(\frac{K^2\log(N+J)}{N}\to0\).  This implies that for any fixed \(m_0>0\), there exists \(N_2\) such that for all \(N\ge N_2\),
\[
\frac{N}{K^2}\ge m_0\log(N+J).
\]

Choose \(m_0\) such that \(cm_0 > 1\).  Then, for all \(N\ge N_2\), we have
\[
e^{-c N/K^2}\le e^{-cm_0\log(N+J)}=(N+J)^{-cm_0}\le N^{-cm_0}.
\]

Assumption~\ref{ass:A4}(ii) gives \(\frac{JK}{N}\to0\). Hence, there exists \(N_3\) such that for all \(N\ge N_3\), \(JK\le N\), which implies \((K-1)J\le KJ\le N\).  Therefore, for all \(N\ge\max(N_2,N_3)\), we have
\[
2(K-1)J e^{-c N/K^2}\le 2N\cdot N^{-cm_0}=2N^{1-cm_0}\xrightarrow[N\to\infty]{}0,
\]
since \(1-cm_0<0\).
\item For the term \((K-1)(N+J)e^{-c_\mathcal{E}N}\):  
From \(K^2=o(N)\) (a direct consequence of Assumption~\ref{ass:A4}(i)), we have \(K=o(N^{1/2})\), and certainly \(K=o(N)\).  Hence, there exists \(N_4\) such that for all \(N\ge N_4\), \(K\le N\).  Also, from \(J/N\to0\), we eventually have \(J\le N\).  Thus for \(N\ge N_4\), we have \((K-1)(N+J)\le N\cdot 2N =2N^2\).  Because \(e^{-c_\mathcal{E}N}=o(N^{-2})\) (exponential decay dominates any polynomial), we obtain
\[
(K-1)(N+J)e^{-c_\mathcal{E}N}\xrightarrow[N\to\infty]{}0.
\]
\end{itemize}

All components of \(\mathbb{P}(\mathcal{G}_B)\) converge to \(0\), therefore \(\mathbb{P}(\mathcal{G}_B)\to0\).

\vspace{3mm}
\noindent\textbf{Part 5:  Conclusion.}  
We have established \(\mathbb{P}(\mathcal{G}_C)\to0\) and \(\mathbb{P}(\mathcal{G}_B)\to0\).  By decomposition Equation (\ref{eq:decomp_cons}), \(\mathbb{P}(\mathcal{G}_A^c)\to0\), i.e. \(\mathbb{P}(\hat K = K)\to1\).  This completes the proof of Theorem~\ref{thm:consistency}.
\end{proof}

\subsection{Proof of Theorem \ref{thm:ratio_behavior}}
\begin{proof}
We prove the two statements separately. Throughout the proof, we work under the given assumptions, which allow $K$ to grow with $N$ subject to $K^3=o(N)$.

\paragraph{Part 1: Divergence at the true model}
Assume \(K_0 = K\) is correctly specified.  
We first show that \(T_K = o_P(1)\).  
From Assumption~\ref{ass:A4}(ii), we have \(J = o(N)\). Since Assumption~\ref{ass:A1} holds, Lemma~\ref{lem:lower_bound_j_on} is applicable and yields \(\|R^*\| \xrightarrow{\mathrm{a.s.}} 1\), i.e., \(\|R^*\| = 1 + o_P(1)\). By Lemma~\ref{lem:perturbation} and Assumption~\ref{ass:A4}(ii), \(\|\tilde{R} - R^*\| =O_P(\sqrt{\frac{JK\log(JK)}{N}})= o_P(1)\).  Weyl's inequality then implies
\[
|\sigma_1(\tilde{R}) - \sigma_1(R^*)| \le \|\tilde{R} - R^*\| = o_P(1),
\]
so \(\sigma_1(\tilde{R}) = \sigma_1(R^*) + o_P(1) = 1 + o_P(1)\).  
Therefore, we get
\[
T_K = \sigma_1(\tilde{R}) - \Bigl(1 + \sqrt{\frac{J}{N}}\Bigr) = o_P(1) - \sqrt{\frac{J}{N}} = o_P(1),
\]
because \(\sqrt{J/N}\to 0\). 

Now consider \(K_0 = K-1 < K\).  Assumption \ref{ass:A6} (assumed to hold for every \(K_0 < K\)) ensures that Theorem~\ref{thm:alt} applies, giving a constant \(2\eta_0 > 0\) such that
\[
\mathbb{P}\bigl(T_{K-1} > 2\eta_0\bigr) \to 1.
\]

Fix an arbitrary constant \(M_a>0\).  $T_K=o_P(1)$ gives \(\mathbb{P}(|T_K| < 2\eta_0 / M_a) \to 1\).  Define two events
\[
A_N := \{ T_{K-1} > 2\eta_0 \},\qquad B_N := \{ |T_K| < 2\eta_0/ M_a\}.
\]

We have \(\mathbb{P}(A_N \cap B_N) \to 1\). On the event \(A_N \cap B_N\), we get
\[
r_K = |\frac{T_{K-1}}{T_K}| > \frac{2\eta_0}{2\eta_0 / M_a} =M_a.
\]

Thus, we get \(\mathbb{P}(r_K > M_a) \ge \mathbb{P}(A_N \cap B_N) \to 1\).  Because \(M_a\) is arbitrary, we conclude that \(r_K \xrightarrow{\mathbb{P}} \infty\). 

\paragraph{Part 2: Boundedness under under‑fitting}
Let $c_{\mathrm{low}}$ be the constant from Lemma~\ref{lem:ratio_lower}. For each $m$ with $1 \le m < K$, Lemma~\ref{lem:ratio_lower} guarantees the existence of an event $\mathcal{E}_N^{(m)}$ such that $\mathbb{P}(\mathcal{E}_N^{(m)}) \to 1$ and on $\mathcal{E}_N^{(m)}$,
\[
T_m \ge c_{\mathrm{low}}\frac{\sqrt{J}}{\sqrt{K}\,m}.
\]

To control the probabilities of the complements, we recall the construction in the proof of Theorem~\ref{thm:alt}. For a given under‑fitted candidate $K_0=m$, the proof of Theorem~\ref{thm:alt} defines three events $\mathscr{E}_A(m)$, $\mathscr{E}_B(m)$ and $\mathscr{E}_\mathcal{E}$ (the latter does not depend on $m$) such that $\mathcal{E}_N^{(m)} = \mathscr{E}_A(m) \cap \mathscr{E}_B(m) \cap \mathscr{E}_\mathcal{E}$. From the estimates obtained there (see the bounds \eqref{eq:EA_bound_final} and the subsequent analysis), there exist positive constants $c_A, c_B, c_\mathcal{E}$ depending only on the model parameters $\delta, c_0, M$ (and on the constant from Lemma~\ref{BanderiaCor} for $c_\mathcal{E}$) such that for all sufficiently large $N$,
\[
\mathbb{P}\bigl(\mathscr{E}^c_A(m)\bigr) \le 2J\exp\Bigl(-\frac{c_A N}{K m}\Bigr),\qquad
\mathbb{P}\bigl(\mathscr{E}^c_B(m)\bigr) \le 2J\exp\Bigl(-\frac{c_B N}{K m}\Bigr),\qquad
\mathbb{P}(\mathscr{E}^c_\mathcal{E}) \le (N+J)e^{-c_\mathcal{E}N}.
\]
(These bounds are uniform in $m$ because the constants $c_A,c_B,c_\mathcal{E}$ are derived from Assumptions~\ref{ass:A1}--\ref{ass:A4} and do not involve $m$; the factor $K$ in the denominator appears because the class sizes scale as $N/K$ uniformly, as established in Assumption~\ref{ass:A2}.)

Using the union bound,
\[
\mathbb{P}\bigl((\mathcal{E}_N^{(m)})^c\bigr) \le \mathbb{P}(\mathscr{E}^c_A(m)) + \mathbb{P}(\mathscr{E}^c_B(m)) + \mathbb{P}(\mathscr{E}^c_\mathcal{E})
\le 2J\exp\Bigl(-\frac{c_A N}{K m}\Bigr) + 2J\exp\Bigl(-\frac{c_B N}{K m}\Bigr) + (N+J)e^{-c_\mathcal{E}N}.
\]

Since $m \le K$, we have $1/(K m) \ge 1/K^2$, and therefore
\[
\exp\Bigl(-\frac{c_A N}{K m}\Bigr) \le \exp\Bigl(-\frac{c_A N}{K^2}\Bigr),\qquad
\exp\Bigl(-\frac{c_B N}{K m}\Bigr) \le \exp\Bigl(-\frac{c_B N}{K^2}\Bigr).
\]

Consequently, for all $m<K$, we have
\[
\mathbb{P}\bigl((\mathcal{E}_N^{(m)})^c\bigr) \le 4J\exp\Bigl(-\frac{c_{AB} N}{K^2}\Bigr) + (N+J)e^{-c_\mathcal{E}N},
\]
where we set $c_{AB} := \min\{c_A, c_B\} > 0$.

Now define $\tilde{\mathcal{E}}_N := \bigcap_{m=1}^{K-1} \mathcal{E}_N^{(m)}$. Applying the union bound gives
\[
\mathbb{P}(\tilde{\mathcal{E}}_N^c) = \mathbb{P}\!\left(\bigcup_{m=1}^{K-1} (\mathcal{E}_N^{(m)})^c\right)
\le \sum_{m=1}^{K-1} \mathbb{P}\bigl((\mathcal{E}_N^{(m)})^c\bigr)
\le (K-1)\Bigl[4J\exp\!\Bigl(-\frac{c_{AB} N}{K^2}\Bigr) + (N+J)e^{-c_\mathcal{E}N}\Bigr].
\]

We show that the right‑hand side tends to zero.
\begin{itemize}
    \item \textbf{First term.} By Assumption~\ref{ass:A4}, we get $K^2\log N\ll N$ and $\log((K-1)J)\ll\log N$ for sufficiently large $N$. Therefore, we have
    \[
    \log\!\left((K-1)J e^{-c_{AB}N/K^2}\right) \ll\log N - \frac{c_{AB} N}{K^2} \longrightarrow -\infty,
    \]
    so $(K-1)J e^{-c_{AB} N/K^2} = o(1)$.
    \item \textbf{Second term.} Since $K\ll N^{1/2}$ and $J\ll N$ for large $N$ by Assumption~\ref{ass:A4}, we have
    \[
    (K-1)(N+J)e^{-c_\mathcal{E}N} \ll 2N^{1.5} e^{-c_\mathcal{E}N} = o(1),
    \]
    as the exponential decay dominates any polynomial growth.
\end{itemize}

Thus, we have $\mathbb{P}(\tilde{\mathcal{E}}_N^c) \to 0$, i.e., $\mathbb{P}(\tilde{\mathcal{E}}_N) \to 1$. On the event $\tilde{\mathcal{E}}_N$, for every $m$ with $1\le m<K$ we have the lower bound $T_m \ge c_{\mathrm{low}}\frac{\sqrt{J}}{\sqrt{K}\,m}$. In particular, for any $K_0$ with $2\le K_0<K$, we have
\[
T_{K_0-1} \ge c_{\mathrm{low}}\frac{\sqrt{J}}{\sqrt{K}\,(K_0-1)}, \qquad
T_{K_0} \ge c_{\mathrm{low}}\frac{\sqrt{J}}{\sqrt{K}\,K_0}.
\]

Lemma~\ref{lem:ratio_upper} provides the deterministic universal bound $T_m \le \sqrt{MJ}$ for every $m\ge 1$, which holds with probability $1$. Thus, on $\tilde{\mathcal{E}}_N$, we have
\[
r_{K_0} = \frac{T_{K_0-1}}{T_{K_0}} \le \frac{\sqrt{MJ}}{c_{\mathrm{low}}\frac{\sqrt{J}}{\sqrt{K}\,K_0}}
= \frac{\sqrt{M}}{c_{\mathrm{low}}}\,\sqrt{K}\,K_0
\le \frac{\sqrt{M}}{c_{\mathrm{low}}}\,\sqrt{K}\,(K-1).
\]

Therefore, we obtain
\[
\mathbb{P}\left( r_{K_0} > \frac{\sqrt{M}}{c_{\mathrm{low}}}\,\sqrt{K}\,(K-1) \right) \le \mathbb{P}(\tilde{\mathcal{E}}_N^c) \longrightarrow 0,
\]
which completes the proof.
\end{proof}

\subsection{Proof of Theorem \ref{thm:ratio_consistency}}
\begin{proof}
We treat the two possibilities \(K=1\) and \(K\ge 2\) separately.  Throughout the proof, all constants may depend on the fixed number \(K\) and on the model parameters \(\delta,c_0,c_1,M,\zeta\), but never on \(N\) or \(J\).  

\vspace{3mm}
\noindent\textbf{Part 1: Case \(K = 1\)}  
Under the null hypothesis \(K=1\), Theorem~\ref{thm:null} gives \(T_1 = o_P(1)\).  A closer inspection of the proof (combining Lemma~\ref{lem:perturbation} and Lemma~\ref{lem:ideal_spectral}) shows that there exist a constant \(\tilde{C_1}>0\) and an event \(\mathcal{F}_N^{(1)}\) with \(\mathbb{P}(\mathcal{F}_N^{(1)}) \to 1\) such that on \(\mathcal{F}_N^{(1)}\),
\begin{equation}\label{eq:T1bound}
|T_1| \le \tilde{C_1} \sqrt{\frac{JK\log(JK)}{N}}.
\end{equation}

Because \(K\) is fixed, the right‑hand side is \(O\bigl(\sqrt{J\log J / N}\bigr)\).  Condition (Con1) guarantees that for all sufficiently large \(N\), we have \(\tau_N \ge \tilde{C_1}\sqrt{JK\log(JK)/N}\) on a set whose probability tends to one.  Consequently, we get
\[
\mathbb{P}\bigl(T_1 < \tau_N\bigr) \ge \mathbb{P}\bigl(\mathcal{F}_N^{(1)}\bigr) - o(1) \longrightarrow 1.
\]

Thus the algorithm returns \(\hat K = 1\) with probability tending to one.

\vspace{3mm}
\noindent\textbf{Part 2: Case \(K \ge 2\)}  
We must show two facts:
\begin{itemize}
  \item[(i)] the algorithm does not stop at any under‑fitted candidate \(K_0 < K\);
  \item[(ii)] it does stop at the true candidate \(K_0 = K\).
\end{itemize}

\subparagraph{(i)  No stop at under‑fitted candidates.}
\begin{itemize}
  \item Candidate \(K_0 = 1\). Because \(K\ge2\), the candidate \(1\) is under‑fitted.  Theorem~\ref{thm:alt} (applied with \(K_0 = 1\)) provides a constant \(2\eta_0 > 0\) such that \(\mathbb{P}(T_1 > 2\eta_0) \to 1\).  By Condition (Con1), we have \(\mathbb{P}(T_1 < \tau_N) \le \mathbb{P}(T_1 \le 2\eta_0) \to 0\).  Hence the algorithm does not stop at \(K_0=1\).

  \item Candidates \(2 \le K_0 < K\).  Fix any such \(K_0\).  Because \(K\) is fixed, the condition \(K^3=o(N)\) in Theorem~\ref{thm:ratio_behavior} is trivially satisfied.  Hence part (b) of Theorem~\ref{thm:ratio_behavior} applies, and we obtain
        \begin{equation}\label{eq:bound_under}
        \lim_{N\to\infty} \mathbb{P}\Bigl( r_{K_0} > \frac{\sqrt{M}}{c_{\mathrm{low}}}\,\sqrt{K}\,(K-1) \Bigr) = 0.
        \end{equation}
        Since \(\gamma_N \to \infty\) by Condition (Con2), for all sufficiently large \(N\) we have \(\gamma_N > \frac{\sqrt{M}}{c_{\mathrm{low}}}\,\sqrt{K}\,(K-1)\).  Therefore,
        \[
        \mathbb{P}\bigl( r_{K_0} > \gamma_N \bigr)
        \le \mathbb{P}\Bigl( r_{K_0} > \frac{\sqrt{M}}{c_{\mathrm{low}}}\,\sqrt{K}\,(K-1) \Bigr) + \mathbf{1}_{\{\gamma_N \le \frac{\sqrt{M}}{c_{\mathrm{low}}}\,\sqrt{K}\,(K-1)\}} \longrightarrow 0.
        \]
        Consequently, with probability tending to one, the ratio condition fails for every such \(K_0\). Thus, the algorithm does not stop at any of them.
\end{itemize}

\subparagraph{(ii)  Stop at the true candidate \(K_0 = K\).}
We now prove that \(\mathbb{P}(r_K > \gamma_N) \to 1\).

From Lemma~\ref{lem:ratio_lower} applied with \(K_0 = K-1\), there exists an event \(\mathcal{A}_N\) with \(\mathbb{P}(\mathcal{A}_N) \to 1\) such that on \(\mathcal{A}_N\),
\begin{equation}\label{eq:TKm1_lower}
T_{K-1} \ge c_{\text{low}}\frac{\sqrt{J}}{\sqrt{K}(K-1)} > 0.
\end{equation}

From the proof of Theorem~\ref{thm:null} (specifically Lemma~\ref{lem:perturbation} and Weyl's inequality), there exists a constant \(\tilde{C_2} > 0\) and an event \(\mathcal{B}_N\) with \(\mathbb{P}(\mathcal{B}_N) \to 1\) such that on \(\mathcal{B}_N\),
\begin{equation}\label{eq:TK_upper}
|T_K| \le \tilde{C_2} \sqrt{\frac{JK\log(JK)}{N}}.
\end{equation}

Define \(\breve{\mathcal{E}}_N := \mathcal{A}_N \cap \mathcal{B}_N\).  Since both \(\mathcal{A}_N\) and \(\mathcal{B}_N\) have probability tending to one, \(\mathbb{P}(\breve{\mathcal{E}}_N) \to 1\).  On \(\breve{\mathcal{E}}_N\), combining Equations \eqref{eq:TKm1_lower} and \eqref{eq:TK_upper} yields
\begin{align}
r_K = \frac{T_{K-1}}{|T_K|}
    &\ge \frac{c_{\text{low}}\dfrac{\sqrt{J}}{\sqrt{K}(K-1)}}
             {\tilde{C_2} \sqrt{\dfrac{JK\log(JK)}{N}}}= \frac{c_{\text{low}}}{\tilde{C_2} (K-1) K} \; \frac{\sqrt{N}}{\sqrt{\log(JK)}}. 
\end{align}

Denote \(\tilde{C_3} := \dfrac{c_{\text{low}}}{\tilde{C_2} (K-1) K} > 0\). Then on \(\breve{\mathcal{E}}_N\), we have
\begin{equation}\label{eq:rK_final_lower}
r_K \ge \tilde{C_3} \frac{\sqrt{N}}{\sqrt{\log(JK)}}.
\end{equation}

Condition (Con2) states that 
\(\gamma_N = o\!\left( \sqrt{\frac{N}{\log J}} \right).\) Given that $K$ is fixed, in particular, there exists an integer \(N_0\) such that for all \(N \ge N_0\),
\[
\tilde{C_3} \frac{\sqrt{N}}{\sqrt{\log(JK)}} > \gamma_N.
\]

Consequently, on the event \(\breve{\mathcal{E}}_N\) and for all \(N \ge N_0\), inequality \eqref{eq:rK_final_lower} yields \(r_K > \gamma_N\).  Hence, we obtain
\[
\mathbb{P}\bigl( r_K > \gamma_N \bigr) \ge \mathbb{P}(\breve{\mathcal{E}}_N) \longrightarrow 1.
\]

We have shown that with probability tending to one the algorithm does not stop at any \(K_0 < K\) and does stop at \(K_0 = K\).  Because the algorithm sequentially examines candidates in increasing order, this implies \(\mathbb{P}(\hat K = K) \to 1\). Thus the theorem holds for both \(K=1\) and \(K\ge 2\). 
\end{proof}
\section{Technical lemmas}\label{app:tech}
\begin{lem}[Consistency of estimated parameters under the null]\label{lem:param_consistency}
Under $H_0$ with the true $K = K_0$, suppose that Assumptions \ref{ass:A1}, \ref{ass:A2}, \ref{ass:A4}, and \ref{ass:A5} hold. Let $\hat{Z},\hat{\Theta},\hat{\mathcal{R}}$ be the estimates from Algorithm \ref{alg:gof_lcm} step 1. Then there exists a permutation matrix $\Pi$ such that, with probability tending to $1$,
\begin{enumerate}
    \item $\hat{Z} = Z\Pi$,
    \item $\|\hat{\Theta} - \Theta\Pi\|_F = O_P\left(\sqrt{\frac{JK^2}{N}}\right)$,
    \item $\|\hat{\mathcal{R}} - \mathcal{R}\|_F = O_P\left(\sqrt{JK}\right)$,
    \item $\max_{i,j} |\hat{\mathcal{R}}(i,j) - \mathcal{R}(i,j)| = O_P\left(\sqrt{\frac{K\log(JK)}{N}}\right)$,
    \item There exists a constant $v_{\min}>0$ (specifically, $v_{\min} = \frac{\delta^2 M}{4}$) such that with high probability, $\hat{\mathcal{V}}(i,j) \ge v_{\min}$ for all $i,j$.
\end{enumerate}
\end{lem}

\begin{lem}[Lower bound on structural residual under under‑fitting]\label{lem:underfit_lower}
Assume $K_0 < K$, and Assumptions~\ref{ass:A2} and~\ref{ass:A3} hold.
Then there exist constants $c>0$ and disjoint sets $S\subset[N]$, $T\subset[J]$, 
with $|S|\asymp N/K$, $|T|\ge c_1 J$, such that for \emph{any} estimated classification matrix $\hat Z$ (and the corresponding $\hat\Theta$, $\hat{\mathcal{R}}$) we have deterministically
\begin{equation}\label{eq:L4_final}
\bigl\| (\mathcal{R} - \hat{\mathcal{R}})_{S,T} \bigr\| \;\ge\; c\; \zeta\; \frac{\sqrt{NJ}}{\sqrt{K}\,K_0},
\end{equation}
where the constant $c = \dfrac{\sqrt{c_0^3 c_1}}{2\sqrt{2}}$ depends only on $c_0$ (Assumption~\ref{ass:A2}) and $c_1$ (Assumption~\ref{ass:A3}), and is independent of $N,J,K,K_0$.
\end{lem}

\begin{lem}[Control of the residual matrix]\label{lem:random_control}
    Let \(X = R - \mathcal{R} \in \mathbb{R}^{N\times J}\), we have
    \begin{equation}\label{eq:X_bound_opt}
        \|X\| = O_P\left(\sqrt{N} + \sqrt{J}\right).
    \end{equation}
    Consequently, for any submatrix \(W = (R-\mathcal{R})_{S,T}\) with \(S\subseteq[N],\; T\subseteq[J]\), we have \(\|W\| = O_P\left(\sqrt{N} + \sqrt{J}\right)\).
\end{lem}

\begin{lem}[Lower bound for \(T_{K_0}\) under under‑fitting]\label{lem:ratio_lower}
Assume \(K_0 < K\) and that Assumptions~\ref{ass:A1}--\ref{ass:A4}, and \ref{ass:A6} hold.  Suppose additionally that \(K^3 = o(N)\) as \(N\to\infty\).  Let the constants \(c,\zeta,M, C_{\mathrm{noise}}, C_{\mathrm{signal}},\eta_0\) be as defined in Theorem~\ref{thm:alt}. Then there exists a constant
\(c_{\mathrm{low}}:=\frac{2\eta_0\,C_{\mathrm{signal}}}{C_{\mathrm{noise}}+1+3\eta_0}\)
such that for every \(K_0 < K\),
\begin{equation}\label{eq:lower_bound_ratio}
\mathbb{P}\Bigl( T_{K_0} \ge c_{\mathrm{low}}\,\frac{\sqrt{J}}{\sqrt{K}\,K_0} \Bigr) \longrightarrow 1 \quad\text{as }N\to\infty.
\end{equation}
\end{lem}

\begin{lem}[Deterministic upper bound for \(T_{K_0}\)]\label{lem:ratio_upper}
For any candidate number of latent classes \(K_0\) (including the true one) and for every realization of the data and the estimation procedure, we have
\[
T_{K_0} \le \sqrt{MJ},
\]
where \(M\) is the maximum response category introduced in Definition~\ref{def:LCM_ordinal}. Consequently, \(\mathbb{P}\bigl(T_{K_0}\le \sqrt{MJ}\bigr)=1\) for all \(N\).
\end{lem}

\begin{lem}[Lower bound for $\|R^*\|$ under $J=o(N)$]\label{lem:lower_bound_j_on}
Suppose that Assumption~\ref{ass:A1} holds and $J = o(N)$, we have
\[
\|R^*\| = \frac{\|Y_N\|}{\sqrt{N}} \xrightarrow{a.s.} 1.
\]  
In particular, for any $\varepsilon>0$,  
\[
\lim_{N\to\infty} \mathbb{P}\bigl(\|R^*\| \ge 1 - \varepsilon\bigr) = 1,
\]  
and $\sigma_1(R^*) = \|R^*\| = 1 + o_P(1)$.
\end{lem}
\subsection{Proof of Lemma \ref{lem:param_consistency}}
\begin{proof}
Throughout, we denote $[n] = \{1,2,\ldots,n\}$ for any positive integer $n$.

\vspace{2mm}
\noindent\textbf{Part 1: Perfect recovery of class assignments.}

By Assumption \ref{ass:A5}, when the true number of latent classes is $K$ (i.e., $K_0 = K$), the classification estimation method $\mathcal{M}$ used in Algorithm \ref{alg:gof_lcm} Step 1 is consistent. That is, there exists a permutation matrix $\Pi$ such that
\[
\mathbb{P}\bigl( \hat{Z} = Z\Pi \bigr) \to 1 \quad \text{as } N \to \infty.
\]
This proves statement 1 directly from the assumption.

\vspace{2mm}
\noindent\textbf{Part 2: Frobenius norm error of $\hat{\Theta}$.}

We now condition on the high-probability event $\mathcal{E}_N^{(1)} = \{\hat{Z} = Z\Pi\}$. By the construction of the estimator in Algorithm \ref{alg:gof_lcm} Step 1 (which implements method $\mathcal{M}$), the item parameter matrix is estimated as:
\[
\hat{\Theta} = R^\top \hat{Z} (\hat{Z}^\top \hat{Z})^{-1} = R^\top Z\Pi (\Pi^\top Z^\top Z \Pi)^{-1} = R^\top Z (Z^\top Z)^{-1} \Pi,
\]
where we used the facts that for a permutation matrix $\Pi$, $\Pi^{-1} = \Pi^\top$, and $(\Pi^\top Z^\top Z \Pi)^{-1} = \Pi^\top (Z^\top Z)^{-1} \Pi$.

The true item parameter matrix $\Theta$ satisfies $\mathcal{R} = Z \Theta^\top$, where $\mathcal{R} = \mathbb{E}[R]$. Since $\mathcal{R} = \mathbb{E}[R] = Z \Theta^\top$, we have:
\[
\Theta = \mathcal{R}^\top Z (Z^\top Z)^{-1}.
\]
Therefore,
\begin{align*}
\hat{\Theta} - \Theta \Pi
&= \left[ R^\top Z (Z^\top Z)^{-1} - \mathcal{R}^\top Z (Z^\top Z)^{-1} \right] \Pi \\
&= (R - \mathcal{R})^\top Z (Z^\top Z)^{-1} \Pi.
\end{align*}
Since $\Pi$ is an orthogonal matrix (permutation matrices are orthogonal), we have $\| \Pi \| = 1$ (spectral norm) and $\| A \Pi\|_F = \| A \|_F$ for any matrix $A$. Thus, we have
\[
\| \hat{\Theta} - \Theta \Pi\|_F = \| (R - \mathcal{R})^\top Z (Z^\top Z)^{-1} \|_F.
\]

Define $W = R - \mathcal{R}$. The entries $W(i,j)$ are independent across $i$ and $j$ (conditional on the latent classes, and also unconditionally since the latent classes are fixed). Moreover, $\mathbb{E}[W(i,j)] = 0$, and by the binomial variance formula:
\[
\mathrm{Var}(W(i,j)) = \mathcal{R}(i,j) \left( 1 - \frac{\mathcal{R}(i,j)}{M} \right) \leq \frac{M}{4},
\]
where the inequality holds because the function $h(x) = x(1-x/M)$ attains its maximum $M/4$ at $x = M/2$, and $\mathcal{R}(i,j) \in [0,M]$.

Let $D = (Z^\top Z)^{-1} = \mathrm{diag}(N_1^{-1}, N_2^{-1}, \ldots, N_K^{-1})$. Define $A = W^\top Z D^{1/2} \in \mathbb{R}^{J \times K}$, where $D^{1/2} = \mathrm{diag}(N_1^{-1/2}, N_2^{-1/2}, \ldots, N_K^{-1/2})$. Then:
\[
\| W^\top Z D \|_F^2 = \| A D^{1/2} \|_F^2 = \mathrm{tr}\left( D^{1/2} A^\top A D^{1/2} \right) = \mathrm{tr}\left( A^\top A D \right) = \sum_{k=1}^K \frac{1}{N_k} \| A_{:k} \|_2^2,
\]
where $A_{:k}$ denotes the $k$-th column of $A$.

Now, observe that:
\[
A_{:k} = \frac{1}{\sqrt{N_k}} \sum_{i \in \mathcal{C}_k} W(i,:)^\top,
\]
where $W(i,:) = (W(i,1), \ldots, W(i,J))^\top$. For each $j \in [J]$, the $j$-th component of $A_{:k}$ is:
\[
A_{:k}(j) = \frac{1}{\sqrt{N_k}} \sum_{i \in \mathcal{C}_k} W(i,j).
\]
For fixed $k$ and $j$, the summands $\{ W(i,j) : i \in \mathcal{C}_k \}$ are independent, zero-mean random variables, bounded by $|W(i,j)| \leq M$ (since $R(i,j) \in \{0,1,\ldots,M\}$ and $\mathcal{R}(i,j) \in [0,M]$). Their variances satisfy $\mathrm{Var}(W(i,j)) \leq M/4$ as noted above.

We now compute the expected squared Euclidean norm of $A_{:k}$:
\begin{align*}
\mathbb{E}\left[ \| A_{:k} \|_2^2 \right] 
&= \sum_{j=1}^J \mathbb{E}\left[ A_{:k}(j)^2 \right] = \sum_{j=1}^J \mathrm{Var}\left( A_{:k}(j) \right) \quad \text{(since $\mathbb{E}[A_{:k}(j)] = 0$)} \\
&= \sum_{j=1}^J \frac{1}{N_k} \sum_{i \in \mathcal{C}_k} \mathrm{Var}(W(i,j)) \quad \text{(by independence across $i$)} \\
&\leq \sum_{j=1}^J \frac{1}{N_k} \sum_{i \in \mathcal{C}_k} \frac{M}{4}= \frac{M}{4} \cdot \frac{1}{N_k} \cdot N_k \cdot J = \frac{MJ}{4}.
\end{align*}
Thus, $\mathbb{E}[ \| A_{:k} \|_2^2 ] \leq \frac{MJ}{4}$ for each $k \in [K]$.

By Assumption \ref{ass:A2}, there exists $c_0 > 0$ such that $N_k \geq c_0 N / K$ for all $k$. Therefore, we have
\begin{align*}
\mathbb{E}\left[ \| \hat{\Theta} - \Theta \Pi\|_F^2 \right] 
&= \mathbb{E}\left[ \sum_{k=1}^K \frac{1}{N_k} \| A_{:k} \|_2^2 \right] = \sum_{k=1}^K \frac{1}{N_k} \mathbb{E}\left[ \| A_{:k} \|_2^2 \right] \leq \sum_{k=1}^K \frac{1}{c_0 N / K} \cdot \frac{MJ}{4} = \frac{M K^2 J}{4 c_0 N}.
\end{align*}

Now, by Markov's inequality, for any $\epsilon > 0$, we have
\[
\mathbb{P}\left( \| \hat{\Theta} - \Theta \Pi\|_F \geq \epsilon \right) \leq \frac{\mathbb{E}\left[ \| \hat{\Theta} - \Theta \Pi\|_F^2 \right]}{\epsilon^2} \leq \frac{M K^2 J}{4 c_0 N \epsilon^2}.
\]

For any fixed $\eta > 0$, choose $\epsilon_\eta = \sqrt{ \frac{M K^2 J}{4 c_0 N \eta} }$. Then, we have
\[
\mathbb{P}\left( \| \hat{\Theta} - \Theta \Pi\|_F \geq \epsilon_\eta \right) \leq \eta.
\]

Hence, by definition, we obtain
\[
\| \hat{\Theta} - \Theta \Pi\|_F = O_P\left( \sqrt{ \frac{JK^2}{N} } \right).
\]

This proves statement 2.

\vspace{2mm}
\noindent\textbf{Part 3: Frobenius norm error of $\hat{\mathcal{R}}$.}

Conditioning again on the event $\mathcal{E}_N^{(1)}$ where $\hat{Z} = Z\Pi$ (which occurs with probability tending to 1), we have
\[
\hat{\mathcal{R}} = \hat{Z} \hat{\Theta}^\top = Z \Pi \hat{\Theta}^\top = Z (\hat{\Theta} \Pi^{-1})^\top = Z (\hat{\Theta} \Pi^\top)^\top.
\]

Given that the true expected response matrix is $\mathcal{R} = Z \Theta^\top$, we have
\[
\hat{\mathcal{R}} - \mathcal{R} = Z (\hat{\Theta} \Pi^\top)^\top - Z \Theta^\top = Z \left( (\hat{\Theta} \Pi^\top)^\top - \Theta^\top \right) = Z \left( \hat{\Theta} \Pi^\top - \Theta \right)^\top.
\]

Because $\| \hat{\Theta} \Pi^\top - \Theta \|_F = \| \hat{\Theta} - \Theta \Pi \|_F$, we have
\[
\| \hat{\mathcal{R}} - \mathcal{R} \|_F = \| Z \left( \hat{\Theta} \Pi^\top - \Theta \right)^\top \|_F \leq \| Z \| \cdot \| \hat{\Theta} \Pi^\top - \Theta \|_F = \| Z \| \cdot \| \hat{\Theta} - \Theta \Pi \|_F.
\]

The matrix $Z$ has orthonormal columns up to scaling: $Z^\top Z = \mathrm{diag}(N_1, \ldots, N_K)$. Thus, its spectral norm is $\| Z \| = \sqrt{ \lambda_{\max}(Z^\top Z) } = \sqrt{ N_{\max} }$. By Assumption \ref{ass:A2}, $N_{\max} \leq \frac{1}{c_0} \frac{N}{K}$, so $\| Z \| \leq \sqrt{ \frac{N}{c_0 K} }$.

Combining this with statement 2, we obtain:
\[
\| \hat{\mathcal{R}} - \mathcal{R} \|_F \leq \sqrt{ \frac{N}{c_0 K} } \cdot O_P\left( \sqrt{ \frac{JK^2}{N} } \right) = O_P\left( \sqrt{ \frac{J}{c_0} \cdot K } \right) = O_P\left( \sqrt{ JK } \right).
\]
This proves statement 3.

\vspace{2mm}
\noindent\textbf{Part 4: Uniform entrywise error of $\hat{\mathcal{R}}$.}

We begin by recalling that on the high-probability event $\mathcal{E}_N^{(1)} = \{\hat{Z} = Z\Pi\}$ (which satisfies $\mathbb{P}(\mathcal{E}_N^{(1)}) \to 1$), we have $\hat{\mathcal{R}}(i,j) = \hat{\Theta}(j, \pi(k))$ where $k = \ell(i)$ and $\pi$ is the permutation corresponding to $\Pi$. Moreover, by the construction of the estimator in Algorithm \ref{alg:gof_lcm} Step 1, $\hat{\Theta}(j, \pi(k))$ is precisely the sample mean of the responses to item $j$ over all subjects in the true class $k$, i.e.,
\[
\hat{\Theta}(j, \pi(k)) = \frac{1}{N_k} \sum_{i \in \mathcal{C}_k} R(i,j),
\]

where $\mathcal{C}_k = \{i: \ell(i) = k\}$ and $N_k = |\mathcal{C}_k|$.

Fix an arbitrary pair $(j,k)$ with $j \in [J]$ and $k \in [K]$. Condition on the latent class assignment $\ell$ (which is fixed in our model) and on $\mathcal{E}_N^{(1)}$. Under the data-generating mechanism in the LCM model, the random variables $\{R(i,j): i \in \mathcal{C}_k\}$ are mutually independent (by conditional independence given $\ell$, and unconditionally since $\ell$ is fixed) and each follows a binomial distribution with parameters $M$ and $\Theta(j,k)/M$. Consequently, each $R(i,j)$ is bounded within the interval $[0, M]$ and has mean $\Theta(j,k)$.

We now apply Hoeffding's inequality in its precise form. Let $Y_1, \dots, Y_n$ be independent random variables such that $a_i \le Y_i \le b_i$ almost surely. Then for any $t > 0$, Hoeffding's inequality gives

\[
\mathbb{P}\left( \Bigl| \frac{1}{n}\sum_{i=1}^n Y_i - \mathbb{E}\Bigl[\frac{1}{n}\sum_{i=1}^n Y_i\Bigr] \Bigr| \ge t \right) \le 2\exp\Bigl( -\frac{2n^2 t^2}{\sum_{i=1}^n (b_i-a_i)^2} \Bigr).
\]

In our setting, for the $N_k$ variables $\{R(i,j): i \in \mathcal{C}_k\}$, we have $a_i = 0$, $b_i = M$, and $\sum_{i=1}^{N_k} (b_i-a_i)^2 = N_k M^2$. Therefore, we get

\[
\mathbb{P}\left( \bigl| \hat{\Theta}(j, \pi(k)) - \Theta(j,k) \bigr| \ge t \;\middle|\; \mathcal{E}_N^{(1)} \right) \le 2 \exp\Bigl( - \frac{2 N_k t^2}{M^2} \Bigr).
\]

By Assumption \ref{ass:A2}, there exists a constant $c_0 > 0$ such that $N_k \ge c_0 N/K$ for all $k$. Hence, uniformly in $j,k$, we have

\[
\mathbb{P}\left( \bigl| \hat{\Theta}(j, \pi(k)) - \Theta(j,k) \bigr| \ge t \;\middle|\; \mathcal{E}_N^{(1)} \right) \le 2 \exp\Bigl( - \frac{2 c_0 N t^2}{K M^2} \Bigr).
\]

Now observe that on $\mathcal{E}_N^{(1)}$, for any $i \in \mathcal{C}_k$, $\hat{\mathcal{R}}(i,j) = \hat{\Theta}(j, \pi(k))$ and $\mathcal{R}(i,j) = \Theta(j,k)$. Therefore, we have

\[
\max_{i \in [N], j \in [J]} \bigl| \hat{\mathcal{R}}(i,j) - \mathcal{R}(i,j) \bigr| = \max_{k \in [K], j \in [J]} \bigl| \hat{\Theta}(j, \pi(k)) - \Theta(j,k) \bigr|.
\]

We control this maximum via a union bound over the $JK$ distinct pairs $(j,k)$:
\begin{align*}
\mathbb{P}\left( \max_{i,j} \bigl| \hat{\mathcal{R}}(i,j) - \mathcal{R}(i,j) \bigr| \ge t \;\middle|\; \mathcal{E}_N^{(1)} \right) &= \mathbb{P}\left( \max_{j,k} \bigl| \hat{\Theta}(j, \pi(k)) - \Theta(j,k) \bigr| \ge t \;\middle|\; \mathcal{E}_N^{(1)} \right)\le \sum_{j=1}^J \sum_{k=1}^K \mathbb{P}\left( \bigl| \hat{\Theta}(j, \pi(k)) - \Theta(j,k) \bigr| \ge t \;\middle|\; \mathcal{E}_N^{(1)} \right)\\
&\le JK \cdot 2 \exp\Bigl( - \frac{2 c_0 N t^2}{K M^2} \Bigr) = 2JK \exp\Bigl( - \frac{2 c_0 N t^2}{K M^2} \Bigr).
\end{align*}

For any fixed $\eta > 0$, let 

\[
t = M \sqrt{ \frac{K}{2 c_0 N} \log\left( \frac{4JK}{\eta} \right) }.
\]

Then, we have
\[
2JK \exp\left( - \frac{2 c_0 N t^2}{K M^2} \right) = 2JK \cdot \frac{\eta}{4JK} = \eta/2,
\]  
which gives
\[
\mathbb{P}\left( \max_{i,j} \bigl| \hat{\mathcal{R}}(i,j) - \mathcal{R}(i,j) \bigr| \ge M \sqrt{ \frac{K}{2 c_0 N} \log\left( \frac{2JK}{\eta} \right) } \;\middle|\; \mathcal{E}_N^{(1)} \right) \le \eta/2.
\]

Now we remove the conditioning on $\mathcal{E}_N^{(1)}$. Using the law of total probability, for any $t > 0$, we have

\[
\mathbb{P}\left( \max_{i,j} \bigl| \hat{\mathcal{R}} - \mathcal{R} \bigr| \ge t \right) \le \mathbb{P}\left((\mathcal{E}_N^{(1)})^c\right) + \mathbb{P}\left( \max_{i,j} \bigl| \hat{\mathcal{R}} - \mathcal{R} \bigr| \ge t \;\middle|\; \mathcal{E}_N^{(1)} \right).
\]

Since $\mathbb{P}((\mathcal{E}_N^{(1)})^c) \to 0$, there exists $N_0$ such that for all $N \ge N_0$, $\mathbb{P}((\mathcal{E}_N^{(1)})^c) \le \eta/2$. Then, for $N \ge N_0$, we have

\[
\mathbb{P}\left( \max_{i,j} \bigl| \hat{\mathcal{R}} - \mathcal{R} \bigr| \ge M \sqrt{ \frac{K}{2 c_0 N} \log\left( \frac{4JK}{\eta} \right) } \right) \le\eta,
\]
which means 

\[
\max_{i,j} \bigl| \hat{\mathcal{R}}(i,j) - \mathcal{R}(i,j) \bigr| = O_P\left( \sqrt{ \frac{K \log(JK)}{N} } \right).
\]
\vspace{2mm}
\noindent\textbf{Part 5: Lower bound on $\hat{\mathcal{V}}(i,j)$.}

Recall that \(\hat{\mathcal{V}}(i,j) = \hat{\mathcal{R}}(i,j) \left( 1 - \hat{\mathcal{R}}(i,j)/M \right)\). Define the function  
\[
h(x) = x \left( 1 - \frac{x}{M} \right), \quad x \in [0, M].
\]  

Under Assumption \ref{ass:A1}, we have \(\mathcal{R}(i,j) = \Theta(j,\ell(i)) \in [\delta M, (1-\delta)M]\).

By statement 4, with probability tending to 1, for all \(i,j\),  we have
\[
\left| \hat{\mathcal{R}}(i,j) - \mathcal{R}(i,j) \right| \le \epsilon_N,
\]  
where \(\epsilon_N = C \sqrt{ \frac{K \log (JK)}{N} }\) for some constant \(C > 0\). Under Assumption \ref{ass:A4}, \(\epsilon_N \to 0\). Hence, for sufficiently large \(N\), we have \(\epsilon_N \le \delta M / 2\). Consequently,  
\[
\hat{\mathcal{R}}(i,j) \in \left[ \delta M - \epsilon_N, (1-\delta)M + \epsilon_N \right] \subseteq \left[ \frac{\delta M}{2}, \left(1 - \frac{\delta}{2}\right)M \right].
\]

Because \(h\) is a concave function (since \(h''(x) = -2/M < 0\)), on the interval \([ \delta M/2, (1-\delta/2)M ]\), its minimum is attained at one of the endpoints. Compute  
\[
h\left( \frac{\delta M}{2} \right) = \frac{\delta M}{2} \left( 1 - \frac{\delta}{2} \right) = \frac{\delta(2-\delta)M}{4}, \qquad
h\left( \left(1-\frac{\delta}{2}\right)M \right) = \left(1-\frac{\delta}{2}\right)M \cdot \frac{\delta}{2} = \frac{\delta(2-\delta)M}{4}.
\]  

Thus, both endpoints yield the same value. Therefore, for any \(x \in [ \delta M/2, (1-\delta/2)M ]\),  we have
\[
h(x) \ge \frac{\delta(2-\delta)M}{4}.
\]  

Since \(\delta \in (0, 1/2]\), we have \(2-\delta \ge \delta\), and hence  
\[
\frac{\delta(2-\delta)M}{4} \ge \frac{\delta^2 M}{4}.
\]  

Thus, with probability tending to 1, for all \(i,j\),  we have
\[
\hat{\mathcal{V}}(i,j) \ge \frac{\delta^2 M}{4}.
\]  

We can therefore set \(v_{\min} = \frac{\delta^2 M}{4}\). This proves statement 5.
\end{proof}
\subsection{Proof of Lemma \ref{lem:underfit_lower}}
\begin{proof}
The argument is completely deterministic and uses only elementary counting, the pigeonhole principle, and basic linear algebra.  No probabilistic statement is required until the final conclusion. The proof proceeds in twelve self‑contained steps.

\vspace{2mm}
\noindent\textbf{Step 1.  A large sub‑block inside each true class.}
For each true class $k\in[K]$ and each estimated class $\kappa\in[K_0]$ define
\[
N_{k\kappa}:=\bigl|\bigl\{i\in\mathcal{C}_k:\hat Z(i,\kappa)=1\bigr\}\bigr|.
\]

Sure, we have $\sum_{\kappa=1}^{K_0}N_{k\kappa}=N_k$.
By the averaging principle (a direct consequence of the pigeonhole principle), there exists at least one estimated class $\kappa(k)\in[K_0]$ such that
\[
N_{k,\kappa(k)}\ge \frac{N_k}{K_0}.
\]

Assumption~\ref{ass:A2} gives $N_k\ge c_0 N/K$. Hence, we have
\[
N_{k,\kappa(k)}\ge \frac{c_0 N}{K K_0}.
\]

\vspace{2mm}
\noindent\textbf{Step 2.  Two true classes mapped to the same estimated class.}  
For each $k\in[K]$, Step 1 gives a non‑empty set $\mathcal{A}_k:=\{\kappa\in[K_0]: N_{k,\kappa}\ge N_k/K_0\}$.  
Choose an arbitrary element $\kappa(k)\in\mathcal{A}_k$ and this defines a function $\kappa:[K]\to[K_0]$.  
Because $K_0<K$, $\kappa$ cannot be injective. Hence, by pigeonhole principle, there exist distinct $k_1,k_2\in[K]$ with $\kappa(k_1)=\kappa(k_2)=:\kappa$.  
These two true classes are therefore merged, at least partially, into the same estimated class $\kappa$.

\vspace{2mm}
\noindent\textbf{Step 3.  Construction of the row index set $S$.}
Define
\[
S_1:=\bigl\{i\in\mathcal{C}_{k_1}:\hat Z(i,\kappa)=1\bigr\},\qquad
S_2:=\bigl\{i\in\mathcal{C}_{k_2}:\hat Z(i,\kappa)=1\bigr\},\qquad
S:=S_1\cup S_2.
\]

From Step 1 we have the deterministic lower bounds
\[
|S_1| = N_{k_1,\kappa} \ge \frac{c_0 N}{K K_0},\qquad 
|S_2| = N_{k_2,\kappa} \ge \frac{c_0 N}{K K_0}.
\]

Using the upper bound in Assumption~\ref{ass:A2} ($N_k\le N/(c_0 K)$) and the fact that $S_1\subseteq\mathcal{C}_{k_1}$, $S_2\subseteq\mathcal{C}_{k_2}$, we obtain
\[
|S_1|\le N_{k_1}\le\frac{N}{c_0 K},\qquad |S_2|\le N_{k_2}\le\frac{N}{c_0 K},
\]
and consequently
\[
|S| = |S_1|+|S_2| \le \frac{2N}{c_0 K}.
\]

Thus, we have $\frac{2c_0 N}{K K_0} \le |S| \le \frac{2N}{c_0 K}$.

\vspace{2mm}
\noindent\textbf{Step 4.  Construction of the column index set $T$.}
For the pair $(k_1,k_2)$, we invoke Assumption~\ref{ass:A3}.  Define
\[
T:=\bigl\{j\in[J]:|\Theta(j,k_1)-\Theta(j,k_2)|\ge \zeta /2\bigr\}.
\]

Assumption~\ref{ass:A3} asserts that $|T|\ge c_1 J$. Trivially, we have $|T|\le J$.

\vspace{2mm}
\noindent\textbf{Step 5.  The residual submatrix $\mathscr{M}:=(\mathcal{R}-\hat{\mathcal{R}})_{S,T}$.}
Because every individual in $S$ is assigned to the same estimated class $\kappa$, the fitted expectation is constant on $S\times T$:
\[
\hat{\mathcal{R}}(i,j)=\hat\Theta(j,\kappa)=:\hat\theta(j),\qquad \forall\,i\in S,\;j\in T.
\]

For the true expectation, we have
\[
\mathcal{R}(i,j)=\begin{cases}
\Theta(j,k_1), & i\in S_1,\\[2mm]
\Theta(j,k_2), & i\in S_2.
\end{cases}
\]

Hence, $\mathscr{M}$ admits the block representation
\[
\mathscr{M}=\begin{bmatrix}\mathscr{M}_1\\ \mathscr{M}_2\end{bmatrix},\qquad
\mathscr{M}_1(i,j)=\Theta(j,k_1)-\hat\theta(j)\;(i\in S_1),\; 
\mathscr{M}_2(i,j)=\Theta(j,k_2)-\hat\theta(j)\;(i\in S_2).
\]

\vspace{2mm}
\noindent\textbf{Step 6.  Construction of two unit vectors.}  
We now exhibit specific unit vectors $u\in\mathbb{R}^{|S|}$ and $v\in\mathbb{R}^{|T|}$ that yield a large value of $|u^\top \mathscr{M} v|$, thereby providing a lower bound for $\|\mathscr{M}\|$ via the variational characterization $\|\mathscr{M}\| = \max_{\|u\|=\|v\|=1} |u^\top \mathscr{M} v|$.

\emph{Column vector $v$.}  
For each $j\in T$, set  
\[
s_j := \operatorname{sign}\bigl(\Theta(j,k_1)-\Theta(j,k_2)\bigr) \in \{-1,1\},
\qquad 
v_j := \frac{s_j}{\sqrt{|T|}}.
\]  

Then, we have $\|v\|_2^2 = \sum_{j\in T} 1/|T| = 1$.  The signs convert the absolute differences $|\Theta(j,k_1)-\Theta(j,k_2)|$ into a sum of non‑negative terms.

\emph{Row vector $u$.}  
Recall that every row of $\mathscr{M}_1$ equals $\Theta(j,k_1)-\hat\theta(j)$ and every row of $\mathscr{M}_2$ equals $\Theta(j,k_2)-\hat\theta(j)$.  
To eliminate the unknown common part $\hat\theta(j)$, we assign positive weights to rows in $S_1$ and negative weights to rows in $S_2$, with equal total mass in absolute value.  

Define the balancing coefficients  
\[
\alpha := \frac{\sqrt{|S_2|}}{\sqrt{|S_1|+|S_2|}}, \qquad 
\beta  := \frac{\sqrt{|S_1|}}{\sqrt{|S_1|+|S_2|}},
\]  
and set  
\[
u_i := 
\begin{cases}
\alpha/\sqrt{|S_1|}, & i\in S_1,\\[4mm]
-\beta/\sqrt{|S_2|}, & i\in S_2.
\end{cases}
\]  

A direct computation gives  
\[
\|u\|_2^2 = \sum_{i\in S_1}\frac{\alpha^2}{|S_1|} + \sum_{i\in S_2}\frac{\beta^2}{|S_2|} = \alpha^2+\beta^2 = 1,
\]  
so $u$ is indeed a unit vector.  The opposing signs and the specific choice of $\alpha,\beta$ guarantee that the total weight on $S_1$,
\[
\alpha\sqrt{|S_1|} = \sqrt{\frac{|S_1||S_2|}{|S_1|+|S_2|}},
\]  
is exactly the opposite of the total weight on $S_2$,
\[
-\beta\sqrt{|S_2|} = -\sqrt{\frac{|S_1||S_2|}{|S_1|+|S_2|}},
\]  
which forces the cancellation of $\hat\theta(j)$ when we form $u^\top\mathscr{M} v$ in the next step.

\vspace{2mm}
\noindent\textbf{Step 7.  Exact evaluation of $u^\top \mathscr{M} v$.}
Expanding the bilinear form \(u^\top \mathscr{M} v\) gives
\begin{align*}
u^\top \mathscr{M} v
&= \sum_{i\in S_1}\sum_{j\in T} \frac{\alpha}{\sqrt{|S_1|}}\bigl(\Theta(j,k_1)-\hat\theta(j)\bigr)\frac{s_j}{\sqrt{|T|}}
   + \sum_{i\in S_2}\sum_{j\in T} \frac{-\beta}{\sqrt{|S_2|}}\bigl(\Theta(j,k_2)-\hat\theta(j)\bigr)\frac{s_j}{\sqrt{|T|}} \\
&= \alpha\sqrt{|S_1|}\,\frac{1}{\sqrt{|T|}}\sum_{j\in T}\bigl(\Theta(j,k_1)-\hat\theta(j)\bigr)s_j
   - \beta\sqrt{|S_2|}\,\frac{1}{\sqrt{|T|}}\sum_{j\in T}\bigl(\Theta(j,k_2)-\hat\theta(j)\bigr)s_j.
\end{align*}

Observe that
\[
\alpha\sqrt{|S_1|} = \frac{\sqrt{|S_2|}}{\sqrt{|S_1|+|S_2|}}\sqrt{|S_1|}
= \sqrt{\frac{|S_1||S_2|}{|S_1|+|S_2|}} =: \gamma,
\qquad
\beta\sqrt{|S_2|} = \frac{\sqrt{|S_1|}}{\sqrt{|S_1|+|S_2|}}\sqrt{|S_2|}
= \gamma.
\]

Therefore, we have
\[
u^\top \mathscr{M} v = \frac{\gamma}{\sqrt{|T|}}\sum_{j\in T}
\Bigl[\bigl(\Theta(j,k_1)-\hat\theta(j)\bigr)-\bigl(\Theta(j,k_2)-\hat\theta(j)\bigr)\Bigr]s_j
= \frac{\gamma}{\sqrt{|T|}}\sum_{j\in T}
\bigl(\Theta(j,k_1)-\Theta(j,k_2)\bigr)s_j.
\]

By the definition of $s_j$, we have $(\Theta(j,k_1)-\Theta(j,k_2))s_j = |\Theta(j,k_1)-\Theta(j,k_2)|\ge 0$, which gives
\[
u^\top\mathscr{M} v = \frac{\gamma}{\sqrt{|T|}}\sum_{j\in T}
|\Theta(j,k_1)-\Theta(j,k_2)|.
\]

\vspace{2mm}
\noindent\textbf{Step 8.  Lower bound using the separation condition.}
For every $j\in T$, Assumption~\ref{ass:A3} guarantees $|\Theta(j,k_1)-\Theta(j,k_2)|\ge \zeta /2$. Thus, we have
\[
u^\top\mathscr{M} v \;\ge\; \frac{\gamma}{\sqrt{|T|}}\cdot|T|\cdot\frac{\zeta}{2}
\;=\; \gamma\sqrt{|T|}\,\frac{\zeta}{2}.
\]

\vspace{2mm}
\noindent\textbf{Step 9.  Transfer to the spectral norm.}
Since $\|u\|_2=\|v\|_2=1$, the variational characterization of the spectral norm gives
\[
\|\mathscr{M}\| \;\ge\; |u^\top \mathscr{M}v| \;\ge\; \gamma\sqrt{|T|}\,\frac{\zeta}{2}.
\]

\vspace{2mm}
\noindent\textbf{Step 10.  Quantitative lower bound for $\gamma$.}
From the size estimates in Step 3, we have
\[
|S_1|,\;|S_2| \;\ge\; \frac{c_0 N}{K K_0},\qquad 
|S_1|+|S_2| \;\le\; \frac{2N}{c_0 K}.
\]

Consequently, we get
\[
\gamma \;=\; \sqrt{\frac{|S_1||S_2|}{|S_1|+|S_2|}}
\;\ge\; \sqrt{\frac{\bigl(c_0 N/(K K_0)\bigr)^2}{2N/(c_0 K)}}
\;=\; \sqrt{\frac{c_0^3 N}{2 K K_0^2}}.
\]

\vspace{2mm}
\noindent\textbf{Step 11.  Combining the lower bounds.}
By $|T|\ge c_1 J$ obtained from Step 4 and the estimate for $\gamma$, we get
\begin{align*}
\|\mathscr{M}\| &\;\ge\; \frac{\zeta}{2}\, \sqrt{\frac{c_0^3 N}{2 K K_0^2}}\; \sqrt{c_1 J}
      \;=\; \frac{\zeta}{2}\,\sqrt{\frac{c_0^3 c_1 N J}{2 K K_0^2}} \;=\; \frac{\sqrt{c_0^3 c_1}}{2\sqrt{2}}\;\zeta\; \frac{\sqrt{NJ}}{\sqrt{K}\,K_0}.
\end{align*}

Define the constant
\[
c \;:=\; \frac{\sqrt{c_0^3 c_1}}{2\sqrt{2}}.
\]

Then, we have established the deterministic inequality
\[
\|(\mathcal{R}-\hat{\mathcal{R}})_{S,T}\| \;\ge\; c\,\zeta\,\frac{\sqrt{NJ}}{\sqrt{K}\,K_0}.
\]

\vspace{2mm}
\noindent\textbf{Step 12.  Probability statement.}
The construction of $S,T$ and the subsequent algebraic estimates depend only on the fixed $\hat Z$ and the model parameters.  They hold for \emph{every} possible realization of $\hat Z$ whenever $K_0<K$.  Hence, regardless of the distribution of $\hat Z$,
\[
\mathbb{P}\left( \|(\mathcal{R}-\hat{\mathcal{R}})_{S,T}\| \ge c\,\zeta\,\frac{\sqrt{NJ}}{\sqrt{K}\,K_0} \right) \;=\; 1.
\]

In particular, the conclusion holds with probability tending to $1$ (indeed with probability $1$) for any classifier $\mathcal{M}$.  
\end{proof}
\subsection{Proof of Lemma \ref{lem:random_control}}
\begin{proof}
We prove this lemma by six steps.

    \textbf{Step 1.  Basic estimates.}  
    By the construction of $X$, we have \(\mathbb{E}[X(i,j)] = 0\) for all \(i\in[N], j\in[J]\), and the entries \(\{X(i,j)\}\) are mutually independent. Recall that \(R(i,j)\in\{0,1,\dots,M\}\) and \(\mathcal{R}(i,j)\in[0,M]\), we have \(|X(i,j)|\le M\). Recall that
    \[
        \operatorname{Var}(R(i,j)) = \mathcal{R}(i,j)\Bigl(1-\frac{\mathcal{R}(i,j)}{M}\Bigr)
        \le \max_{x\in[0,M]} x\Bigl(1-\frac{x}{M}\Bigr)=\frac{M}{4},
    \]
    and because \(\mathbb{E}[X(i,j)]=0\), we have \(\mathbb{E}[X(i,j)^2]=\operatorname{Var}(R(i,j))\le M/4\).

    \textbf{Step 2.  Exact quantities required in Lemma~\ref{BanderiaCor}.}  
    Define
    \[
        \sigma_1^* :=\max_{i\in[N]}\sqrt{\sum_{j=1}^{J}\mathbb{E}[X(i,j)^2]},\qquad
        \sigma_2^* :=\max_{j\in[J]}\sqrt{\sum_{i=1}^{N}\mathbb{E}[X(i,j)^2]},\qquad
        \sigma_*^*   :=\max_{i,j}\|X(i,j)\|_\infty .
    \]
    
    From the bounds obtained in Step~1, we obtain deterministic upper bounds for these quantities:
    \[
        \sigma_1^* \le \sqrt{J\cdot\frac{M}{4}} = \frac{\sqrt{MJ}}{2},\qquad
        \sigma_2^* \le \sqrt{N\cdot\frac{M}{4}} = \frac{\sqrt{MN}}{2},\qquad
        \sigma_*^* \le M .
    \]
    
    Without confusion, we introduce the symbols
    \[
        \tilde{\sigma}_1 := \frac{\sqrt{MJ}}{2},\qquad \tilde{\sigma}_2 := \frac{\sqrt{MN}}{2},\qquad \tilde{\sigma}_* := M,
    \]
    which are \textbf{upper bounds} for the true quantities: \(\sigma_1^*\le\tilde{\sigma}_1,\; \sigma_2^*\le\tilde{\sigma}_2,\; \sigma_*^*\le\tilde{\sigma}_*\).

    \textbf{Step 3.  Applying Lemma~\ref{BanderiaCor} with an inflated threshold.}  
    Fix \(\eta = \frac12\),  where this value satisfies the condition \(0<\eta\le\frac12\) of Lemma~\ref{BanderiaCor}.  
    Let \(C_\eta\) be the constant whose existence is guaranteed by the lemma (it depends only on \(\eta\)).  
    For any \(t\ge0\), applying Lemma~\ref{BanderiaCor} to the matrix \(X\) yields
    \[
        \mathbb{P}\Bigl(\|X\|\ge (1+\eta)(\sigma_1^*+\sigma_2^*) + t\Bigr)
        \le (N+J)\exp\Bigl(-\frac{t^2}{C_\eta(\sigma_*^*)^2}\Bigr). 
    \]

    Because \(\sigma_1^*\le\tilde{\sigma}_1\) and \(\sigma_2^*\le\tilde{\sigma}_2\), we have
    \[
        (1+\eta)(\tilde{\sigma}_1+\tilde{\sigma}_2)+t \;\ge\; (1+\eta)(\sigma_1^*+\sigma_2^*)+t .
    \]
    
    Hence, the event \(\{\|X\|\ge (1+\eta)(\tilde{\sigma}_1+\tilde{\sigma}_2)+t\}\) is contained in the event \(\{\|X\|\ge (1+\eta)(\sigma_1^*+\sigma_2^*)+t\}\). Thus, we have
    \[
        \mathbb{P}\Bigl(\|X\|\ge (1+\eta)(\tilde{\sigma}_1+\tilde{\sigma}_2)+t\Bigr)
        \le \mathbb{P}\Bigl(\|X\|\ge (1+\eta)(\sigma_1^*+\sigma_2^*)+t\Bigr). 
    \]

    Moreover, \(\sigma_*^*\le\tilde{\sigma}_*\) implies \((\sigma_*^*)^2\le\tilde{\sigma}_*^2\) and thus
    \[
        \exp\Bigl(-\frac{t^2}{C_\eta(\sigma_*^*)^2}\Bigr) \le \exp\Bigl(-\frac{t^2}{C_\eta\tilde{\sigma}_*^2}\Bigr). 
    \]

    Combining the above inequalities, we obtain a convenient upper bound expressed entirely in terms of the deterministic constants \(\tilde{\sigma}_1,\tilde{\sigma}_2,\tilde{\sigma}_*\):
    \[
        \mathbb{P}\Bigl(\|X\|\ge (1+\eta)(\tilde{\sigma}_1+\tilde{\sigma}_2)+t\Bigr)
        \le (N+J)\exp\Bigl(-\frac{t^2}{C_\eta\tilde{\sigma}_*^2}\Bigr). 
    \]

    \textbf{Step 4.  Choice of \(t\) and an explicit tail estimate.}  
    Taking \(t = \tilde{\sigma}_1+\tilde{\sigma}_2\) gives
    \begin{align}\label{eqstep4lemma5}
        \mathbb{P}\Bigl(\|X\|\ge (2+\eta)(\tilde{\sigma}_1+\tilde{\sigma}_2)\Bigr)
        \le (N+J)\exp\Bigl(-\frac{(\tilde{\sigma}_1+\tilde{\sigma}_2)^2}{C_\eta\tilde{\sigma}_*^2}\Bigr). 
    \end{align}

    Since \(\tilde{\sigma}_1+\tilde{\sigma}_2 = \frac{\sqrt{M}}{2}(\sqrt{N}+\sqrt{J})\) and \(\tilde{\sigma}_* = M\), we have
    \[
        (\tilde{\sigma}_1+\tilde{\sigma}_2)^2 = \frac{M}{4}\bigl(\sqrt{N}+\sqrt{J}\bigr)^2 \ge \frac{M}{4}(N+J).
    \]

Consequently, we get
    \[
        \frac{(\tilde{\sigma}_1+\tilde{\sigma}_2)^2}{C_\eta\tilde{\sigma}_*^2} \;\ge\; \frac{M}{4}\cdot\frac{N+J}{C_\eta M^2}
        \;=\; \frac{N+J}{4C_\eta M}.
    \]

    Define the constant \(C_{\text{opt}} := (2+\eta)\frac{\sqrt{M}}{2}.\) For our choice \(\eta=\frac12\), \(C_{\text{opt}} = \frac{2.5\sqrt{M}}{2}=1.25\sqrt{M}\).  
    Using Equation (\ref{eqstep4lemma5}) together with the monotonicity of the exponential function, we obtain
    \[
        \mathbb{P}\Bigl(\|X\|\ge C_{\text{opt}}\bigl(\sqrt{N}+\sqrt{J}\bigr)\Bigr)
        \le (N+J)\exp\Bigl(-\frac{N+J}{4C_\eta M}\Bigr).
    \]

    \textbf{Step 5.  Asymptotic \(O_P\) bound.}  
   \((N+J)\exp\Bigl(-\frac{N+J}{4C_\eta M}\Bigr)\) tends to \(0\) as \(N\to\infty\) (it decays exponentially in \(N+J\)).  
    Hence, for every \(\varepsilon>0\) there exists an integer \(N_0\) (depending on \(\varepsilon\) and on the fixed constants \(M,\eta,C_\eta\)) such that for all \(N\ge N_0\) and all \(J\) (which may depend on \(N\) or be fixed), we have
    \[
        \mathbb{P}\Bigl(\|X\|\ge C_{\text{opt}}\bigl(\sqrt{N}+\sqrt{J}\bigr)\Bigr) < \varepsilon.
    \]
    
    By the definition of the \(O_P(\cdot)\) notation, this is exactly
    \[
        \|X\| = O_P\bigl(\sqrt{N}+\sqrt{J}\bigr),
    \]
    which establishes Equation \eqref{eq:X_bound_opt}.

    \textbf{Step 6.  Extension to arbitrary submatrices.}  
    Let \(S\subseteq[N]\) and \(T\subseteq[J]\) be any subsets, and set \(W := (R-\mathcal{R})_{S,T}\).  
    For any matrix, the spectral norm of a submatrix never exceeds that of the full matrix, i.e., \(\|W\|\le\|X\|\).  
    Combining this with Equation \eqref{eq:X_bound_opt} yields
    \[
        \|W\| = O_P\bigl(\sqrt{N}+\sqrt{J}\bigr),
    \]
    which complets the proof. 
\end{proof}

\subsection{Proof of Lemma \ref{lem:ratio_lower}}
\begin{proof}
From the proof of Theorem~\ref{thm:alt} (in particular the analysis leading to the lower bound for \(\sigma_1(\tilde{R})\)), there exists an event \(\mathcal{E}_N^{(K_0)}\) with \(\mathbb{P}(\mathcal{E}_N^{(K_0)})\to 1\) such that on \(\mathcal{E}_N^{(K_0)}\), \(\sigma_1(\tilde{R}) \ge C_{\mathrm{signal}}\frac{\sqrt{J}}{\sqrt{K}K_0} - C_{\mathrm{noise}}\). Hence, on the same event, we have
\begin{equation}\label{eq:T_lower_initial}
T_{K_0} = \sigma_1(\tilde{R}) - \Bigl(1+\sqrt{\frac{J}{N}}\Bigr) \ge C_{\mathrm{signal}}\frac{\sqrt{J}}{\sqrt{K}K_0} - C_{\mathrm{noise}} - 1 - \sqrt{\frac{J}{N}}.
\end{equation}

Define \(D_{N,K_0} := \frac{\sqrt{J}}{\sqrt{K}K_0}.\) Then, Assumption \ref{ass:A6} is exactly \(C_{\mathrm{signal}} D_{N,K_0} \ge C_{\mathrm{noise}} + 1 + 3\eta_0\). Consequently, we get
\begin{equation}\label{eq:D_lower_bound}
D_{N,K_0} \ge d_0 := \frac{C_{\mathrm{noise}}+1+3\eta_0}{C_{\mathrm{signal}}}.
\end{equation}

Now we use the additional condition \(K^3 = o(N)\).  Since \(K_0 \le K-1 < K\), we have
\[
\frac{\sqrt{K}K_0}{\sqrt{N}} \le \frac{K^{3/2}}{\sqrt{N}} = \left(\frac{K^3}{N}\right)^{1/2} \longrightarrow 0.
\]

Thus, there exists an integer \(N_1\) (depending only on the constants \(\eta_0\) and \(d_0\)) such that for all \(N \ge N_1\),
\begin{equation*}
\frac{K^{3/2}}{\sqrt{N}} \le \frac{\eta_0}{d_0}.
\end{equation*}

Consequently, for every \(K_0 < K\) and every \(N \ge N_1\), we have
\begin{equation}\label{eq:K3_bound}
\frac{\sqrt{K}K_0}{\sqrt{N}} \le \frac{\eta_0}{d_0}.
\end{equation}

Observe that \(\sqrt{J/N}\) can be expressed as
\(\sqrt{\frac{J}{N}} = \frac{\sqrt{K}K_0}{\sqrt{N}}\, D_{N,K_0}.\) Combining this with Equation \eqref{eq:K3_bound} yields, for all \(N \ge N_1\) and on \(\mathcal{E}_N^{(K_0)}\),
\begin{equation}\label{eq:sqrtJN_ratio}
\sqrt{\frac{J}{N}} \le \frac{\eta_0}{d_0}\, D_{N,K_0}.
\end{equation}

Substituting Equation \eqref{eq:sqrtJN_ratio} into Equation \eqref{eq:T_lower_initial} gives
\begin{equation}\label{eq:T_lower_after_sub}
T_{K_0} \ge C_{\mathrm{signal}} D_{N,K_0} - C_{\mathrm{noise}} - 1 - \frac{\eta_0}{d_0} D_{N,K_0}
= \left(C_{\mathrm{signal}} - \frac{\eta_0}{d_0}\right) D_{N,K_0} - (C_{\mathrm{noise}}+1).
\end{equation}

We now show that the right‑hand side of Equation \eqref{eq:T_lower_after_sub} is at least \(c_{\mathrm{low}} D_{N,K_0}\).  This is equivalent to
\[
\left(C_{\mathrm{signal}} - \frac{\eta_0}{d_0} - c_{\mathrm{low}}\right) D_{N,K_0} \ge C_{\mathrm{noise}}+1.
\]

Because \(D_{N,K_0} \ge d_0\) by Equation \eqref{eq:D_lower_bound}, it suffices to prove
\begin{equation}\label{eq:sufficient}
\left(C_{\mathrm{signal}} - \frac{\eta_0}{d_0} - c_{\mathrm{low}}\right) d_0 \ge C_{\mathrm{noise}}+1.
\end{equation}

Computing the left‑hand side of Equation \eqref{eq:sufficient} gives
\[
\begin{aligned}
\left(C_{\mathrm{signal}} - \frac{\eta_0}{d_0} - c_{\mathrm{low}}\right) d_0
&= C_{\mathrm{signal}} d_0 - \eta_0 - c_{\mathrm{low}} d_0= (C_{\mathrm{signal}} - c_{\mathrm{low}}) d_0 - \eta_0.
\end{aligned}
\]

Recall the definition of \(c_{\mathrm{low}}\):\(c_{\mathrm{low}} = \frac{2\eta_0 C_{\mathrm{signal}}}{C_{\mathrm{noise}}+1+3\eta_0}.\) Using the expression for \(d_0\) and \(c_{\mathrm{low}}\) obtains
\[
(C_{\mathrm{signal}} - c_{\mathrm{low}}) d_0 = \left(C_{\mathrm{signal}} - \frac{2\eta_0 C_{\mathrm{signal}}}{C_{\mathrm{noise}}+1+3\eta_0}\right) \frac{C_{\mathrm{noise}}+1+3\eta_0}{C_{\mathrm{signal}}}
= C_{\mathrm{noise}}+1+3\eta_0 - 2\eta_0 = C_{\mathrm{noise}}+1+\eta_0.
\]

Therefore, we have
\[
(C_{\mathrm{signal}} - c_{\mathrm{low}}) d_0 - \eta_0 = (C_{\mathrm{noise}}+1+\eta_0) - \eta_0 = C_{\mathrm{noise}}+1,
\]
which means that Equation \eqref{eq:sufficient} holds with equality.  Consequently, for all \(N \ge N_1\) and on \(\mathcal{E}_N^{(K_0)}\), we have
\[
T_{K_0} \ge c_{\mathrm{low}} D_{N,K_0}.
\]

Finally, because \(\mathbb{P}(\mathcal{E}_N^{(K_0)}) \to 1\), we have
\[
\mathbb{P}\Bigl( T_{K_0} \ge c_{\mathrm{low}} \frac{\sqrt{J}}{\sqrt{K}K_0} \Bigr) \ge \mathbb{P}(\mathcal{E}_N^{(K_0)}) \longrightarrow 1,
\]
which establishes Equation \eqref{eq:lower_bound_ratio}.
\end{proof}
\subsection{Proof of Lemma \ref{lem:ratio_upper}}
\begin{proof}
Fix an arbitrary candidate number \(K_0\ge 1\).  Let \(\hat{Z}\in\{0,1\}^{N\times K_0}\) be the estimated classification matrix returned by the classifier \(\mathcal{M}\) in Algorithm~\ref{alg:gof_lcm}.  For each estimated class \(k\in[K_0]\), define its index set \(\hat{C}_k=\{i\in[N]:\hat{Z}(i,k)=1\}\) and its size \(n_k=|\hat{C}_k|\).  By construction of the classifier, we may assume \(n_k>0\) for all \(k\). Otherwise the estimation of \(\hat{\Theta}\) would be undefined.  The fitted expected response matrix is \(\hat{\mathcal{R}}=\hat{Z}\hat{\Theta}^\top\), where \(\hat{\Theta}\in[0,M]^{J\times K_0}\).  Consequently, for any item \(j\in[J]\) and any estimated class \(k\), the fitted value \(\hat{\mathcal{R}}(i,j)\) is constant over all individuals in \(\hat{C}_k\). Denote this common value by \(\hat{\Theta}_{jk}\in[0,M]\).

For each pair \((j,k)\), the quantity
\(\hat{\mathcal{V}}_{jk}= \hat{\Theta}_{jk}\Bigl(1-\frac{\hat{\Theta}_{jk}}{M}\Bigr)\)
is always non‑negative.  If \(\hat{\mathcal{V}}_{jk}>0\), the corresponding entries of the normalized residual matrix \(\tilde{R}\) (see Equation \eqref{eq:norm_res}) are well‑defined and, for \(i\in\hat{C}_k\),
\[
\tilde{R}(i,j)=\frac{R(i,j)-\hat{\Theta}_{jk}}{\sqrt{N\hat{\mathcal{V}}_{jk}}}.
\]

If \(\hat{\mathcal{V}}_{jk}=0\), then necessarily \(\hat{\Theta}_{jk}\in\{0,M\}\) and every \(R(i,j)\) for \(i\in\hat{C}_k\) must equal \(\hat{\Theta}_{jk}\) (otherwise the average could not be exactly \(0\) or \(M\)).  In this case the numerator \(\sum_{i\in\hat{C}_k}(R(i,j)-\hat{\Theta}_{jk})^2\) is zero, and the definition of Equation \eqref{eq:norm_res} sets \(\tilde{R}(i,j)=0\) for all \(i\in\hat{C}_k\).  Hence, regardless of the value of \(\hat{\mathcal{V}}_{jk}\), the contribution of the block \((\hat{C}_k,j)\) to the squared Frobenius norm of \(\tilde{R}\) can be bounded uniformly.

We now derive an upper bound for \(\|\tilde{R}\|_F^2\).  For a fixed block \((j,k)\) and assuming first that \(\hat{\mathcal{V}}_{jk}>0\), we have
\[
\sum_{i\in\hat{C}_k} \tilde{R}(i,j)^2
= \frac{1}{N\hat{\mathcal{V}}_{jk}}\sum_{i\in\hat{C}_k}\bigl(R(i,j)-\hat{\Theta}_{jk}\bigr)^2.
\]

Because each \(R(i,j)\) takes values in \(\{0,1,\dots,M\}\), we have the elementary inequality \(R(i,j)^2\le M\,R(i,j)\).  Summing over \(i\in\hat{C}_k\) yields
\[
\sum_{i\in\hat{C}_k} R(i,j)^2 \le M\sum_{i\in\hat{C}_k} R(i,j)=M n_k\hat{\Theta}_{jk}.
\]

Hence, we get
\[
\sum_{i\in\hat{C}_k}\bigl(R(i,j)-\hat{\Theta}_{jk}\bigr)^2
= \sum_{i\in\hat{C}_k} R(i,j)^2 - n_k\hat{\Theta}_{jk}^2
\le M n_k\hat{\Theta}_{jk} - n_k\hat{\Theta}_{jk}^2
= n_k\hat{\Theta}_{jk}(M-\hat{\Theta}_{jk}).
\]

Now observe that \(\hat{\Theta}_{jk}(M-\hat{\Theta}_{jk}) = M\hat{\mathcal{V}}_{jk}\).  Therefore, we get
\[
\sum_{i\in\hat{C}_k}\bigl(R(i,j)-\hat{\Theta}_{jk}\bigr)^2 \le n_k M \hat{\mathcal{V}}_{jk}.
\]

Substituting this into the expression for the block contribution gives
\[
\sum_{i\in\hat{C}_k} \tilde{R}(i,j)^2
\le \frac{1}{N\hat{\mathcal{V}}_{jk}}\cdot n_k M \hat{\mathcal{V}}_{jk}
= \frac{n_k M}{N}.
\]

If \(\hat{\mathcal{V}}_{jk}=0\), then by construction \(\tilde{R}(i,j)=0\) for all \(i\in\hat{C}_k\), so the left‑hand side is zero, and the inequality \(\sum_{i\in\hat{C}_k} \tilde{R}(i,j)^2 \le \frac{n_k M}{N}\) remains valid because the right‑hand side is non‑negative.  Thus, for every block \((j,k)\), we have
\[
\sum_{i\in\hat{C}_k} \tilde{R}(i,j)^2 \le \frac{n_k M}{N}.
\]

Now summing over all items \(j=1,\dots,J\) and all estimated classes \(k=1,\dots,K_0\) gives
\[
\|\tilde{R}\|_F^2 = \sum_{k=1}^{K_0}\sum_{j=1}^{J} \sum_{i\in\hat{C}_k} \tilde{R}(i,j)^2
\le \sum_{k=1}^{K_0}\sum_{j=1}^{J} \frac{n_k M}{N}
= \frac{M}{N}\Bigl(\sum_{k=1}^{K_0} n_k\Bigr) J.
\]

Since the estimated classes partition the set of individuals, we have \(\sum_{k=1}^{K_0} n_k = N\).  Consequently, we have
\[
\|\tilde{R}\|_F^2 \le \frac{M}{N}\cdot N\cdot J = MJ,
\]
and therefore \(\|\tilde{R}\|_F \le \sqrt{MJ}\).

For any matrix, the spectral norm does not exceed the Frobenius norm, i.e., \(\sigma_1(\tilde{R})\le \|\tilde{R}\|_F\).  Hence, we have
\[
\sigma_1(\tilde{R}) \le \sqrt{MJ}.
\]

Recall from Equation \eqref{eq:test_stat} that the test statistic is defined as \(T_{K_0}= \sigma_1(\tilde{R}) - \bigl(1+\sqrt{J/N}\bigr)\).  Dropping the negative term yields the desired bound
\[
T_{K_0} \le \sqrt{MJ}.
\]

This inequality holds for every possible realization of the data and of the estimation procedure, regardless of whether the candidate model is correctly specified or not.  In particular, it holds with probability one for all sample sizes \(N\).
\end{proof}
\subsection{Proof of Lemma \ref{lem:lower_bound_j_on}}
\begin{proof}
Recall from Equation \eqref{eq:Rstar} that
\(R^*(i,j)=\frac{R(i,j)-\mathcal{R}(i,j)}{\sqrt{N\mathcal{V}(i,j)}}.\)
Assumption~\ref{ass:A1} guarantees $\mathcal{V}(i,j)\ge M\delta(1-\delta)>0$, hence each entry is well defined.  Let $Y=\sqrt{N}R^*$. Thus, $Y$'s $(i,j)$-th entry is $Y_{ij} = \sqrt{N}R^*(i,j)$.  Then the entries $\{Y_{ij}\}$ are independent (because the $R(i,j)$ are independent given the latent structure, and the latent structure is fixed), and satisfy
\[
\mathbb{E}[Y_{ij}]=0,\qquad \operatorname{Var}(Y_{ij})=1,\qquad |Y_{ij}|\le C,
\]
where
\[
C:=\sqrt{\frac{M}{\delta(1-\delta)}}<\infty.
\]

We first establish an almost sure lower bound.  Consider the first column $Y(\cdot,1)=(Y_{11},\dots,Y_{N1})^\top$.  The random variables $\{Y_{i1}^2\}_{i=1}^{N}$ are independent, bounded by $C^2$, and satisfy $\mathbb{E}[Y_{i1}^2]=1$ (i.e., $Y_{i1}$ is standardized). For any $\varepsilon>0$, Hoeffding's inequality yields
\begin{align}\label{eq:hoeffding}
\mathbb{P}\Bigl(\bigl|\frac1N\sum_{i=1}^{N}Y_{i1}^2-1\bigr|>\varepsilon\Bigr)\le 2\exp\Bigl(-\frac{2N\varepsilon^2}{C^4}\Bigr). 
\end{align}

The right‑hand side of Equation \eqref{eq:hoeffding} is summable over $N$ because $\sum_{N=1}^{\infty}e^{-cN}<\infty$ for any $c>0$.  By the Borel–Cantelli lemma, for every fixed $\varepsilon>0$, we have
\[
\mathbb{P}\Bigl(\bigl|\frac1N\sum_{i=1}^{N}Y_{i1}^2-1\bigr|>\varepsilon \text{ infinitely often}\Bigr)=0.
\]

Taking a countable sequence $\varepsilon_m=1/m$ ($m=1,2,\dots$) and intersecting the corresponding almost sure events, we obtain almost surely
\[
\lim_{N\to\infty}\frac1N\sum_{i=1}^{N}Y_{i1}^2 = 1.
\]

Hence, we get
\[
\frac{\|Y(\cdot,1)\|_F}{\sqrt{N}}=\Bigl(\frac1N\sum_{i=1}^{N}Y_{i1}^2\Bigr)^{1/2}\xrightarrow{a.s.}1.
\]

Since $\|Y\|\ge\|Y(\cdot,1)\|_F$, we have
\begin{align}\label{eq:liminf}
\liminf_{N\to\infty}\frac{\|Y\|}{\sqrt{N}}\ge 1\quad\text{almost surely.} 
\end{align}

Next we derive an almost sure upper bound. For $Y$, we have
\[
\tilde{\tilde{\sigma}}_1:=\max_{i\in[N]}\sqrt{\sum_{j=1}^{J}\mathbb{E}[Y_{ij}]^2}=\sqrt{J},\qquad
\tilde{\tilde{\sigma}}_2:=\max_{j\in[J]}\sqrt{\sum_{i=1}^{N}\mathbb{E}[Y_{ij}]^2}=\sqrt{N},\qquad
\tilde{\tilde{\sigma}}_*:=\max_{i,j}\|Y_{ij}\|_\infty\le C.
\]

Fix any $\eta\in(0,1/2]$.  Lemma~\ref{BanderiaCor} guarantees the existence of a constant $C_\eta>0$ such that for every $t\ge0$,
\begin{align}\label{eq:tail_bound}
\mathbb{P}\Bigl(\|Y\|\ge (1+\eta)(\tilde{\tilde{\sigma}}_1+\tilde{\tilde{\sigma}}_2)+t\Bigr)\le (N+J)\exp\Bigl(-\frac{t^2}{C_\eta\tilde{\tilde{\sigma}}_*^2}\Bigr). 
\end{align}

Now fix an arbitrary $\tilde{\delta}>0$.  Choose $\eta=\tilde{\delta}/2$  for $0<\tilde{\delta}\le1$.  Because $J=o(N)$, there exists $N_0$ such that for all $N\ge N_0$,
\[
\sqrt{J} \le \frac{\tilde{\delta}}{4(1+\tilde{\delta}/2)}\sqrt{N}.
\]

Consequently, we get
\[
(1+\eta)\sqrt{J}=(1+\tilde{\delta}/2)\sqrt{J} \le \frac{\tilde{\delta}}{4}\sqrt{N}.
\]

Take $t = \frac{\tilde{\delta}}{4}\sqrt{N}$ in Equation \eqref{eq:tail_bound}.  Then for all $N\ge N_0$, we have
\[
(1+\eta)(\sqrt{N}+\sqrt{J})+t \le (1+\tilde{\delta}/2)\sqrt{N} + \frac{\tilde{\delta}}{4}\sqrt{N} + \frac{\tilde{\delta}}{4}\sqrt{N} = (1+\tilde{\delta})\sqrt{N}.
\]

Hence, the event $\{\|Y\|\ge (1+\tilde{\delta})\sqrt{N}\}$ is contained in the event $\{\|Y\|\ge (1+\eta)(\sqrt{N}+\sqrt{J})+t\}$, and by Equation \eqref{eq:tail_bound}, we have
\begin{align}\label{eq:prob_ub}
\mathbb{P}\Bigl(\|Y\|\ge (1+\tilde{\delta})\sqrt{N}\Bigr) \le (N+J)\exp\Bigl(-\frac{t^2}{C_\eta C^2}\Bigr) = (N+J)\exp\Bigl(-\frac{\tilde{\delta}^2 N}{16 C_\eta C^2}\Bigr). 
\end{align}

The right‑hand side of Equation \eqref{eq:prob_ub} is summable over $N$ because it decays exponentially in $N$.  By the Borel–Cantelli lemma, the event $\{\|Y\|\ge (1+\tilde{\delta})\sqrt{N}\}$ occurs only finitely many times almost surely.  Therefore, we get
\[
\limsup_{N\to\infty}\frac{\|Y\|}{\sqrt{N}}\le 1+\tilde{\delta}\quad\text{almost surely.}
\]

Since $\tilde{\delta}>0$ was arbitrary, we can take a countable sequence $\tilde{\delta}_m\downarrow0$ (e.g., $\tilde{\delta}_m=1/m$) and intersect the corresponding almost sure events to obtain
\begin{align}\label{eq:limsup}
\limsup_{N\to\infty}\frac{\|Y\|}{\sqrt{N}}\le 1\quad\text{almost surely.} 
\end{align}

Combining Equations \eqref{eq:liminf} and \eqref{eq:limsup} yields
\[
\frac{\|Y\|}{\sqrt{N}}\xrightarrow{a.s.}1.
\]

Because $R^*=Y/\sqrt{N}$, we have $\|R^*\|=\|Y\|/\sqrt{N}\xrightarrow{a.s.}1$.  Almost sure convergence implies convergence in probability, so for any $\varepsilon>0$,
\[
\lim_{N\to\infty}\mathbb{P}\bigl(\|R^*\|\ge 1-\varepsilon\bigr)=1,
\]
and in particular $\sigma_1(R^*)=\|R^*\|=1+o_P(1)$.  This completes the proof.
\end{proof}
\section{SC-LCM algorithm and its consistency}\label{app:sc_lcm}
Here, we introduce SC-LCM, a simple spectral clustering algorithm for estimating the latent class membership matrix \(Z\) under the latent class model for ordinal categorical data with polytomous responses. The algorithm takes the top \(K\) left singular vectors of \(R\) and applies \(k\)-means clustering to its rows. Under mild conditions, we prove that the procedure consistently recovers the true latent classes even as \(N\), \(J\), and \(K\) all diverge (i.e., the large‑scale setting). Our analysis begins with an oracle case where the population parameters are known. Let $\mathcal{R} = \mathbb{E}[R] = Z\Theta^\top$ be the $N\times J$ expected response matrix. To simplify our theoretical analysis, we let the rank of $\Theta$ be $K$. Thus, \(\mathcal{R}\) has a low-rank structure with rank \(K\). Since $\mathcal{R}$ has rank $K$, consider its compact singular value decomposition
\begin{equation*}
\mathcal{R} = U\Sigma V^\top,
\end{equation*}
where $U\in\mathbb{R}^{N\times K}$ satisfies $U^\top U = I_K$, $V\in\mathbb{R}^{J\times K}$ satisfies $V^\top V = I_K$, and $\Sigma = \operatorname{diag}(\sigma_1,\ldots,\sigma_K)$ with $\sigma_1\ge\cdots\ge\sigma_K>0$. The following lemma shows that $U$ has exactly $K$ distinct rows and that these rows are perfectly aligned with the latent class memberships.
\begin{lem}\label{lem:U_property}
Under the latent class model, the left singular vectors $U$ of $\mathcal{R}$ satisfy that all rows belonging to the same latent class are identical, and for any two distinct classes $k\neq l$,
\[\|U(i,:)-U(j,:)\|_F = \sqrt{\frac{1}{N_k} + \frac{1}{N_l}}, \qquad i\in\mathcal{C}_k,\ j\in\mathcal{C}_l.
\]
\end{lem}

By Lemma \ref{lem:U_property}, applying $k$-means with $K$ clusters to the rows of $U$ recovers the classification matrix $Z$ exactly up to a permutation of labels. In practice we only observe $R$, not $\mathcal{R}$. Let
\begin{equation*}
R = \hat{U}\hat{\Sigma}\hat{V}^\top + \hat{U}_\perp \hat{\Sigma}_\perp \hat{V}_\perp^\top
\end{equation*}
be the full SVD of $R$, where $\hat{U}\in\mathbb{R}^{N\times K}$ contains the left singular vectors corresponding to the $K$ largest singular values. The practical SC-LCM algorithm is summarized in Algorithm \ref{alg:sc_lcm}.
\begin{algorithm}[H]
\caption{Spectral Clustering for Latent Class Models (SC-LCM)}\label{alg:sc_lcm}
\begin{algorithmic}[1]
\Require{Observed response matrix $R\in\{0,1,\dots,M\}^{N\times J}$, number of latent classes $K$}
\Ensure{Estimated classification matrix $\hat{Z}$}
\State Compute the top $K$ left singular vectors $\hat{U}\in\mathbb{R}^{N\times K}$ of $R$.
\State Apply $k$-means algorithm to the rows of $\hat{U}$ with $K$ clusters to obtain $\hat{Z}$.
\end{algorithmic}
\end{algorithm}

The intuition behind SC-LCM is that $\hat{U}$ is a perturbed version of $U$ (up to an orthogonal rotation), and by Lemma~\ref{lem:U_property} the rows of $U$ are perfectly separable. Hence, under a sufficiently small perturbation, $k$-means on $\hat{U}$ will recover the true classes with high accuracy. 

To establish the consistency of SC-LCM, we introduce a parameter that governs both the signal strength and the sparsity of the data. Define
\[
\rho \;:=\; \max_{j\in[J],\,k\in[K]} \Theta(j,k),
\]
which is the maximum expected response across all items and latent classes. This quantity directly affects the scale of the entries of the observed matrix \(R\): a larger \(\rho\) leads to larger expected responses and thus a denser observation matrix, whereas a smaller \(\rho\) pushes the expected responses toward zero, making the data sparser. In our asymptotic regime we allow \(\rho\to 0\), which corresponds to a sparse setting where most entries of \(R\) are zero.  The following assumption controls the speed of this decay.
\begin{assum}[Sparsity scaling]\label{ass:rho}
\(\rho\max(N,J)\ge M^2\log(N+J)\).
\end{assum}

We also define the scaled item parameter matrix \(\Theta_0 := \Theta / \rho\). By construction, we have every entry of \(\Theta_0\) lies in \([0,1]\). Let \(\sigma_K(\Theta_0)\) denote the \(K\)-th largest singular value of \(\Theta_0\). To measure the difference between $Z$ and $\hat{Z}$, we consider the \emph{Clustering error} used in \citep{joseph2016impact}. This metric is defined as
\[
\operatorname{err}(\hat Z,Z)=\min_{\pi\in S_K}\max_{k\in[K]}\frac{|\mathcal{C}_k\bigcap\hat{\mathcal{C}}^c_{\pi(k)}|+|\hat{\mathcal{C}}_{\pi(k)}\bigcap\mathcal{C}^c_k|}{N_k},
\]
where \(\hat{\mathcal{C}}_k=\{i:\hat Z(i,k)=1\}\) and \(S_K\) is the set of permutations of \(\{1,\dots,K\}\). The following theorem guarantees SC-LCM's estimation consistency under the LCM model.
\begin{thm}[Consistency of SC-LCM]\label{thm:sc_lcm_consistency}
Under Assumption~\ref{ass:rho}, the estimator \(\hat Z\) produced by Algorithm~\ref{alg:sc_lcm} satisfies, with probability tending to \(1\) as \(N\to\infty\),
\[
\operatorname{err}(\hat Z,Z)=O(\frac{K^2 N_{\max}\max(N,J)\log(N+J)}{N_{\min}^2\rho\sigma^2_K(\Theta_0)}).
\]
\end{thm}
When $N_{\min}\asymp N/K$, $N_{\max}\asymp N/K$, and $J=o(N)$, the bound in Theorem~\ref{thm:sc_lcm_consistency} reduces to  
\(\operatorname{err}(\hat Z,Z)=O_P\!\left(\frac{K^3\log N}{\rho\sigma^2_K(\Theta_0)}\right)\). When $\rho\sigma^2_K(\Theta_0)\gg K^3\log N$, we have  $\operatorname{err}(\hat Z,Z)\xrightarrow{\mathbb{P}}0$, which shows SC-LCM's estimation consistency.
\subsection{Proof of Lemma \ref{lem:U_property}}
\begin{proof}
From the compact singular value decomposition $\mathcal{R}=U\Sigma V^\top$ and the factorization $\mathcal{R}=Z\Theta^\top$, we obtain
\(U = \mathcal{R}V\Sigma^{-1} = Z(\Theta^\top V\Sigma^{-1}).\) Setting $X_U=\Theta^\top V\Sigma^{-1}\in\mathbb{R}^{K\times K}$, we have $U = ZX_U$ and $U^\top U = I_K$. This structure—a membership matrix $Z$ post‑multiplied by a square matrix $X_U$ with orthonormal columns—is exactly the one considered in Lemma 2.1 of \citep{lei2015consistency} for the eigenvectors of the stochastic block model's mean matrix. Applying that lemma directly yields the desired properties of the rows of $U$ and the distance formula. 
\end{proof}
\subsection{Proof of Theorem \ref{thm:sc_lcm_consistency}}
\begin{proof}
Set \(W:=R-\mathcal{R}\). For each pair \((i,j)\) define \(W_{ij}=W(i,j)e_i\tilde e_j^\top\) where \(\{e_i\}\) and \(\{\tilde e_j\}\) are the standard basis vectors in \(\mathbb{R}^N\) and \(\mathbb{R}^J\). The matrices \(\{W_{ij}\}\) are independent, centred, and satisfy \(\|W_{ij}\|\le M\). Moreover,
\[
\mathbb{E}[W(i,j)^2]=\operatorname{Var}(R(i,j))=\mathcal{R}(i,j)\Bigl(1-\frac{\mathcal{R}(i,j)}{M}\Bigr)\le\mathcal{R}(i,j)\le\rho.
\]

Hence, we have
\[
\Bigl\|\sum_{i,j}\mathbb{E}[W_{ij}W_{ij}^\top]\Bigr\|\le\rho J,\qquad
\Bigl\|\sum_{i,j}\mathbb{E}[W_{ij}^\top W_{ij}]\Bigr\|\le\rho N.
\]

Under Assumption~\ref{ass:rho}, applying the matrix Bernstein inequality \citep{tropp2012user} with a sufficiently large constant \(C_3\) yields
\[
\mathbb{P}\Bigl(\|W\|\ge C_3\sqrt{\rho\max(N,J)\log(N+J)}\Bigr)\le (N+J)^{-2}\;\xrightarrow[N\to\infty]{}\;0,
\]
which gives
\begin{equation}\label{eq:Wbound_simple}
\|W\|\leq C_3\sqrt{\rho\max(N,J)\log(N+J)}.
\end{equation}

From \(\mathcal{R}=Z\Theta^\top\) and \(\Theta=\rho \Theta_0\), we have
\begin{equation}\label{eq:sigmaK_lower_simple}
\sigma_K(\mathcal{R})\ge\sqrt{N_{\min}}\,\rho\,\sigma_K(\Theta_0).
\end{equation}

The Davis–Kahan \(\sin\Theta\) theorem \citep{Yu2015} guarantees the existence of an orthogonal matrix \(\mathcal{O}\in\mathbb{R}^{K\times K}\) such that
\[
\|\hat{U}\mathcal{O}-U\|_F\le\frac{2\sqrt{2K}\,\|R-\mathcal{R}\|}{\sigma_K(\mathcal{R})}.
\]

Inserting Equations \eqref{eq:Wbound_simple} and \eqref{eq:sigmaK_lower_simple} gives
\begin{equation}\label{eq:Udiff_simple}
\|\hat U\mathcal{O}-U\|_F\le C_4\sqrt{\frac{K\max(N,J)\log(N+J)}{N_{\min}\,\rho\,\sigma^2_K(\Theta_0)}},
\end{equation}
where \(C_4:=2\sqrt{2}\,C_3\).

Lemma~\ref{lem:U_property} shows that the rows of \(U\) are constant within each true latent class: for \(i\in\mathcal{C}_k\), \(U(i,:)=u_k\) for some vector \(u_k\in\mathbb{R}^K\). Moreover, for any two distinct classes \(k\neq l\), Lemma~\ref{lem:U_property} gives
\(\|u_k-u_l\|_F=\sqrt{\frac1{N_k}+\frac1{N_l}}.\) Define \(d_{kl}:=\sqrt{\frac1{N_k}+\frac1{N_l}},\Delta_{kl}:=\frac1{\sqrt{N_k}}+\frac1{\sqrt{N_l}}\). A simple inequality gives \(\Delta_{kl}\le\sqrt{2}\,d_{kl}\) for all \(k\neq l\). Set \(\varsigma:=\sqrt{\frac{2K N_{\max}}{N_{\min}}}\,\|\hat U \mathcal{O}-U\|_F.\) For any distinct classes \(k,l\), we have
\[
\frac{\sqrt{K}}{\varsigma}\|\hat U \mathcal{O}-U\|_F\,\Delta_{kl}
   =\frac{\sqrt{K}}{\sqrt{\frac{2K N_{\max}}{N_{\min}}}\|\hat U \mathcal{O}-U\|_F}\,\|\hat U \mathcal{O}-U\|_F\,\Delta_{kl}
   =\frac{\sqrt{N_{\min}}}{\sqrt{2N_{\max}}}\,\Delta_{kl}
   \le\frac{\sqrt{N_{\min}}}{\sqrt{2N_{\max}}}\sqrt{2}\,d_{kl}
   =\frac{\sqrt{N_{\min}}}{\sqrt{N_{\max}}}\,d_{kl}\le d_{kl}.
\]

Hence, by Lemma 2 of \citep{joseph2016impact}, we obtain
\[
\operatorname{err}(\hat Z,Z)=O(\varsigma^2).
\]

Substituting the definition of \(\varsigma\) and Equation \eqref{eq:Udiff_simple} gives
\[
\operatorname{err}(\hat Z,Z)= O(\frac{KN_{\max}}{N_{\min}}\cdot\frac{K\max(N,J)\log(N+J)}{N_{\min}\rho\sigma^2_K(\Theta_0)})=O(\frac{K^2 N_{\max}\max(N,J)\log(N+J)}{N_{\min}^2\rho\sigma^2_K(\Theta_0)}).
\]
\end{proof}

\bibliographystyle{elsarticle-harv}
\bibliography{refLCMK}
\end{document}